\documentclass[11pt,a4paper]{article}

\usepackage[margin=1in]{geometry}

\usepackage[utf8]{inputenc} 
\usepackage[T1]{fontenc}    
\usepackage[numbers]{natbib} 
\usepackage{hyperref}       
\usepackage{url}            
\usepackage{booktabs}       
\usepackage{amsfonts}       
\usepackage{nicefrac}       
\usepackage{microtype}      
\usepackage{xcolor}         
\usepackage{amsmath,amssymb,amsthm} 
\usepackage{graphicx}       
\usepackage{titletoc}       

\usepackage{microtype}
\usepackage{graphicx}
\usepackage{subfigure}
\usepackage{booktabs} 
\usepackage{multirow}
\usepackage{authblk}

\usepackage{xcolor}
\definecolor{obsOrange}{HTML}{E6B34A}   
\definecolor{latBlue}{RGB}{76,120,168}
\definecolor{targetGreen}{RGB}{84,162,75}
\definecolor{edgePurple}{HTML}{A987B3}  

\usepackage{amsmath}

\DeclareMathOperator*{\argmin}{arg\,min}

\usepackage{caption}
\usepackage{tikz}
\usetikzlibrary{positioning, fit}
\usetikzlibrary{backgrounds}
\usetikzlibrary{tikzmark}
\usetikzlibrary{calc}
\usetikzlibrary{arrows,shapes,positioning,shadows,trees,mindmap}
\usetikzlibrary{arrows.meta}

\usepackage{xcolor}            
\usepackage[normalem]{ulem}    
\usepackage{titletoc}          
\usepackage{tikz}

\usepackage{hyperref}

\usepackage[ruled,vlined,linesnumbered]{algorithm2e}
\DontPrintSemicolon
\SetKwInput{KwIn}{Input}
\SetKwInput{KwOut}{Output}
\SetKw{KwRet}{return}

\SetCommentSty{algcommentstyle} 
\SetKwComment{tcc}{\textcolor{green!50!black}{\(\triangleright\)}\ }{}

\newcommand{\ExoSet}{\mathcal{V}}
\newcommand{\StratExoSet}{\mathcal{V}_L}
\newcommand{\Xset}{\mathcal{X}}

\newcommand{\Lcdf}{L_{cdf}}

\newcommand{\U}{V}
\newcommand{\Objectivev}{J_v}
\newcommand{\Objective}{J}

\newcommand{\Idx}[2]{\mathcal{J}_{#1,#2}}

\usepackage[utf8]{inputenc} 
\usepackage[T1]{fontenc}    
\usepackage{hyperref}       
\usepackage{xr-hyper} 
\usepackage{url}            
\usepackage{booktabs}       
\usepackage{amsfonts}       
\usepackage{nicefrac}       
\usepackage{microtype}      
\usepackage{xcolor}         
\usepackage{bbm}
\colorlet{darkgreen}{green!45!black}
\hypersetup{
  colorlinks,
  citecolor=darkgreen,
  urlcolor=darkgreen}
\usepackage{hyphenat}

\usepackage{pifont}

\usepackage{graphicx}
\usepackage{subfigure}
\usepackage{enumitem}

\usepackage{amssymb}
\usepackage{mathtools}
\usepackage{amsthm}
\usepackage{thmtools, thm-restate}
\usepackage{thmtools, thm-restate}
\usepackage{wrapfig}

\usepackage[capitalize,noabbrev]{cleveref}
\crefname{property}{property}{Properties}
\Crefname{property}{Property}{Properties}

\theoremstyle{plain}
\declaretheorem{theorem}
\declaretheorem[name=Proposition]{proposition}
\declaretheorem[name=Lemma]{lemma}

\declaretheorem[sibling=theorem, numberwithin=section]{property}
\declaretheorem[name=Theorem,numbered=no]{theorem*}
\theoremstyle{definition}
\declaretheorem[name=Definition]{definition}
\declaretheorem[name=Assumption]{assumption}
\theoremstyle{remark}
\declaretheorem[sibling=theorem, numberwithin=section]{remark}

\usepackage[textsize=tiny]{todonotes}

\usepackage[utf8]{inputenc} 
\usepackage[T1]{fontenc}    
\usepackage{hyperref}       
\usepackage{url}            
\usepackage{booktabs}       
\usepackage{amsfonts}       
\usepackage{nicefrac}       
\usepackage{microtype}      
\usepackage{xcolor}         

\title{Counterfactually Fair Regression via Optimal Transport}

\author[1, 3, 4]{Marie Generali Lince \thanks{\texttt{marie.generali@ensae.fr}}}
\author[2]{Solenne Gaucher \thanks{\texttt{solenne.gaucher@polytechnique.edu}}}
\author[1]{Jill-Jênn Vie \thanks{\texttt{jill-jenn.vie@inria.fr}}}
\author[4]{Patrick Loiseau \thanks{\texttt{patrick.loiseau@inria.fr}}}

\affil[1]{Inria, Soda team}
\affil[2]{CMAP, CNRS, Ecole Polytechnique, Institut Polytechnique de Paris}
\affil[3]{CREST, ENSAE, IP Paris}
\affil[4]{Inria, FairPlay joint team}

\begin{document}

\maketitle

\begin{abstract}
  We consider the problem of learning a counterfactually fair regressor.
  We adopt a causal uncertainty view in which counterfactual fairness is defined with resampled noise.
  We focus on obtaining theoretical fairness guarantees for a new post-processing estimator.
  We begin by showing that counterfactual fairness is equivalent to satisfying demographic parity conditional on the latent variable.
  This allows us to provide a closed-form expression of the optimal fair regressor via a barycentric quantile map.
  In order to handle continuous latent variables, we propose a discretized post-processing method.
  Then, under mild regularity assumptions, we prove high-probability finite-sample fairness guarantees for our estimator, providing an unfairness decay at rate $\tilde O(n^{-1/3})$, and establishing a matching risk bound of order $\tilde O(n^{-1/3})$. We provide a matching lower bound on the excess risk of almost fair predictions.
  Finally, we extend our results to the setting of relaxed counterfactual fairness. 
  We validate our approach on real-world and synthetic data. 
      
  \end{abstract}
  
  \section{Introduction}
  
  \noindent
  Machine learning (ML) is increasingly used in high-stakes domains, motivating a rich literature on fair ML \citep{barocas-hardt-narayanan} that aims to mitigate disparate impact or treatment across demographic groups defined by sensitive attributes (e.g., gender).
  In this work, we focus on regression, where the task is to predict a real-valued outcome.
  As a running example, we consider student evaluation: a predictor assigns admission scores based on features such as exam performance and other achievements.
  Our goal is to ensure that the predictor is \emph{fair}. 
  We work in the awareness setting: the sensitive attribute is available at training and inference time and may be used explicitly to enforce fairness constraints \citep{dwork2011fairness}.

  \noindent
  A prominent fairness notion is \emph{demographic parity} (DP), which requires marginal identical prediction distributions across groups. 
  However, such marginal criteria can miss important forms of unfairness \citep{Dhacking,hardt2016equalityopportunitysupervisedlearning}, are sensitive to gerrymandering \citep{gerrymandering,crenshaw2013demarginalizing,pmlr-v81-buolamwini18a}, and to relaxation pitfalls \citep{TooRelaxedToBeFair}. We instead adopt \emph{counterfactual fairness} (CF) \cite{kusner2018counterfactual}.
  Formally, CF requires predictions to be invariant to interventions on sensitive attributes. 
  In practice, complete causal graphs are often unavailable \cite{When_worlds_collide}. We adopt a “causal uncertainty” perspective—specifically, the level-2 setting of \cite{kusner2018counterfactual}—assuming access to a latent ability proxy without specifying the full causal model. 

  \noindent
  Most prior CF regression methods are in-processing approaches that require access to training data and often lack guarantees. 
  We propose a \emph{post-processing} algorithm which recalibrates a pre-trained predictor without retraining. 
  This is crucial when models or datasets are third-party or proprietary or when training data is unavailable.
  Our post-processor comes with finite-sample statistical guarantees.
  
  \paragraph{Contributions.} Our contributions are as follows:
  \begin{itemize}[leftmargin=10pt, topsep=-4pt, partopsep=0pt, parsep=0pt, itemsep=0.3em, beginpenalty=10000]
    \setlength\itemsep{0.3em}
    \setlength{\parskip}{0pt}
    \setlength{\topsep}{-5pt}
    \item  \textbf{CF reduces to \textit{conditional} DP.} We show that under causal uncertainty and awareness, counterfactual fairness is equivalent to demographic parity \emph{conditional} on the ability: within each ability level, group score distributions must coincide. 
    \item \textbf{Closed-form CF-optimal regressor.} We derive the CF-optimal predictor in closed form via a barycentric quantile construction.
    \item \textbf{Practical post-processor with finite-sample guarantees.} 
    We propose a model-agnostic post-processing method requiring no retraining and no labels for calibration, with high-probability finite-sample guarantees: excess risk and unfairness decay both as $\smash{\tilde O(n^{-1/3})}$ in the number $n$ of post-processing samples, and establish matching lower bounds (up to log factors). 
    \item \textbf{Relaxed CF with explicit tuning.}  We extend to relaxed CF and provide explicit tuning rules to meet a prescribed unfairness budget, accounting for both relaxation and estimation error.
\end{itemize}

  \section{Related work}
  \label{sec:related-work}

  \textbf{Counterfactual fairness.}
  Causality-based fairness notions, particularly \textit{counterfactual fairness} have attracted substantial attention \citep{chiappa2019path,creager2020causalmodelingfairnessdynamical,makhlouf2022surveycausalbasedmachinelearning,When_worlds_collide}.
  The dominant approach relies on \emph{structural causal models} (SCMs) \citep{Pearl_2009,pearl2018book}, which provide a formal framework for counterfactual reasoning.
  Within this framework, CF has been studied under known causal structure \citep{kusner2018counterfactual,johansson2018learningrepresentationscounterfactualinference,ijcai2019p199_CF,zhou2025counterfactualfairnesscombiningfactual} or partially known structure \citep{duong2023achievingcounterfactualfairnessimperfect}.
  However, complete SCMs are rarely available in practice. 
  This motivated alternative approaches: \emph{potential-outcomes} formulations that sidestep full causal graphs \citep{coston2020counterfactualriskassessmentsevaluation}, \emph{optimization-based} methods that treat counterfactuals as search problems \cite{ML_CF_Verma}, and \emph{generative/representation} methods that learn to synthesize counterfactuals without explicit causal structure \citep{Ma_2023,zuo2023counterfactuallyfairrepresentation,robertson2024fairpfntransformerscounterfactualfairness}.  Despite progress, challenges remain: most CF methods require training data access and model retraining, and often lack statistical guarantees, or rely on restrictive SCM assumptions to obtain them \citep{zhou2025counterfactualfairnesscombiningfactual}. Moreover, few works address the utility-fairness tradeoff via relaxed CF constraints \citep{When_worlds_collide}, and counterfactual fairness remains difficult to deploy in practice.

 \paragraph{Demographic parity.} A complementary line of work considers regression algorithms that enforce \emph{demographic parity} (DP) in a \emph{post-processing} regime — that is, by applying fairness corrections to an already-trained estimator, while retaining provable statistical guarantees.
  Optimal Transport (OT) has become particularly prominent for DP enforcement \citep{Silvia_Ray_Tom_Aldo_Heinrich_John_2020}. 
  The key insight is that aligning group distributions can be formulated as finding Wasserstein barycenters. 
  This OT viewpoint yields geometric characterizations of the fairness-constrained problem — notably projections and geodesics in Wasserstein space — and naturally gives rise to plug-in recalibration procedures that can be applied to any trained regressor without retraining \citep{agarwal2019fairregressionquantitativedefinitions,plugin_regression_recalibration,gouic2020projectionfairnessstatisticallearning,chiappa2021continuousOT,chzhen2020fairregressionwassersteinbarycenters}, and provide a framework for calibrated risk–fairness trade-offs in regression \citep{chzhen2022minimaxframeworkquantifyingriskfairness}.

  \paragraph{Connection between CF and DP.} 
  To address the need for post-processing algorithms for exact and relaxed CF with rigorous statistical guarantees, we show that in the level-2 CF setting \citep{kusner2018counterfactual}, fixing the SCM reduces CF to \emph{conditional} DP -- a strictly stronger requirement than global DP. 
  This structural reduction allows us to leverage DP methodology to derive closed-form optimal predictors under both exact and relaxed constraints, which form the basis of our post-processing algorithms. 
  The CF-DP link has been source of debate in the literature \citep{rosenblatt2023counterfactualfairnessbasicallydemographic,silva2024counterfactualfairnessdemographicparity}, but only considering global DP; the latest result showing that global DP is insufficient for CF \citep{silva2024counterfactualfairnessdemographicparity}. 
  Our result is the first to connect (positively) CF to conditional DP and hence brings a new perspective to this debate.
  
  \section{Preliminary}
  
  \paragraph{Notation.}
  Bold letters (e.g., $\textbf{X}$, $\textbf{x}$) denote vectors, upper-case letters (e.g., $\textbf{X}$, $X$) random variables, 
  and lower-case letters (e.g., $\textbf{x}$, $x$) their realizations. 
  We write $[K] = \{1,\dots,K\}$ 
  and $\mathcal{P}_2(\mathbb{R})$ the set of probability measures on $\mathbb{R}$ with finite second moment.
  For all $\nu\in \mathcal{P}_2(\mathbb{R})$ we denote $F_\nu$ (\textit{resp.} $F^{-1}_\nu$) its c.d.f. (\textit{resp.} quantile function). 
 Finally, $X \stackrel{\smash{d}}{=} Y$ denotes equality in distribution.

  \subsection{Causal model}
  

  We consider \((\mathbf X,\U,S,Y)\in\Xset\times\ExoSet\times[K]\times\mathbb R\), where \(S\) is a discrete sensitive attribute with group weights \(w_s:=\mathbb P(S=s)\) for \(s\in[K]\). 
  The exogenous variable \(\U\in \ExoSet\) has law \(\nu\) and is independent of \(S\); for readability we assume $\ExoSet\!=\![0, 1]$ and \(\U\sim\mathrm{Unif}[0,1]\) (extensions to bounded densities on a compact support are straightforward). 
  The observed features \(\mathbf X\) are causally influenced by \((\U,S)\), and the outcome \(Y\) by \(\mathbf X\) (see \cref{fig:causal_graph}). 


  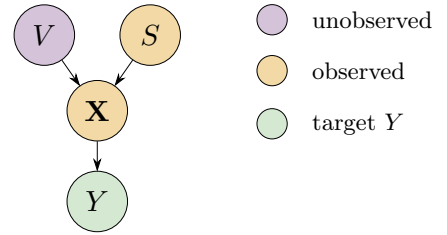
\begin{wrapfigure}{R}{0.37\textwidth}
    \vspace{-10pt} 
    \captionsetup{aboveskip=5pt,belowskip=0pt}
    \centering
    \begin{tikzpicture}[
      >={Stealth[length=1.6mm,width=1mm]},
      very thin,
      latent/.style   ={draw,circle,minimum size=8mm,inner sep=0pt,fill=edgePurple!50},
      observed/.style ={draw,circle,minimum size=8mm,inner sep=0pt,fill=obsOrange!50},
      prediction/.style ={draw,circle,minimum size=8mm,inner sep=0pt,fill=targetGreen!25}
    ]
      \node[latent]   (U)   at (-0.7, 1) {$\U$};
      \node[observed] (S)   at ( 0.7, 1) {$S$};
  
      \node[observed] (X)    at (0, 0) {$\mathbf{X}$};
      \node[prediction] (Y) at (0, -1.2) {$Y$};
  
      \draw[->] (U) -- (X);
      \draw[->] (S) -- (X);
      \draw[->] (X) -- (Y);
  
      \node[draw=none, inner sep=2pt, fit=(U)(S)(X)(Y)] (SCMbox) {};
  
      \matrix (legend) [right=8mm of SCMbox.east, yshift=6mm, anchor=west,
                        row sep=2mm, column sep=2mm, nodes={anchor=west, font=\footnotesize}]
      {
        \node[latent, minimum size=4mm] {};    & \node{unobserved}; \\
        \node[observed, minimum size=4mm] {};  & \node{observed}; \\
        \node[prediction, minimum size=4mm] {};& \node{target $Y$}; \\
      };
  
    \end{tikzpicture}
    \caption{Structural Causal Model.}
    \label{fig:causal_graph}
    \vspace{-10pt} 
  \end{wrapfigure}
  
  \paragraph{Illustration.} Throughout the paper, we consider the following example: $S$ is a gender; $\U$ is a student’s intrinsic ability; $\mathbf X$ aggregates
  coursework signals (e.g., homework, class participation of previous year) downstream of $(\U,S)$;
  and $Y$ is the future year course grade, downstream of $\mathbf X$. Both  $\mathbf X$  and $Y$ reflect systemic biases \citep{Colnet2025}.
  Two students with the same $\U$ can have different $Y$ due to this systemic bias and downstream randomness influencing $\mathbf X$ (e.g., stress handling).\\
  
  \noindent
    We adopt a non deterministic causal setting--\textit{causal uncertainty}--where $\U$ influences $\mathbf{X}$ in a non-deterministic manner.
    Under this assumption, the SCM is only approximately known, with unknown noise affecting the feature variable (see \cref{proof:DP_CF} for a discussion).
    Moreover, we make the following assumption.
    \begin{assumption}\label{ass:U_measurable}
      There exists a measurable map $\psi$ such that $\U= \psi(\mathbf{X}, S)$ a.s.
    \end{assumption}
    \vspace{-0.2cm}

    \noindent
    This assumption enables to recover the unobserved variable  $\U$ from observed features. 
    This is a common assumption in counterfactual fairness \citep{ kusner2018counterfactual, chiappa2018pathspecificcounterfactualfairness,zhou2025counterfactualfairnesscombiningfactual, tian2025counterfactualfairnessauxiliaryvariables} and is well-motivated in our context, as $\mathbf{X}$ (e.g., homework performance and class participation) provide substantial information about $\U$ (e.g., a student's underlying ability).
    In the following we assume that $\psi$ is known and do not address its estimation; in practice, $\psi$ can be estimated as in \citep{kusner2018counterfactual,ijcai2019p199_CF,Ma_2023,robertson2024fairpfntransformerscounterfactualfairness}.
  
  \subsection{Counterfactual fairness: a conditional invariance characterization.}
  
  Let $\mathcal{F}$ denote the class of all regression functions ${f:\mathbb{R}^d \times [K] \mapsto \mathbb{R}}$.
  For $f \in \mathcal{F}$ we denote $\hat{Y}= f(\mathbf{X}, S)$ the output of the predictor $f$.
  When writing $\hat Y_{S\leftarrow s'}$, we adopt a re-sampled-noise approach: (i) \textbf{abduction}---from $(\mathbf X=\mathbf x,S=s)$ we recover $v=\psi(\mathbf x,s)$; (ii) \textbf{action}—set $S:=s'$ and draw $\mathbf X'\sim \mathrm{Law}(\mathbf X\mid V=v,S=s')$; (iii) \textbf{prediction}—output $\hat Y_{S\leftarrow s'}:=f(\mathbf X',s')$. Counterfactual fairness is introduced by \cite{kusner2018counterfactual} and reformulated in \cite{When_worlds_collide}.
  
    \begin{definition}[Counterfactual fairness (CF) \citep{When_worlds_collide}] \label{def:CF}
      A predictor $\hat{Y} = f(\textbf{X}, S)$ is counterfactually fair if, for all $s, s' \in [K], \mathbf{x} \in \Xset$, $(\hat{Y}_{S \leftarrow s} \vert \mathbf{X}=\mathbf{x}, S=s) \stackrel{d}{=} (\hat{Y}_{S \leftarrow s'} \vert \mathbf{X}=\mathbf{x}, S=s).$
    \end{definition}
    \vspace{-3pt}
    \noindent
    Under this approach, CF requires that switching $S$ leaves the prediction distribution unchanged at fixed latent context $v$. This is natural: noise captures external disturbances (e.g., a distracting neighbor affecting homework completion), and the counterfactual is computed under a fresh noise realization (see Appendix~\ref{an:cf_operator} for a discussion).

    \begin{restatable}{proposition}{propCFDP}\label{prop:CF_DP}
      Under the SCM in Fig.\ref{fig:causal_graph}, CF (Def.~\ref{def:CF}) is equivalent to: for all $ v \in \ExoSet, s, s' \in [K]$,
      \begin{equation*}
        (f(\mathbf{X},S) \vert \U \!=\!v, S\!=\!s)\overset{d}{=}(f(\mathbf{X},S)\vert \U\!=\!v, S\!=\!s').
        \end{equation*}
    \end{restatable}
\vspace{-0.2cm}
  \noindent
  Proposition~\ref{prop:CF_DP}--proved in Appendix~\ref{proof:DP_CF}--shows that, under the SCM in \cref{fig:causal_graph} with re-sampled noise and \cref{ass:U_measurable}, counterfactual fairness (CF) reduces to demographic parity (DP) conditional on the variable $V=\psi(\mathbf X,S)$. 
  Note that \emph{conditional} DP differs from \emph{marginal} DP \citep{DP_calders}, which requires $(f(\mathbf X,S)\mid S{=}s)\overset{\smash{d}}{=}{(f(\mathbf X,S)\mid S{=}s')}$. 
  Our result is compatible with \cite{silva2024counterfactualfairnessdemographicparity}: global DP is generally insufficient for  CF (see Section \ref{subsec:synthetic}). Crucially, recasting CF as conditional DP lets us leverage existing statistical fairness methods \citep{chzhen2020fairregressionwassersteinbarycenters, chzhen2022minimaxframeworkquantifyingriskfairness} to address CF problems and provides a constructive post-processing approach with finite-sample, high-probability guarantees.
  
  
  \section{Closed-form solution for counterfactually fair regression}\label{sec:optimal_CF_regression}
  In this section, we derive the optimal fair predictor minimizing the risk under the CF constraint from Proposition~\ref{prop:CF_DP}. 
  We first define the conditional risk with respect to $v$ and formalize the fair regression objective. 
  We denote by $f^*(\textbf{X}, S) \!:= \!\mathbb{E}[Y \vert \textbf{X}, S]$ the Bayes optimal predictor.
  All expectations are taken over $(\mathbf X,\U,S,Y)$ unless stated otherwise.
  For $v \in \ExoSet$ and $f \in \mathcal{F}$, the conditional excess risk is
  \begin{equation*}
    \mathcal{R}_v(f):= \sum_{s=1}^K w_s \mathbb{E}[(f(\mathbf{X}, S)-f^*(\mathbf{X}, S))^2 \mid \U = v, S = s].
  \end{equation*}
  Our goal is to minimize the population excess risk $ \mathcal{R}(f) = \int_\ExoSet \mathcal{R}_v(f)\, d\nu(v),$
  subject to the CF constraint (relaxed in \cref{sec:relaxed_case}). 
  Specifically, we seek the optimal predictor 
    \begin{equation}
      \label{eq:optimization_pb}
      f^*_0\in \argmin_{f \in \mathcal{F}} \left\{\mathcal{R}(f) : f\text{ is CF}\right\}.
      \tag{CF}
    \end{equation}
    \textbf{Regularity conditions.} 
    For a predictor $f \in \mathcal{F}$, $\mu_{f \vert {v, s}} := Law(f(\mathbf{X}, S) \mid \U = v, S = s)$ is the distribution of the prediction of $f$ conditional on ($\U=v$ and $S=s$) and $F_{f \vert v, s}$ (resp. $F^{-1}_{f \vert v, s}$) is its c.d.f (resp. its quantile function). 
   The following assumptions ensure that the problem is well-posed. 
  \begin{property}[Continuous predictions with bounded support]
    \label{ass:distrib_densities}
    There exists $M>0$ such that 
    for all $(v,s)\in \ExoSet\times [K]$,
    $\mu_{f \vert v,s}$ is non-atomic 
    and supported in $[-M, M]$.
    \end{property}

    \noindent
    Assuming boundedness of  $Y$ is common in regression, and it further implies that the optimal prediction function is itself bounded. 
    In particular, for all $v \in \ExoSet, s \in [K]$, we have $\mu_{f \vert {v, s}} \in \mathcal{P}_2(\mathbb{R})$. 
    Continuity of the prediction distribution is also a common assumption \citep{NEURIPS2020_ddd80877, chzhen2020fairregressionwassersteinbarycenters},
    as it ensures that fairness can be achieved through deterministic post-processing algorithms. 
  
    \begin{property}[Lipschitz conditional c.d.f.]
    \label{ass:lipschitz_cdf}
    There exists $\Lcdf>0$ such that for all $s\in [K]$ and $t \in \mathbb{R}$, the map
    $v \mapsto F_{f\vert v,s}(t)$ is $\Lcdf$-Lipschitz in $v \in \ExoSet$.
    \end{property}

    \noindent
    The smooth variation of the c.d.f. with respect to the exogenous variable $v$, guaranteed by Property~\ref{ass:lipschitz_cdf}, is essential to control the unfairness gap. 
    Indeed, if the distribution of the predictions can change arbitrarily for nearby values of $v$, then treating individuals with similar values of $v$ in a similar way does not guarantee fairness, and there is no hope of achieving CF. 
    This assumption is reasonable, for instance, when the distribution of $X$ within a group conditional on $V = v$ varies smoothly with $v$, and when the prediction function $f$ is sufficiently regular. 
    In our motivating example, this assumption is natural: the distribution of scores for individuals with similar skill levels should remain close.
  
  \paragraph{Closed form of the fair predictor.}
  For $v \in \ExoSet $, we define the unfairness at $V=v$ of a predictor $f$ as  
  \begin{equation}
    \label{def:unfairness}
    \mathcal{U}_v(f) := \min_{\mu \in \mathcal{P}_2(\mathbb{R})}\sum_{s=1}^K w_s \mathcal{W}_2^2(\mu_{f \vert v, s}, \mu),
  \end{equation}
  where $\mathcal{W}_2$ is the Wasserstein-2 distance (see \cref{an:wasserstein_reminder}). Now, satisfying the criterion in Proposition~\ref{prop:CF_DP} is equivalent to having $\mathcal{U}_v(f)=0$  $\nu$-almost everywhere. Thus, defining 
    $\mathcal{U}(f) := \int_\ExoSet \mathcal{U}_v(f)d\nu(v),$
  Problem~\eqref{eq:optimization_pb} is equivalent to the OT projection
  \begin{equation}
    \label{eq:OT_pb}
    f_0^* \in \argmin_{f \in \mathcal{F}} \left\{\mathcal{R}(f) : \mathcal{U}(f) =0\right\}. 
  \end{equation}
  Building on this formulation, we extend \cite{chzhen2020fairregressionwassersteinbarycenters,gouic2020projectionfairnessstatisticallearning} to obtain the closed-form solution in Proposition~\ref{prop:fair_optimal_predictor}, proved in Appendix.~\ref{proof:reduction_OT}. 
  
  \begin{proposition}\label{prop:fair_optimal_predictor}
  Assume $f^*$ satisfies Property~\ref{ass:distrib_densities}. 
  Then, the solution of \eqref{eq:optimization_pb} admits the closed form
  \begin{equation*}
    f_0^*(\mathbf{x}, s) = \left( \sum_{s'=1}^K w_{s'} F^{-1}_{{f^* \vert v,s'}} \right) \circ F_{{f^* \vert {v , s}}} (f^*(\mathbf{x}, s)), 
  \end{equation*}
  where $v= \psi(\mathbf{x}, s)$, and its risk verifies $\mathcal{R}(f_0^* )\leq \mathcal{U}(f^*)$.
  \end{proposition}





  \paragraph{Illustration.} In the student example, the fair predictor can be interpreted as follows. Consider a student with features $\mathbf{x}$ and sensitive attribute $s$. First, we compute $F_{{f^* \mid v , s}}\!\left(f^*(\mathbf{x}, s)\right)$,
  that is, the proportion of students within the same group and at the same value $v = \psi(\mathbf{x}, s)$ whose predicted grade does not exceed $f^*(\mathbf{x}, s)$. This step amounts to ranking the student relative to others in the same group with the same value of $v$. Next, the prediction is obtained by mapping this rank to the barycenter distribution, which corresponds to the average grade across groups for students with the same rank conditional on $v$. As a consequence, two students who are equally strong relative to their respective groups and intrinsic value $v$ receive the same fair prediction under this transformation.
  

  \section{Estimation of the fair regression function} \label{sec:discrete_post_processing}

  In this section, given an i.i.d.\ sample \(\mathcal D_n\!=\!\{(\mathbf X_i, V_i = \psi(X_i, S_i), S_i)\}_{i=1}^n\) and a black-box predictor \(f^{bb}\) approximating \(f^*\), 
  we construct a post-processing estimator of the CF-optimal regressor \(f_0^*\) from Prop.~\ref{prop:fair_optimal_predictor}. 
  Since fairness is defined \emph{conditionally} on a \emph{continuous} exogenous variable \(\U\), we partition \(\ExoSet\) into \(L\) intervals and, within each interval, map each group's prediction distribution to its Wasserstein barycenter via a closed-form monotone quantile map.
  This post-processor applies to any \(f^{bb}\) and makes explicit the discretization trade-off: larger \(L\) reduces bias while increasing estimation error. 
  We provide finite-sample guarantees at rate \(\tilde{O}(n^{-1/3})\) and a data-driven rule for choosing \(L\). 
  Moreover, we establish a lower bound (Thm.~\ref{thm:mse_lower_bound}) showing that the rate \(\tilde{O}(n^{-1/3})\) is tight for estimating $f^*_0$.

  \subsection{Post-processing estimator}\label{sec:form_estimator}
  Our estimator is constructed in two steps: discretizing the exogenous variable $\U$ into $L$ intervals, then estimating conditional laws and applying the closed-form quantile transform from Proposition~\ref{prop:fair_optimal_predictor} within each interval.
  
  \paragraph{Discretization of $\mathcal{V}$.} We partition $\ExoSet$ into $L$ uniform intervals $\StratExoSet\! =\! \{\mathcal{I}_\ell\}_{\ell=1}^L$. 
  Each sample ${(\textbf{X}, \U, S) \!\in\! \mathcal{X}\!\times\!\ExoSet \!\times [K]}$ is assigned to a unique $\mathcal{I}_\ell$ based on the value of $\U$, enabling comparisons within the same interval. 
  Denote $p_\ell\! =\! \mathbb{P}(\U \!\in\! \mathcal{I}_\ell)$ the probability that $\U$ belongs to the $\ell^{\smash{th}}$ interval.

  \paragraph{Construction of the discretized plug-in estimator of $f_0^*$.}
  Given a black-box predictor $f^{bb}$ satisfying Property~\ref{ass:distrib_densities} and \ref{ass:lipschitz_cdf}
  , for each $(\ell, s)$ the conditional c.d.f. $\smash{F_{f^{bb} \vert \ell, s}}$ and the quantile function $\smash{F^{-1}_{f^{bb} \vert \ell, s}}$ depend only on the distributions of $(\U, S)$ and $\mathbf{X} \vert \U, S$, and thus can be estimated from unlabeled data.
  For each interval $\mathcal{I}_\ell$ and $s \in [K]$, we define $\smash{\Idx{\ell}{s}} =  \{i: S_i = s, \U_i \in \mathcal{I}_\ell\}$, $\smash{N_{\ell, s}\ :=\ |\Idx{\ell}{s}|}$, and $\mathcal{D}_{\ell, s}:=\{\mathbf{X}_{\ell, s}^i\}_{i\in [{N_{\ell, s}}]} = \{\mathbf{X}_i : i \in \Idx{\ell}{s}\}$; then $\mathbf{X}^i_{l,s}  \sim \mathbb P_{X\vert \U \in \mathcal{I}_\ell,S=s}$ and $n\!=\!\sum_{\ell \leq L} \sum_{s\leq K} N_{\ell, s}$. Without loss of generality, we assume that $N_{\ell, s}$ is positive and even.
  Let $\Idx{\ell}{s}^0, \Idx{\ell}{s}^1$ be a partition of $\smash{\Idx{\ell}{s}}$ into two equal parts of size $\smash{N_{\ell, s}/2}$. We use $\smash{\Idx{\ell}{s}^0}$ to estimate the quantiles $F^{-1}_{f^{bb}\vert \ell, s}$ and $\smash{\Idx{\ell}{s}}$  to estimate the c.d.f. $F_{f^{bb} \vert \ell, s}$.
  More precisely, $\mu_{f^{bb} \vert \ell, s}(A) := \nicefrac{1}{p_\ell} \int_{\mathcal{I}_\ell} \mu_{f^{bb} \vert v, s}(A) d\nu(v)$ is the conditional distribution of $f^{bb}(X) \vert V\in \mathcal{I}_{\ell}, S = s$, estimated as 
  $\widehat\mu_{f^{bb} \vert \ell,s}^j\! :=\! \nicefrac{2}{N_{\ell, s}}\sum_{i\in\Idx{\ell}{s}^j} \delta_{\,f(X_{\ell, s}^i, s)}$ for $j \in \{0, 1\}$.
  With the empirical distributions we define the empirical c.d.f. $\smash{\widehat F_{f \vert \ell,s}\!:=\!F_{\widehat\mu_{f \vert \ell,s}^0}}$ and quantile $\widehat F^{-1}_{f \vert \ell,s}\!:=\!F_{\widehat\mu_{f \vert \ell,s}^1}^{-1}$. 
We define $\smash{\widehat{f_{\ell}}}$ the local post-processing on interval $\ell$ and $\widehat{f}_L$ its global aggregation as 
\noindent
\begin{minipage}{.5\textwidth}
  \begin{equation*}\label{eq:local_post_processed}
    \widehat{f_{\ell}}(\mathbf{x}, s) := \left( \sum_{s'=1}^K w_{s'} \widehat{F}_{{f^{bb} \vert {\ell,s'}}}^{-1}\right) \circ \widehat{F}_{{f^{bb}\vert \ell,s}}(f^{bb}(\mathbf{x}, s)),
  \end{equation*}
\end{minipage}%
\begin{minipage}{.5\textwidth}
  \begin{equation}\label{eq:global_post_processed}
    \widehat{f}_L(\mathbf{x}, s) := \sum_{\ell=1}^{L} \widehat{f}_{\ell}(\mathbf{x}, s) \mathbbm{1}_{v \in \mathcal{I}_\ell},
  \end{equation}
\end{minipage}

\vspace{3pt}
\noindent where $v=\psi(\mathbf{x}, s)$. We provide implementation details in Algorithm~\ref{alg:post_processing_compact}. In practice, we estimate $w_s$ as $\hat{w}_s = \frac{1}{n}\sum_{i=1}^n \mathbbm{1}_{S_i=s}$. 
which converges at rate $\smash{\mathcal{O}(n^{-1/2})}$ negligible compared to $\smash{\mathcal{O}(n^{-1/3})}$.

    \begin{algorithm}[]
      \caption{Discretized Post-Processing}
      \label{alg:post_processing_compact}
      \KwIn{Data $\{(x_i,v_i=\psi(x_i,s_i),s_i)\}_{i=1}^n$; black-box $f^{bb}$.}
    
      \textbf{Initialization:} Compute $L^*$ (Thm.~\ref{thm:l_star}); build uniform partition $\{\mathcal I_\ell\}_{\ell=1}^{L^*}$ of $\mathcal V$.\; 
    
      \For{$\ell=1$ \KwTo $L^*$}{
        \For{$s\in[K]$}{
          $\Idx{\ell}{s}\gets \{\,i:\ v_i\in\mathcal I_\ell,\ s_i=s\,\}$, split into $\Idx{\ell}{s}^0\cup \Idx{\ell}{s}^1$.\;
          $Z^{j}\gets\{f^{bb}(x_i,s):i\in\Idx{\ell}{s}^{j}\}$, $j\in \{0,1\}$.\;

          $\widehat F^{-1}_{\ell,s}(u):=\inf\{y:\frac{2}{|\Idx{\ell}{s}|}\sum_{z\in Z^{0}}\mathbbm 1\{z\le y\}\ge u\}$ \tcc*[r]{Quantile (Fold $\Idx{\ell}{s}^0$)}
    
          $\widehat F_{\ell,s}(t):=\frac{2}{|\Idx{\ell}{s}|}\sum_{z\in Z^{1}}\mathbbm 1\{z\le t\}$ \tcc*[r]{c.d.f. (Fold $\Idx{\ell}{s}^1$)}
        }
        $\widehat F^{-1}_{\ell,\star}(u):=\sum_{s=1}^K w_s\,\widehat F^{-1}_{\ell,s}(u)$ \tcc*[r]{Empirical barycenter}
      }
    
      \textbf{Prediction on new $(x,v=\psi(x,s),s)$:}\;
      \qquad Find $\ell$ such that $v\in\mathcal I_\ell$, \KwRet $\widehat f_{L^*}(x,s)\triangleq \widehat F^{-1}_{\ell,\star}(\widehat F_{\ell, s}(f^{bb}(x, s)))$.\;
    \end{algorithm}

  \subsection{Bound on the unfairness of the post-processing estimator}

  The unfairness of the post-processed predictor $\widehat{f}_L$ satisfies the following bound. Proof of Theorem~\ref{thm:l_star} is deferred to Appendix~\ref{proof:l_star}.

  \begin{restatable}{theorem}{Lstar}  \label{thm:l_star}
    Assume $f^{bb}$ satisfies Properties~\ref{ass:distrib_densities} and \ref{ass:lipschitz_cdf}.
    Then, there exists a constant $c>0$, such that if $n\geq c$, for 
    $L^* = \left\lfloor \left( \tfrac{8\Lcdf^2 n}{K \log(2Kn)} \right)^{\smash{1/3}} \right\rfloor,$
  with probability $1 - (KL^*\exp(-n \tfrac{\min_s w_{s}}{8L^*}) + \frac{1}{n})$
  \begin{equation*}
      \mathcal{U}(\widehat f_{L^*})  \leq C\left(\frac{\log(n)}{n}\right)^{1/3},
  \end{equation*}
  where $C = 40M^2\left(\Lcdf K \log(2K)\right)^{1/3}$ and $c$ is a constant dependant of the constants of the problem.
  \end{restatable}

  \noindent
  CF is a pointwise requirement in \(v\): for each \(v\), the group-conditional prediction laws must coincide. 
  Replacing \(v\) by an interval \(\mathcal I_\ell\) averages over heterogeneous values; enforcing parity on \(\mathcal I_\ell\) thus compares non-identical subpopulations whenever within-interval group distributions differ, inducing \emph{discretization bias}. 
  Narrower intervals reduce this bias but increase per-interval \emph{estimation error}; \(L^*\) is chosen by balancing this tradeoff.
  More precisely, the Lipschitz constant $\Lcdf$ captures the problem's complexity: rapidly changing distributions require finer discretization. 
  However, finite-sample validity is constrained by the risk of ``empty cells,'' which is controlled by a term scaling exponentially in $-n\min_s w_s/L^*$. 
  Increasing $L^*$ starves local estimators; thus, practically, one must avoid over-discretizing when sample sizes are small or groups are imbalanced.
  
  \noindent
  In the remainder of the section we provide the main intuitions behind the results. 
  \begin{proof}[Sketch of proof.]

      First, we decompose the unfairness into a discretization \emph{bias} term and a statistical \emph{estimation error} term, and then bound each separately. 
      For interval $\ell$, the within-interval unfairness is defined as
  $\mathcal{U}_{\ell}(f) := \min_{\mu \in \mathcal{P}_2(\mathbb{R})}\sum_{s=1}^K w_s \mathcal{W}_2^2(\mu_{f \vert \ell, s}, \mu).$
  The unfairness of $\widehat f_L$ decomposes as
  \begin{align*}
    \mathcal{U}(\widehat f_L)\!
    =\sum_{\ell=1}^{L} p_\ell \left[\mathcal{U}_\ell(\widehat f_L) +\frac{1}{p_\ell} \smashoperator{\int_{\mathclap{\mathcal{I}_\ell}}}\!\!
    \left(\mathcal{U}_v(\widehat f_L)\!-\!\mathcal{U}_\ell(\widehat f_L)\!\right)\!d\nu(v)\right].
    \end{align*}

    \noindent
      \emph{Bias (right).} If $f^{bb}$ satisfies Properties~\ref{ass:distrib_densities}--\ref{ass:lipschitz_cdf}, 
    the global post-processed predictor $\widehat f_L(\mathbf x,s):=\sum_{\ell} \widehat f_\ell(\mathbf x,s)\,\mathbbm 1_{\{v\in\mathcal I_\ell\}}$ is $\Lcdf$-Lipschitz, with constant $\Lcdf$ on each interval $\mathcal{I}_\ell$. This implies that the unfairness is Lipschitz in $v$, which implies a discretization bias of order $\frac{1}{L}$. 
      \emph{Estimation error (left).} We show that the $W_2$ gap between post-processed group laws reduces to the corresponding Kolmogorov distance,
      before bounding it by the sum of per-group empirical c.d.f. deviations within the interval. 
      A uniform DKW bound over all cells \((\ell,s)\) yields empirical c.d.f. errors of order $N_{\ell,s}^{-1/2}$ with high probability. 
      A multinomial Chernoff bound then ensures $N_{\ell,s}\!\approx\!n\,p_\ell w_s$ simultaneously. 
      The sum of the $p_l$-weighted error rates is of order \(N_{\ell,s}^{-1/2}\!\!\approx\!(nL)^{1/2}\). 
      Choosing $L^* \!\!\approx \!n^{-1/3}$ yields the result.
      \end{proof}
  
  \subsection{Bound on the risk of the post-processing estimator.} \label{sec:risk_bound}
    We turn to establishing a bound on the risk of the estimated fair post-processed predictor at $L^*$ compared to the risk of the oracle post-processed predictor $f^{bb}_0$ defined as 
    \begin{equation*}
      f^{bb}_0(\mathbf x,s) := \left( \sum_{s'=1}^K w_{s'} F^{-1}_{{f^{bb} \vert v,s'}} \right) \circ F_{{f^{bb} \vert {v , s}}} (f^{bb}(\mathbf{x}, s)), 
    \end{equation*}
    for $v=\psi(\mathbf x,s)$. The proof of Theorem~\ref{cor:risk-exact-fair} is in  Appendix~\ref{proof:thm_risk_exact}.
    \begin{restatable}{theorem}{RiskExactFair}
      \label{cor:risk-exact-fair}
      Let $c$ and $L^*$ be as in Thm.~\ref{thm:l_star}. Assume $f^{bb}$ satisfies Properties~\ref{ass:distrib_densities}--\ref{ass:lipschitz_cdf}.
      If $n>c$,  for all $\lambda >0$, with probability $1 - (KL^*\exp(-n \frac{\min_s w_{s}}{8L^*}) + \frac{1}{n})$ we have

      \begin{equation*}
        \mathcal R(\widehat f_{L^\star})
        \;\le\;
        (1+\lambda)\mathcal R(f^{bb}_0) + C_{risk}\left(1 +\frac{1}{\lambda}\right)\cdot\left(\frac{\log(n)}{n}\right)^{1/3}.
      \end{equation*}
      \end{restatable}
      where $C_{risk} = \left(10M\right)^2\left(\Lcdf K \log(2K)\right)^{1/3}$.
      \begin{proof}[Sketch of proof]

        We condition on $(\U,S)$ and compare the estimator $\smash{\widehat f_L}$ to the \emph{oracle} fair post-processor $f^{bb}_0$. 
        The excess risk is decomposed using Young's inequality into:
        (i) the risk of the oracle $\smash{f^{bb}_0}$ (scaled by $1+\lambda$), and 
        (ii) a deviation term (scaled by $\smash{1+\lambda^{-1}}$), which further splits into \emph{estimation} and \emph{discretization} errors.
        The deviation terms are controlled as in Theorem~\ref{thm:l_star}. 
        Choosing the same optimal $\smash{L=L^*}$ balances these error terms, yielding the same rate of $\smash{\tilde{O}(n^{-1/3})}$ for the excess risk. 
    \end{proof}
\noindent
Theorem~\ref{cor:risk-exact-fair} guarantees that the learned post-processor $\smash{\widehat f_{L^*}}$ converges to the ideal population oracle $\smash{f^{bb}_0}$ at a rate of $\smash{n^{-1/3}}$. 
This has two key practical implications. 
First, if the black-box $\smash{f^{bb}}$ is already fair, then $\smash{f^{bb}_0=f^{bb}}$, meaning our method incurs no asymptotic accuracy penalty. 
Second, if $\smash{f^{bb}}$ is unfair, the excess risk is strictly limited to the inevitable transport cost required to align the groups (bounded by $\smash{\mathcal U(f^{bb})}$). 
Since the same $\smash{L^*}$ controls both unfairness and risk, our estimator simultaneously achieves fairness and matches the oracle's risk in the large-sample limit.

\noindent
We next show that the $\smash{n^{-1/3}}$ rate is optimal for estimating $\smash{f^*_0}$.
Proof of Thm~\ref{thm:mse_lower_bound} is deferred to App~\ref{proof:thm_estimation_only_minimax}.

\begin{restatable}{theorem}{MinimaxMSE}
  \label{thm:mse_lower_bound}
  Let $M, \Lcdf > 0$ and let $\mathcal{P}(M, \Lcdf)$ be the class of data-generating distributions satisfying Properties~\ref{ass:distrib_densities} and \ref{ass:lipschitz_cdf}.
  Then, there exists a constant $\smash{C_{MSE} > 0}$ depending only on $M, \Lcdf$, and $\{w_s\}_s$ such that the mean squared error of any estimator $\smash{\widehat{f}_n}$ relative to the optimal fair predictor $f_0^*$ is lower-bounded by
  \begin{equation*}
      \inf_{\widehat{f}_n} \sup_{\mathbb{P} \in \mathcal{P}(M, \Lcdf)} \mathbb{E}\left[ \|\widehat{f}_n - f_0^*\|^2 \right] \;\ge\; C_{MSE} \, n^{-1/3}.
  \end{equation*}
\end{restatable}
\begin{proof}[Sketch of proof]
We construct a "least favorable" family of distributions, $\mu_{f^*\vert v,s}$ supported on two disjoint intervals with gap $M$. 
In one group, the mass distribution between these intervals is uniform; in the other, it is perturbed by a high-frequency "bump" function $\theta(v) = 1/2 + \gamma \sum \eta_j \varphi_j(v)$ modulated by a hypercube parameter $\eta \in \{-1, 1\}^J$.
We show that the $L_2$ error of any fair estimator is lower-bounded by the $L_1$ estimation error of the mass transport $\delta(v)$ required to align the groups: $\|\widehat f - f^*_0\|^2 \gtrsim M^2 \|\widehat \delta - \delta_{opt}^{(\eta)}\|_{L^1}$.
Finally, we apply Assouad's Lemma to this mass estimation problem. 
The proof balances two conflicting requirements:
(i) \emph{Indistinguishability:} The bumps must be small enough ($\smash{\gamma \lesssim \sqrt{J/n}}$) so the hypotheses are statistically close in KL-divergence.
(ii) \emph{Regularity:} They must be smooth enough ($\smash{\gamma \lesssim J^{-1}}$) to satisfy the Lipschitz constraint $\smash{\Lcdf}$.
Optimizing the number of bumps $J$ under these constraints yields the resolution $\smash{J \asymp n^{1/3}}$, leading to the rate $\smash{n^{-1/3}}$.
\end{proof}

\noindent
Theorem~\ref{thm:mse_lower_bound} highlights that the $\smash{n^{-1/3}}$ rate is not an artifact of our discretization, but a fundamental limit of fair regression. 
The lower bound arises because estimating the optimal fair predictor requires non-parametric density estimation. 

\begin{restatable}{corollary}{MinimaxExcessRisk}
  \label{cor:risk_lower_bound}
Consider the setting of Theorem~\ref{thm:mse_lower_bound}. Let $u_n = C_1 (\log(n)/n)^{1/3}$ with $C_1$ a sufficiently small constant independent of $n$, and let $F_{u_n} = \{f : \mathcal{U}(f) \le u_n\}$ be the set of approximately fair predictors.
There exists a constant $C_R$ such that for any estimator $\widehat{f}_n$
  \begin{equation*}
      \sup_{\mathbb{P} \in \mathcal{P}(M, \Lcdf)} \left( \mathcal{R}(\widehat{f}_n) - \mathcal{R}(f^*_{u_n}) \right) \;\ge\; C_R \, n^{-1/3},
  \end{equation*}
  where $f^*_{u_n}$ is the risk-optimal predictor within $F_{u_n}$.
\end{restatable}

\noindent
Corollary~\ref{cor:risk_lower_bound} proves that enforcing the fairness constraint imposes a statistical cost. 
For this class of problems, the excess risk cannot vanish faster than the fairness from Thm.~\ref{thm:l_star}. 
Thus, the slow convergence rate is not a limitation of our specific method, but the intrinsic price of achieving fairness.

\section{Relaxed CF under unfairness budget}\label{sec:relaxed_case}
  
In practice, exact fairness is unreachable and can be undesirable. 
  We therefore introduce a relaxed variant of Problem~\eqref{eq:optimization_pb} that enables explicit control of the risk–unfairness trade-off, together with a post-processing estimator admitting finite-sample guarantees for both risk and unfairness.
  We adopt the relaxation approach of \cite{chzhen2022minimaxframeworkquantifyingriskfairness} to the constraint in Eq.~\eqref{eq:OT_pb} to allow for a controlled amount of unfairness $\alpha \in [0, 1]$ relative to the unfairness of the optimal unconstrained predictor $\mathcal{U}(f^*)$. 
  With a slight abuse of notation ($f^*_{u_n}$ targets an absolute statistical rate $u_n$, whereas $f^*_\alpha$ is a relative relaxation of the initial bias $\mathcal{U}(f^*)$ independent of $n$),
  we consider the relaxed problem
  \begin{equation}
    \label{eq:relaxed_OT_pb}
    f^*_\alpha\! \in\! \argmin_{f \in \mathcal{F}} \left\{\mathcal{R}(f):\! \mathcal{U}(f)\! \leq\! \alpha \mathcal{U}(f^*)\right\}.
    \tag{$\alpha$-CF}
  \end{equation}
  Problem \eqref{eq:relaxed_OT_pb} admits a closed-form solution expressed as a convex combination of parameter $\sqrt \alpha$ of the optimal unconstrained predictor $f^*$ and the fair optimal predictor $f_0^*$.
  The proof of Proposition~\ref{prop:closed_form_relaxed} can be found in Appendix~\ref{proof:reduction_OT}.
  
  \begin{restatable}{proposition}{ClosedFormRelaxed}\label{prop:closed_form_relaxed}
    Assume $f^*$ satisfies Property~\ref{ass:distrib_densities}. 
    Let $\alpha \in [0, 1]$. 
    For all $s \in [K]$, the solution of the relaxed problem \eqref{eq:relaxed_OT_pb} is unique and given by
    \begin{align*}\label{an:closed_form_relaxed}
      f_\alpha^*(\mathbf{x}, s)= \sqrt{\alpha} f^*(\mathbf{x}, s) + (1-\sqrt{\alpha}) f^*_0(\mathbf{x}, s).
    \end{align*}
    The optimal risk-unfairness pair is
      $\mathcal{R}(f_\alpha^*) = (1- \sqrt{\alpha})^2 \mathcal{R}(f^*_0) \enspace \text{and} \enspace \mathcal{U}(f_\alpha^*) = \alpha \mathcal{U}(f^*).$
  \end{restatable}
  

  \paragraph{Form of the relaxed estimator.}
  For $\alpha\in [0, 1]$, we define the relaxed post-processed estimator with $L$ intervals as
  \begin{equation*}\label{eq:relaxed_post_processed}
    \widehat{f}_{\alpha, L}(\mathbf{x}, s) := \sqrt{\alpha} f^{bb}(\mathbf{x}, s)+(1-\sqrt{\alpha}) \widehat{f}_L(\mathbf{x}, s).
  \end{equation*}
  Suppose we are given an unfairness budget $\mathcal B \geq 0$. Our goal is to select $\alpha^*$ such that the estimator satisfies 
   $\mathcal U(\widehat f_{\alpha^*, L^*}) \leq \mathcal B, \quad w.h.p.$
  However, we cannot simply invert the relation $\mathcal{U}(\widehat f_{\alpha, L}) \approx \alpha\,\mathcal{U}(f^{bb})$ to find $\alpha^*$, because the intrinsic unfairness of the black-box predictor, $\mathcal{U}(f^{bb})$, is unknown. 
  If we underestimate it, the chosen $\alpha$ might violate the budget. 

  \paragraph{Estimation of the unfairness of the black-box predictor.}
  To overcome this 
  issue we estimate the unfairness of 
  $f^{bb}$ using the empirical unfairness within each interval defined as
  \begin{equation*}
    \widehat{\mathcal{U}}_L(f^{bb}):=\sum_{\ell=1}^{L} \sum_{s=1}^K p_\ell w_s \mathcal W_2^2\left(\widehat\mu_{\ell,s},\widehat\mu_{\ell,\star}\right).
  \end{equation*}
  where $\widehat\mu_{\ell,\star}$ is the empirical Wasserstein barycenter of the group distributions on interval $\ell$.

  \noindent
  Theorem~\ref{thm:relaxed_l_star} provides an explicit choice of $\alpha^*$ that attains the budget with high probability whenever $\mathcal B$ is feasible.
  The proof of this result can be found in Appendix~\ref{proof:relaxed_l_star}.
  
    \begin{restatable}{theorem}{RelaxedLstar}
      \label{thm:relaxed_l_star} 
      Assume $f^{bb}$ satisfies \ref{ass:distrib_densities}--\ref{ass:lipschitz_cdf}.
      Let $c, C$, $L^*$ as in Thm.~\ref{thm:l_star} and $\delta^* = C (\log(n)/n)^{1/3}$.
      Given an unfairness budget $\mathcal{B} \!\geq \!\delta^*$, if $n\!>\!c$, for the choice $\alpha^* \!:= \!\min \{ 1, [(\mathcal B - \delta^*)/(2\widehat{\mathcal U}_{L^*}(f^{bb}))]^2 \}$, with probability $1 - ( KL^*\exp(-n \min_s w_{s} / 8L^*) + 1/n)$ we have $\mathcal{U}(\widehat f_{\alpha^*\!,L^*}) \leq \mathcal B.$
    \end{restatable}
  \begin{proof}[Sketch of proof]
    The result follows from the fact that the per-interval unfairness is convex under affine mixing of predictors.
    $\alpha^*$ is chosen so that the unfairness budget is exactly attained when possible, if the $\mathcal{B}$ is inferior to the bound on the unfairness, $\alpha^*=1$ and we cannot ensure this guarantee.
  \end{proof}

\noindent
  Algorithm~\ref{alg:post_processing_relaxed} summarizes the relaxed post-processor.
  Corollary~\ref{cor:risk-relaxed-fair}--proved in Appendix~\ref{proof:corollary_relaxed_risk}--is obtained by plugging $\widehat f_{\alpha^*\!\!, L^*}$ into the risk. 

  \begin{restatable}{corollary}{RiskRelaxedFair}\label{cor:risk-relaxed-fair}
    Under the assumptions of Theorem~\ref{thm:relaxed_l_star}, for any $\alpha \in [0, 1]$, the risk of the relaxed predictor $\widehat f_{\alpha, L^*}$ satisfies
    \begin{equation*}
      \mathcal{R}(\widehat f_{\alpha, L^*})  \leq (1-\sqrt{\alpha})\mathcal{R}(\widehat f_{L^*}) + \sqrt{\alpha}\mathcal{R}(f^{bb}).
    \end{equation*}
    In particular, this holds for the calibrated $\alpha^*$ returned by Algorithm~\ref{alg:post_processing_relaxed}.
  \end{restatable}
  
  \noindent
  Corollary~\ref{cor:risk-relaxed-fair} establishes that the risk bound interpolates linearly in $\sqrt{\alpha}$ between the fully fair estimator ($\alpha=0$) and the black-box ($\alpha=1$). 
  This explicitly characterizes the price of fairness: providing a smooth guarantee that bridges the gap between the best fair predictor and the most accurate one.

      \begin{figure}[t]
        \centering
        \begin{minipage}[b]{0.5\textwidth}
            \vspace{0pt} 
            \begin{algorithm}[H] 
                \caption{Relaxed Post-Processing}
                \label{alg:post_processing_relaxed}
                \small
                \KwIn{$\{(x_i,v_i,s_i)\}_{i=1}^n$; black-box $f^{bb}$; budget $\mathcal B$.}
              
                Compute $L^*$ and $\delta^*$ via Thm.~\ref{thm:l_star}.\;
                \lIf{$\delta^* \ge \mathcal B$}{$\alpha^* \gets 0$}
                \lElse{Compute $\smash{\widehat{\mathcal U}_{L^*}(f^{bb})}$, $\smash{\alpha^* \gets \left(\frac{\mathcal B - \delta^*}{2\widehat{\mathcal U}_{L^*}(f^{bb})}\right)^2}$}
                
                \textbf{Prediction on new $(x,v,s)$:}
                $\widehat f_{L^*}(x,s)$ via Alg.~\ref{alg:post_processing_compact}.\;
                \KwRet $ \sqrt{\alpha^*}\,f^{bb}(x,s) + (1-\sqrt{\alpha^*})\,\widehat f_{L^*}(x,s)$.\;
            \end{algorithm}
        \end{minipage}%
        \hfill
      \begin{minipage}[b]{0.47\textwidth}
        \captionof{table}{Law School Results relative to Base.}
        \label{tab:law_rel}
            \vspace{3.4pt} 
            \small 
            \begin{tabular*}{\linewidth}{@{\extracolsep{\fill}}llccc}
                \toprule
                Proxy & Method & RMSE $\downarrow$ & CF $\downarrow$ & DP $\downarrow$ \\
                \midrule
                     & \texttt{Fair K} & 1.0861 & \textbf{0.0015} & \textbf{0.0015} \\
                SCM  & WFR             & \textbf{1.0518} & 0.0880 & 0.0046 \\
                     & Ours            & 1.0539 & 0.0091 & \textbf{0.0015} \\
                \midrule
                     & \texttt{Fair K} & 1.0033 & 0.4000 & 0.6871 \\
                VAE  & WFR             & 1.0239 & 16.6000 & \textbf{0.0068} \\
                     & Ours            & \textbf{1.0023} & \textbf{0.2000} & 0.7687 \\
                \bottomrule
            \end{tabular*}
        \end{minipage}
    \end{figure}

\section{Experiments}\label{sec:experiments}

In the main text, we evaluate on a synthetic and real-world Law School (LSAC) dataset, comparing our post-processing method against two baselines:
\textbf{Counterfactual Fairness (\texttt{Fair K})} \cite{kusner2018counterfactual}: a causal approach that predicts $Y$ using \textit{only} the exogenous variable $V$, effectively discarding $X$; and
\textbf{Wasserstein Fair Regression (WFR)} \cite{chzhen2020fairregressionwassersteinbarycenters}: a post-processing method that enforces \textit{Global} Demographic Parity (ignoring $V$). 
Appendix~\ref{an:additional_results} details further evaluations (e.g., Communities \& Crime, in-processing baselines, multi-group extensions, and robustness). Code is in the supplement.


\paragraph{Estimation of the Exogenous Variable.}
Our framework relies on the availability of a proxy for the exogenous variable $V$. 
While identifying $\U$ from observational data is a significant challenge \citep{kusner2018counterfactual, ijcai2019p199_CF}, we rely on existing methods to estimate it either via structural causal models as in \cite{kusner2018counterfactual} or variational autoencoders (VAE). Details about the estimation of $\U$ can be found in Appendix~\ref{an:estimation_V}.

\paragraph{Metrics.}
We assess predictor quality using three metrics.
\textbf{Risk (RMSE):} The root mean squared error on the test set.
\textbf{CF:} Our primary metric $\smash{\mathcal{U}(f)}$, the expected squared $\smash{\mathcal W_2}$ distance between the conditional distributions of predictions given $V$ estimated as stated before.
\textbf{DP:} The Global Demographic Parity violation, the squared $\mathcal{W}_2$ distance between the global distributions of predictions for each group, ignoring $V$. Metrics are stated relative to the base model (<1 indicate improvement).

\subsection{Synthetic Data}\label{subsec:synthetic}



\paragraph{Setup.} We generate samples with $V\!\sim\!\mathcal{U}(0,1)$ and $S\!\sim\!\mathcal{B}(0.5)$. 
Let $S'\!=\!2S\!-\!1$. 
The feature vector is defined as $\textbf{X}\!=\!(X, V)$ where $X\!=\!S'V\!+\!\varepsilon_x$ and the response is $Y\!=\!X\!+\!\varepsilon_y$, with $\varepsilon_x\!\sim\!\mathcal{U}(-0.5,0.5)$ and $\varepsilon_y\!\sim\!\mathcal{U}(-0.01,0.01)$. 
In this setting, global DP aligns high values of $X$ across both groups, whereas CF requires matching high values of $X$ in one group with low values of $X$ in the other. 


\begin{figure}[]
  \centering
  \makebox[\textwidth][c]{%
      \begin{minipage}[b]{0.33\textwidth}
          \centering
          \includegraphics[
              width=\linewidth,
              keepaspectratio,
          ]{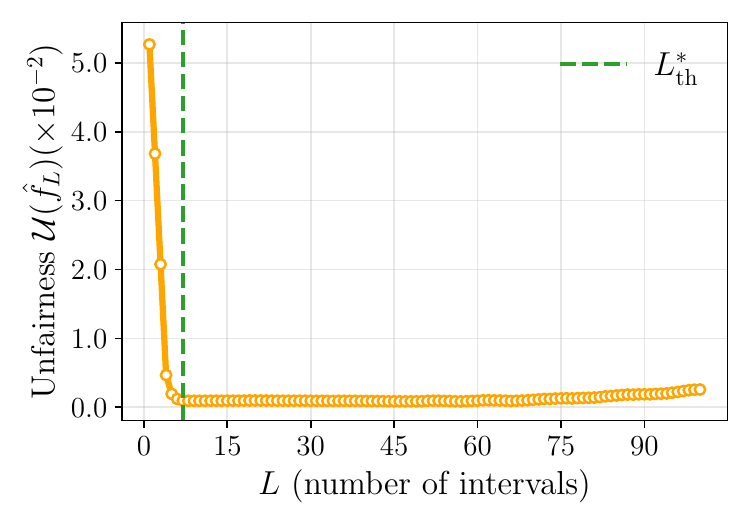}
      \end{minipage}
      \hfill 
      \begin{minipage}[b]{0.33\textwidth}
          \centering
          \includegraphics[
              width=\linewidth,
              keepaspectratio,
          ]{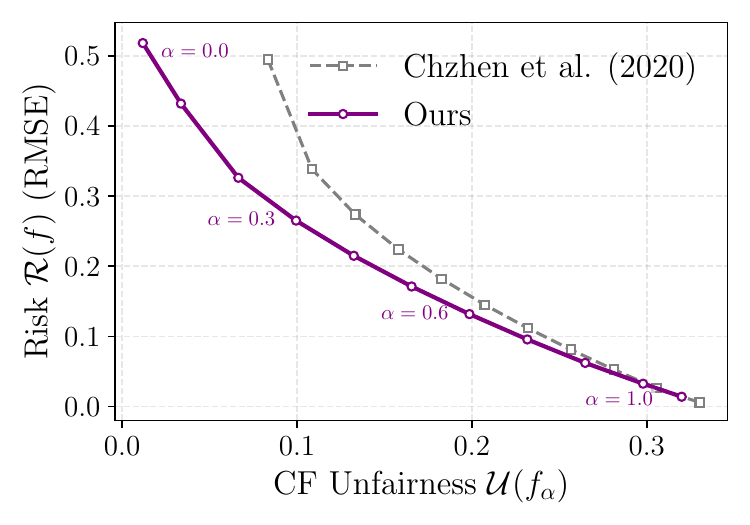}
      \end{minipage}
      \hfill
      \begin{minipage}[b]{0.33\textwidth}
          \centering
          \includegraphics[
              width=\linewidth,
              keepaspectratio,
          ]{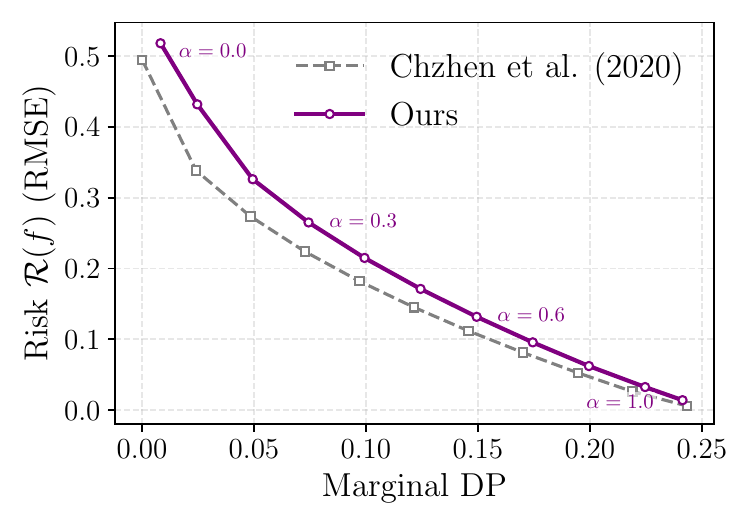}
      \end{minipage}
  }
  \caption{\textbf{Synthetic Dataset.} 
  \textbf{Left:} Unfairness drops and stabilizes near the theoretical $L^*_{\mathrm{th}}$. 
  \textbf{Middle:} Our method (purple) dominates WFR (gray) on the CF tradeoff frontier. 
  \textbf{Right:} WFR exhibits a better DP tradeoff frontier.}  
  \label{fig:synthetic_combined}
\end{figure}


\noindent
\begin{minipage}[c]{0.62\textwidth}
  \textbf{Results.} \cref{fig:synthetic_combined} confirms the theory.
\textbf{Left:} Unfairness drops as discretization bias vanishes, stabilizes near the optimum $L_{th}^\ast$ before estimation error dominates.
\textbf{Middle:} Our estimator dominates WFR on the CF frontier, recovering accuracy as $\alpha$ increases. 
\textbf{Right:} Conversely, when evaluated on global DP, WFR achieves a lower risk tradeoff. 
This is expected, as WFR optimizes for marginal parity, whereas our method enforces CF.
\end{minipage}\hfill
\begin{minipage}[c]{0.36\textwidth}
  \centering
  \captionof{table}{Performance on synthetic data relative to Base.}
  \begin{tabular}{@{}lccc@{}}
    \toprule
    Method & RMSE $\downarrow$ & CF $\downarrow$ & DP $\downarrow$ \\
    \midrule
    \texttt{Fair K} & 110.67 & \textbf{0.00} & \textbf{0.00} \\
    WFR             & \textbf{85.34} & 0.25 & \textbf{0.00} \\
    Ours            & 100.99 & \textbf{0.00} & \textbf{0.00} \\
    \bottomrule
\end{tabular}
  \vspace{0.3em}
  \captionsetup{hypcap=false}
  \label{tab:synthetic_rel}
\end{minipage}\\

\noindent
Table~\ref{tab:synthetic_rel} details this performance relative to the Base Model:
\textbf{\texttt{Fair K}} achieves perfect fairness (0.00) but incurs high error (relative RMSE 110.67) by discarding useful information in $X$. 
While \textbf{WFR} achieves the lowest error (RMSE 85.34), it violates counterfactual fairness (CF 0.25). 
\textbf{Our Method} matches the perfect fairness of \texttt{Fair K} while achieving lower error (RMSE 100.99).

\subsection{Law School Dataset}

\paragraph{Setup.}
Following \cite{kusner2018counterfactual}, we use the LSAC dataset recording the admission credentials and first-year academic performance of over 21,000 students across 163 US law schools. 
We predict First-Year Average (ZFYA) from race ($S$), LSAT, and UGPA. We infer the latent variable $V$ using their Level-2 method. 
Focusing on Black ($S=1$) vs.\ White ($S=0$) students, we train a penalized quadratic regression $f^{bb}$, which satisfies our regularity assumptions.


\paragraph{Interpretation.}
Table~\ref{tab:law_rel} (SCM) shows that baselines occupy extremes: \texttt{Fair K} sacrifices accuracy for safety (RMSE 1.086), while WFR preserves accuracy (1.052) but fails to reduce conditional unfairness (0.088). Our method bridges this gap, matching WFR's accuracy (1.054) while reducing unfairness by an order of magnitude (0.009). In the VAE setting (bottom), baselines degrade significantly, whereas our method remains stable, achieving the lowest RMSE (1.002) and unfairness (0.200).


\section{Discussion} \label{sec:discussion}
We propose a model-agnostic post-processing framework for counterfactual fairness with finite-sample guarantees and a calibrated fairness-utility trade-off. Our results come with the following limitations. First, the fairness guarantee hinges on the validity of our causal model; if the intrinsic ability $V$ remains entangled with the sensitive attribute, or if the map $\psi$ is poorly estimated, bias may persist despite statistical alignment. Second, discretizing the continuous latent variable forces a trade-off: in low-data regimes, we must use wider intervals to ensure statistical stability, which reduces the precision of our fairness method.

\section*{Acknowledgments and Disclosure of Funding}

This work was supported by the Hi! PARIS Center and by the French National Research Agency (ANR) through grant ANR-23-CE23-0002, the PEPR IA FOUNDRY project (grant ANR-23-PEIA-0003), and the France 2030 program (grant ANR-23-IACL-0005).

\clearpage
\bibliographystyle{plainnat}
\bibliography{biblio}
\clearpage
\appendix
\usetikzlibrary{matrix}

\onecolumn
\clearpage
\appendix
\setcounter{page}{1}

\addcontentsline{toc}{section}{Appendix}

\startcontents[appendix]
\begin{center}
    \Large\bfseries Appendix Table of Contents
\end{center}
\hrule
\vspace{1em}
\printcontents[appendix]{}{1}{\setcounter{tocdepth}{2}}
\vspace{1em}
\hrule
\vspace{3em}

%
%






\def\add#1{\textcolor{blue}{#1}}
\def\del#1{\textcolor{red}{\sout{#1}}}
\clearpage
\section{Useful background}

This section introduces useful notation and tools for the proofs.
We first provide a reminder of the notations used in the paper (\S\ref{an:notations}). 
Next, we revisit the one-dimensional Wasserstein--2 space—its quantile formula, barycenters, and an abstract geometric lemma (\S\ref{an:wasserstein_reminder}). 
Then, we sketch how optimal transport encodes demographic parity, including a closed-form projection and a risk–unfairness identity (\S\ref{an:OT_fairness}). 
In parallel, we recall a few basic facts about c.d.f.s (\S\ref{an:cdf_properties})
and generalized inverses. 

\subsection{Additional notations}\label{an:notations}

The conditional distribution of the features $\mathbf{X}$ given $\U=v$ and $S=s$ is denoted by $\mu_{\mathbf X\vert v, s}$, and the conditional distribution of $\mathbf{X}$ given $S=s$ by $\mu_{\mathbf X\vert s}$.

For all predictor $f$, we denote $\mu^*_{f \vert v}$ the Wasserstein-2 barycenter of the group-conditional distributions $(\mu_{f \vert v, s})_{s \in[K]}$ (see Section~\ref{an:wasserstein_reminder} for the definition of barycenters).
In particular, for $v \in \ExoSet$, we $\mu^*_v$ is the Wasserstein-2 barycenter of the Bayes optimal predictor's conditional distributions $(\mu_{f^* \vert v, s})_{s \in[K]}$.

For predictors $f$ satisfying Property~\ref{ass:distrib_densities}, the conditional distributions $\mu_{f \vert v, s}$ are supported on a bounded interval for all $v \in \ExoSet$ and $s \in [K]$, with the support contained in $[-M, M]$ for some constant $M > 0$. 
This ensures that ${\sup_{v, s} \int x^2 d\mu_{f \vert {v, s}}(x) \leq M^2}$ almost surely.

Throughout, we assume that the Bayes optimal predictor $f^*$ satisfies Properties~\ref{ass:distrib_densities}~and~\ref{ass:lipschitz_cdf}.

\subsection{The Wasserstein-2 space.}\label{an:wasserstein_reminder}

We introduce fundamental concepts from one-dimensional optimal transport theory that will be used in our proofs. 
For a comprehensive introduction to optimal transport, we refer to \cite{Santambrogio2015}.

\paragraph{Wasserstein-2 distance.} The Wasserstein distance provides a metric between probability measures based on optimal transport. 
\begin{definition}[Wasserstein--2 distance]
  For two univariate probability measures $\mu$ and $\nu$, 
    the squared Wasserstein--2 distance is defined as
    \begin{equation}
      \label{def:wasserstein}
      \mathcal W_2^2(\mu,\nu)
      \coloneqq \inf_{\gamma \in \Gamma_{\mu,\nu}}
      \int_{\mathbb R\times \mathbb R} |x-y|^2 \, \mathrm d\gamma(x,y),
    \end{equation}
    where $\Gamma_{\mu,\nu}$ denotes the set of all couplings between $\mu$ and $\nu$, i.e., probability measures $\gamma$ on
    $\mathbb R\times \mathbb R$ with marginals $\mu$ and $\nu$: for all measurable sets $A,B\subset\mathbb R$,
    $\gamma(A\times \mathbb R)=\mu(A)$ and $\gamma(\mathbb R\times B)=\nu(B)$.
    \end{definition}

    The coupling $\gamma$ that achieves the infimum is called the optimal coupling.
    In 1D the Wasserstein-2 distance has a closed-form expression in terms of quantile functions, as stated in the following proposition that can be found for example in Proposition 2.17 from \cite{Santambrogio2015}.

\begin{lemma}
  \label{prop:wasserstein_quantiles}
  Let $\mu$ and $\nu$ be two univariate probability measures.
  Then the squared Wasserstein--2 distance between $\mu$ and $\nu$ can be expressed as
  \begin{equation*}
    \mathcal W_2^2(\mu,\nu) \;=\; \int_0^1 \left\vert F^{-1}_\mu(t) - F^{-1}_\nu(t) \right\vert^2 \, \mathrm dt.
  \end{equation*}

\end{lemma}

\paragraph{Barycenters in the Wasserstein-2 space.}
We introduce the notion of Wasserstein-2 barycenters and properties of the barycenters of a metric space.
The following lemma characterizes the Wasserstein-2 barycenter of a finite collection of univariate probability measures in terms of their quantile functions stated for instance in Section 5.5.5 of \citep{Santambrogio2015} or Section 6.1 of \cite{agueh_carlier}.
\begin{lemma}[Wasserstein-2 barycenter in 1D]\label{prop:cumulative_bary}
  Let $\mu_1,\ldots,\mu_K$ be univariate probability measures in $\mathcal{P}_2(\mathbb{R})$. For all weights $w_1,\ldots,w_K \ge 0$ with
  $\sum_{s=1}^K w_s = 1$, define
  \begin{equation*}
    \mu^* \in \arg\min_{\mu}\; \sum_{s=1}^K w_s\, \mathcal W_2^2(\mu_s,\mu).
  \end{equation*}
  Then the cumulative distribution function of $\mu^*$ is
  \begin{equation*}
    F_{\mu^*}(t) \;=\; \Bigg(\sum_{s=1}^K w_s\, F^{-1}_{\mu_s}\Bigg)^{-1}(t) \quad t \in [0, 1].
  \end{equation*}
  \end{lemma}

\begin{remark}
  Although this result is originally stated for non-atomic measures in \cite{agueh_carlier,Santambrogio2015}, the extension to general measures with atoms is straightforward using the generalized inverse definition of the quantile function.
      Let $\mu_1,\ldots,\mu_K \in \mathcal{P}_2(\mathbb{R})$ and weights $w_1,\ldots,w_K \ge 0$ with $\sum_{s=1}^K w_s = 1$.
    For a probability measure $\mu$, write $F_\mu$ for its c.d.f. and
    \begin{equation*}
    F_\mu^{-1}(u) \;:=\; \inf\{x \in \mathbb{R} : F_\mu(x) \ge u\}, \qquad u \in (0,1),
  \end{equation*}

    for its (left-continuous) generalized quantile function.
    Define
    \begin{equation*}
    q^*(u) \;=\; \sum_{s=1}^K w_s \, F^{-1}_{\mu_s}(u), \qquad u \in (0,1),
  \end{equation*}

    and let $\mu^*$ be the pushforward of $\mathrm{Unif}(0,1)$ by $q^*$.
    Then $\mu^*$ is a minimizer of
    \begin{equation*}
    \mu \;\longmapsto\; \sum_{s=1}^K w_s\, \mathcal{W}_2^2(\mu_s,\mu),
  \end{equation*}

    and it is unique. Equivalently,
    \begin{equation*}
    F_{\mu^*} \;=\; \Big(\sum_{s=1}^K w_s\, F^{-1}_{\mu_s}\Big)^{-1}
  \end{equation*}

    in the generalized-inverse sense. Moreover $F^{-1}_{\mu^*} = q^*$ and,
    \begin{equation*}
    \mathcal{W}_2^2(\mu_s,\mu^*) \;=\; \int_0^1 \big(F^{-1}_{\mu_s}(u) - F^{-1}_{\mu^*}(u)\big)^2 \, du
    \quad \text{for each } s\in[K].
  \end{equation*}

    In one dimension, for all $\mu,\nu \in \mathcal{P}_2(\mathbb{R})$,
    \begin{equation*}
    \mathcal{W}_2^2(\mu,\nu) \;=\; \int_0^1 \big(F^{-1}_\mu(u) - F^{-1}_\nu(u)\big)^2 \, du,
  \end{equation*}

    with $F^{-1}$ the generalized inverse (works with atoms). Hence the barycenter problem is
    \begin{equation*}
    \inf_{q} \int_0^1 \sum_{s=1}^K w_s\, \big(F^{-1}_{\mu_s}(u) - q(u)\big)^2\,du,
  \end{equation*}
    where the infimum is over nondecreasing, left-continuous $q \in L^2(0,1)$ (each such $q$ corresponds to a $\mu$ via pushforward of $\mathrm{Unif}(0,1)$).
    Pointwise in $u$, the strictly convex quadratic problem is minimized at
    $q^*(u) = \sum_{s=1}^K w_s F^{-1}_{\mu_s}(u)$.
    Since each $F^{-1}_{\mu_s}$ is nondecreasing and left-continuous, so is $q^*$; thus it is feasible and globally optimal. Strict convexity in $q$ yields uniqueness in $L^2$, hence uniqueness of $\mu^*$. The cost identity follows by substitution.
  \end{remark}

  We recall the definition of the barycenter property of a metric space.
  \begin{definition}{(Barycenter property)} \label{def:barycenter_property}
    We say that a metric space $(\mathcal{M}, d)$ satisfies the barycenter property if for all weight vector $\mathbf{w} \in \Delta^{K-1}$ and tuples $\mathbf{a} = (a_1, \dots, a_K) \in \mathcal{M}^K$ there exists a barycenter
    \begin{equation*}
    C_{\mathbf{a}_{\mathbf{w}}} \in \arg\min_{C \in \mathcal{M}} \sum_{s=1}^K w_s d^2(a_s, C).    
  \end{equation*}
    Moreover, for all tuple $\mathbf{a} = (a_1, \dots, a_K) \in \mathcal{M}^K$ we denote by $C_{\mathbf{a}_{\mathbf{w}}}$ a barycenter of $\mathbf{a}$ weighted by $\mathbf{w} \in \Delta^{K-1}$.
    \end{definition}  

    For metrics spaces satisfying this property, the following abstract geometric Lemma holds (Lemma 4.3 in \cite{chzhen2022minimaxframeworkquantifyingriskfairness}).
    \begin{lemma}[Abstract geometric lemma] \label{lem:geom_lemma}
      Let $(\mathcal{M}, d)$ be a metric space satisfying the barycenter property from Definition~\ref{def:barycenter_property}. Let $\mathbf{a} = (a_1, \dots, a_K) \in \mathcal{M}^K$, $\mathbf{w} = (w_1, \dots, w_K) \in \Delta^{K-1}$ and let $C_{\mathbf{a}_{\mathbf{w}}}$ be a barycenter of $\mathbf{a}$ with respect to weights $\mathbf{w}$. 
      For a fixed $\alpha \in [0,1]$, assume that there exists $\mathbf{b} = (b_1, \dots, b_K) \in \mathcal{M}^K$ which satisfies
      \begin{align}
          d(a_s,C_{\mathbf{a}_{\mathbf{w}}}) &= d(a_s, b_s) + d(b_s, C_{\mathbf{a}_{\mathbf{w}}}), \quad s = 1, \dots, K, \tag{P1} \label{P1} \\
          d(b_s, a_s) &= (1 - \sqrt{\alpha}) d(a_s, C_{\mathbf{a}_{\mathbf{w}}}), \quad s = 1, \dots, K. \tag{P2} \label{P2}
      \end{align}
      Then, $\mathbf{b}$ is a solution of
      \begin{equation*}
      \inf_{\mathbf{b} \in \mathcal{X}^K} \sum_{s=1}^K w_s d^2(b_s, a_s) \quad \text{subject to} \quad \sum_{s=1}^K w_s d^2(b_s, C_{\mathbf{b}_\mathbf{w}}) \leq \alpha \sum_{s=1}^K w_s d^2(a_s, C_{\mathbf{a}_\mathbf{w}}).
    \end{equation*}
    \end{lemma}
    Informally, Property~\eqref{P1} says that $b_{(s)}$ lies on the $W_2$-geodesic joining $a_s$ and $C_{\mathbf{a}_{\mathbf{w}}}$, while Property~\eqref{P2} fixes its position: $b_{(s)}$ is $(1-\sqrt{\alpha})$ of the way from $a_s$ toward $C_{\mathbf{a}_{\mathbf{w}}}$ along that geodesic.

    \begin{remark}\label{remark:barycenter_W2}
    The metric space $(\mathcal{P}_2(\mathbb{R}), \mathcal{W}_2)$ satisfies the barycenter property, we refer the interested reader to e.g. \cite{agueh_carlier} Proposition 2.3 for details about the existence of Wasserstein barycenters.   
  \end{remark}

\subsection{Optimal transport for fairness} \label{an:OT_fairness}

Optimal transport theory has been successfully applied to address fair prediction problems under demographic parity constraints, as demonstrated in works such as \cite{chzhen2020fairregressionwassersteinbarycenters} and \cite{gouic2020projectionfairnessstatisticallearning}. 
We recall some of their key findings, which serve as the basis for our approach to solving regression problems under causality-based fairness constraints. 
The optimization problem for achieving exact fairness under the demographic parity constraint is formulated as
  \begin{equation}
    \label{eq:OT_pb_DP}
    f_{DP}^* \in \arg\min_{f} \left\{\mathbb{E}\left[(f(\mathbf{X}, S)-f^*(\mathbf{X}, S))^2\right] : \min_{\mu \in \mathcal{P}_2(\mathbb{R})} \sum_{s=1}^K w_s \mathcal{W}_2^2(\mu_{f \vert s}, \mu)=0\right\}.
  \end{equation}
For the predictor $f_{DP}^*$ that solves the optimization problem above, its risk is equal to the unfairness of the optimal predictor $f^*$ for the unconstrained problem. 
Furthermore, $f_{DP}^*$ has a closed-form expression, as demonstrated in Theorem~2.3 of \cite{chzhen2020fairregressionwassersteinbarycenters} and Theorem~6 of \cite{gouic2020projectionfairnessstatisticallearning}.
\begin{theorem} \label{thm:risk_opt_unfairness_bayes}

 For $s \in [K]$, let $\mu_{f^*\vert s}$ be a univariate non-atomic measures and let $w_s = \mathbb{P}(S=s)$ then
  \begin{equation*}
    \mathbb{E}[(f_{DP}^*(\mathbf{X}, S)-f^*(\mathbf{X}, S))^2 ] = \min_{\mu \in \mathcal{P}_2(\mathbb{R})} \sum_{s=1}^K w_s \mathcal{W}_2^2(\mu_{f \vert s}, \mu).
  \end{equation*}
  Additionnally, the expression of the optimal fair predictor is given by 
  \begin{equation*}
    f_{DP}^*(x, s) = \left( \sum_{s' \in [K]}w_{s'} F^{-1}_{f^* \vert s'}\right) \circ F_{f^* \vert s}(f^*(x, s)).
  \end{equation*}
\end{theorem}

  The following Lemma (Theorem~3 of \cite{gouic2020projectionfairnessstatisticallearning}) highlights the relationship between the risk of a predictor and the discrepancy between its induced distributions and those of the optimal unconstrained predictor.

    \begin{lemma}\label{lem:B1}
      Let $f\in\mathcal{F}$ and assume that for all $v\in\ExoSet$ and $s\in[K]$, the measure $\mu_{f^*\mid v,s}$ is non-atomic. Then
      \begin{equation*}
      \mathcal{R}_v(f)\ \ge\ \sum_{s=1}^{K} w_s\, \mathcal{W}_2^2\!\big(\mu_{f\mid v,s},\,\mu_{f^*\mid v,s}\big).
      \end{equation*}
      \end{lemma}
      
\subsection{Properties of  the c.d.f.} \label{an:cdf_properties}
The following result establishes the almost sure equality between the generalized inverse of the quantile function and the cumulative distribution function (c.f. Lemma 21.1 in \citep{Vaart_1998} for the proof).
\begin{lemma}\label{prop:equ_cdf_quantile}
  Let $\mu$ be a probability measure on $\mathbb{R}$ with c.d.f. $F_\mu$ and quantile function $F_\mu^{-1}$.

  Then, for every $x\in\mathbb{R}$
  \begin{equation*}
  \left(F_\mu^{-1}\right)^{-1}(x)=F_\mu(x-):=\sup_{y<x}F_\mu(y),
\end{equation*}
where $(F_\mu^{-1}\big)^{-1}$ denotes the generalized inverse of $F_\mu^{-1}$.
  In particular, if  $F_\mu$ is continuous  (e.g. $\mu$ is non-atomic), then
  $\big(F_\mu^{-1}\big)^{-1}(x)=F_\mu(x)$ for all $x$.
  \end{lemma}

\section{Proof of Proposition~\ref{prop:CF_DP}} \label{proof:DP_CF}

Before proving the proposition, we introduce the necessary notation and definitions from the causality literature (see, e.g., \cite{Pearl_2009} for further details).

\subsection{Counterfactual/Interventional operator.} \label{an:cf_operator}

\paragraph{Noise–re-sampled counterfactual.}
In a stochastic SCM with $\mathbf{X}=\phi(\U,S,\varepsilon_X)$ where $\varepsilon_\mathbf{X}$ is the exogenous noise on $\mathbf{X}$ and $Y=f(\mathbf{X},S)$,
define the counterfactual with re-sampled noise as
\begin{equation*}
\widetilde{Y}_{S\leftarrow s'} \;:=\; f\!\big(\phi(\U,s',\varepsilon'_X),\, s'\big),
\end{equation*}
where $\varepsilon'_X$ is an independent copy of $\varepsilon_\mathbf{X}$. 
This differs from the standard unit-level counterfactual (which replays the realized noise).

    With re-sampled noise, we generally lose almost-sure consistency
    \begin{equation*}
    \widetilde{Y}_{S\leftarrow s} \neq Y \quad\text{in general,}
    \end{equation*}
    but we retain \emph{consistency in distribution},
    $\widetilde{Y}_{S\leftarrow s}\stackrel{d}{=}Y$, and \emph{invariance in distribution}:
    if $Y$ is not a descendant of $S$, then $\widetilde{Y}_{S\leftarrow s'}\stackrel{d}{=}Y$.
    This choice intentionally averages out idiosyncratic shocks:
    even when $s'=s$, a new draw of noise may yield a different value, which is desired in our setting.\\
    \begin{remark}
      By construction, we have $\U_{S \leftarrow s}= \U_{S \leftarrow s'}$.\\
      \end{remark}

This setting is in adequation with the \textbf{level 2} described in Section 4.2 of \cite{kusner2018counterfactual} where the latent variable is assumed to ``act as non-deterministic causes of observable variables''.
By abuse of notation we will use the usual unit-level counterfactual notations $X_{S \leftarrow s'}$ and $\hat Y_{S \leftarrow s'}$ to denote the counterfactuals under the re-sampled noise approach computed as follows.

The counterfactual predictive distribution is the pushforward measure
\begin{equation*}
  \mathrm{Law}\!\left(Y_{S\leftarrow s'}\right)
\;=\;
f(\cdot,s')_{\#}\,\mathrm{Law}\!\left(X_{S\leftarrow s'}\right).
\end{equation*}
    
\paragraph{Connection to the causal fairness framework of \citet{plecko2022causalfairnessanalysis}.}
Our structural causal model fits naturally into the general framework of \emph{causal fairness analysis} of \citet{plecko2022causalfairnessanalysis}. 
We work with a stochastic SCM in which the sensitive attribute $S$, an exogenous latent variable $\U$ and downstream noises $(\varepsilon_{\mathbf{X}},\dots)$ generate the observable features and outcomes via
\begin{equation*}
  \mathbf{X} \;=\; \phi(\U,S,\varepsilon_{\mathbf{X}}), 
\qquad 
Y \;=\; g(\mathbf{X},S,\varepsilon_Y),
\qquad
\widehat Y \;=\; f(\mathbf{X},S),
\end{equation*}

and we assume (Assumption~\ref{ass:U_measurable}) that $\U$ is a measurable function of $(\mathbf{X},S)$, i.e.\ $\U=\psi(\mathbf{X},S)$ almost surely.  
This corresponds to the ``more realistic level~2'' in \citet[Section~4.2]{kusner2018counterfactual}, where latent variables act as non-deterministic causes of the observables: the analyst models the structural equations and the distribution of the exogenous variables, but does \emph{not} condition on the realized downstream noise in counterfactual worlds.  
Accordingly, when we intervene on $S$ we re-sample the exogenous noises downstream of $S$ (e.g.\ $\varepsilon_{\mathbf{X}}$) instead of replaying their realized values; this leads to the resampled-noise counterfactuals $\widehat Y_{S\leftarrow s'}$ used in Proposition~\ref{prop:CF_DP}.  

Within the terminology of \citet{plecko2022causalfairnessanalysis}, one can view our model as a particular instance of their \emph{standard fairness model} (SFM) after grouping variables into blocks: $S$ plays the role of the protected attribute, $\U$ encodes latent background factors that are treated as business-necessary, and $(\mathbf{X},Y,\widehat Y)$ collect observed covariates, outcomes and decisions.  
Our fairness \emph{target} is then the requirement that, once we condition on the latent background $\U$, changing $S$ has no remaining effect on the predictor: for $P_{\U}$-almost every $v$ and all $s,s'$ in the support of $P(S\mid \U{=}v)$,
\begin{equation*}
  \mathrm{Law}\!\big(\widehat Y_{S\leftarrow s} \mid \U{=}v\big)
\;=\;
\mathrm{Law}\!\big(\widehat Y_{S\leftarrow s'} \mid \U{=}v\big),
\end{equation*}

which, by Proposition~\ref{prop:CF_DP}, is equivalent to the demographic-parity-type constraint
\begin{equation*}
  \mathrm{Law}\!\big(f(\mathbf{X},S)\mid \U{=}v,S{=}s\big)
\;=\;
\mathrm{Law}\!\big(f(\mathbf{X},S)\mid \U{=}v,S{=}s'\big).
\end{equation*}

In the language of causal fairness analysis, this can be read as a specific fairness \emph{signature}: conditional on the admissible background $\U$, we rule out any unfair direct or indirect influence of $S$ on the decision $\widehat Y$.  
Thus, our counterfactual constraint is one particular point in the broader causal fairness map of \citet{plecko2022causalfairnessanalysis}, specialised to (i) a level~2 stochastic SCM with resampled noise and (ii) a fairness requirement formulated at the level of the predictor $\widehat Y$ rather than the outcome $Y$.

\subsection{Proof of Proposition~\ref{prop:CF_DP}}

With the definition of counterfactual given above, we are now ready to prove Proposition~\ref{prop:CF_DP}.

\propCFDP*
For clarity, we restate the proposition below before proving it. Under Assumption~\ref{ass:U_measurable}, the SCM is $\mathbf{X}=\phi(\U, S, \varepsilon_\mathbf{X})$ with $\U = \psi(\mathbf{X}, S)$ a.s..
The noise $\varepsilon_\mathbf{X}$ is independent of $S$ conditional on $\U$.
\begin{proposition}\label{prop:CF_DP_resampled}
  Assume Assumption~\ref{ass:U_measurable} is satisfied.
  Let $\widehat Y=f(\mathbf{X},S)$ for some regression function $f\in \mathcal{F}$. 
  In the non-deterministic SCM where $\mathbf{X}=\phi(\U,S,\varepsilon_\mathbf{X})$, under Assumption~\ref{ass:U_measurable} the following two constraints are equivalent: 
  \begin{enumerate}[label=(\roman*)]
  \item \label{it:CFX_resampled}
  For all $s \in [K]$, for $P_{\mathbf{X}\mid S=s}$-almost every $\mathbf{x}$ and all $s'\in[K]$,
  \begin{equation*}
  Law\!\big(\hat Y_{S\leftarrow s}\mid \mathbf{X}{=}\mathbf{x},S{=}s\big)
  =Law\!\big(\hat Y_{S\leftarrow s'}\mid \mathbf{X}{=}\mathbf{x},S{=}s\big).
\end{equation*}
  \item \label{it:DPU_resampled}
  For $P_\U$-almost every $v$ and all $s,s'\in Supp(P(S= \cdot \vert \U = v))$,
  \begin{equation*}
  Law\!\big(f(\mathbf{X},S)\mid \U{=}v,S{=}s\big)
  =Law\!\big(f(\mathbf{X},S)\mid \U{=}v,S{=}s'\big).
\end{equation*}

  \end{enumerate}
  \end{proposition}
  
  \begin{proof}
  For $(\mathbf{x},s)\in Supp(P_{\mathbf{X},S})$ and $v:=\psi(\mathbf{x},s)$, let $s' \in [K]$, we have 
  \begin{equation}\label{eq:key}
  Law\!\left(\hat Y_{S\leftarrow s'}\mid \mathbf{X}{=}\mathbf{x},S{=}s\right)
  =Law\!\left(f(\mathbf{X},S)\mid \U{=}v,S{=}s'\right)= \mu_{f \vert v, s'}.
  \end{equation}
  We condition on $(\mathbf{X}{=}\mathbf{x},S{=}s)$ and set $v:=\psi(\mathbf{x},s)$ by Assumption~\ref{ass:U_measurable} (abduction).
  Since $\phi, f$ and $Law(\varepsilon_\mathbf{X}\vert \U=v)$ are unmodified by the intervention on $S$ we have 
\begin{equation*}
Law\!\left(\hat Y_{S\leftarrow s'}\mid \mathbf{X}{=}\mathbf{x},S{=}s\right)
=Law\!\left(f(\phi(v,s',\varepsilon'_\mathbf{X}),s')\right),
\end{equation*}
where $\varepsilon'_\mathbf{X}$ is drawn from the same noise law as in the factual model (i.e. the initial model with unmodified equations) and is independent of $S$.
The exogenous noise is independent of $S$, $Law(\varepsilon'_\mathbf{X} \mid \U{=}v)\stackrel d=Law(\varepsilon_\mathbf{X} \mid \U{=}v)$, hence
\begin{align*}
Law\left(f(\phi(v,s',\varepsilon'_\mathbf{X}),s')\right)
&=Law\left(f(\phi(v,s',\varepsilon_\mathbf{X}),s')\right)\\
&=Law\left(f(\mathbf{X},S)\mid \U{=}v,S{=}s'\right),
\end{align*}
since, given $(\U{=}v,S{=}s')$, the factual model generates $\mathbf{X}=\phi(v,s',\varepsilon_\mathbf{X})$.
This yields \eqref{eq:key}.

  \smallskip\noindent
  (ii)$\Rightarrow$(i): Fix $\mathbf{x},s$ and put $v=\psi(\mathbf{x},s)$. Using \eqref{eq:key} twice,
  \begin{align*}
  Law(\hat Y_{S\leftarrow s}\mid \mathbf{X}{=}\mathbf{x},S{=}s)
  &=\mu_{f \vert v, s},\\
  &=\mu_{f \vert v, s'},\\
  &=Law(\hat Y_{S\leftarrow s'}\mid \mathbf{X}{=}\mathbf{x},S{=}s).
\end{align*}
  
  \smallskip\noindent
  (i)$\Rightarrow$(ii): By (i) and \eqref{eq:key},
  \begin{equation*}
  Law(f(\mathbf{X},S)\mid \U{=}\psi(\mathbf{x},s),S{=}s)
  =Law(f(\mathbf{X},S)\mid \U{=}\psi(\mathbf{x},s),S{=}s')
  \quad\text{for }P_{\mathbf{X}\mid S=s}\text{-a.e. } \mathbf{x}.
\end{equation*}

  Since $\U=\psi(\mathbf{X},S)$ a.s., the pushforward of $P_{\mathbf{X}\mid S=s}$ through $\mathbf{x}\mapsto \psi(\mathbf{x},s)$ equals $P_{\U\mid S=s}$. Hence
  \begin{equation*}
  Law(f(\mathbf{X},S)\mid \U{=}v,S{=}s)
  =Law(f(\mathbf{X},S)\mid \U{=}v,S{=}s')
  \quad\text{for }P_{\U\mid S=s}\text{-a.e. }v.
\end{equation*}
From this identity, push forward $P_{\mathbf{X} \vert S=s}$ by $\mathbf{x} \mapsto \psi(\mathbf{x}, s)$ to get $P_{\U \vert S=s}$.
Then there exists a $P_{\U \vert S=s}$ null set $D_s$ such that for all $v \notin D_s$, and all $s'$, $\mu_{f \vert v, s} = \mu_{f \vert v, s'}$. 
Let $D:=\bigcup_{s=1}^K D_s$, which is a $\mathbb{P}_\U$ null set since the exogenous $\U$ is independent of $S$.

Hence, for $P_\U$-a.e. $v$ and all $s, s'$, $\mu_{f \vert v, s}= \mu_{f \vert v, s'}$, i.e.
\begin{equation*}
  Law(f(\mathbf{X},S)\mid \U{=}v,S{=}s)
  =Law(f(\mathbf{X},S)\mid \U{=}v,S{=}s').
\end{equation*}
which yields (ii).
  \end{proof}
  
  \begin{remark}
  If the predictor is randomized, $\hat Y=\xi(\mathbf{X},S,\varepsilon_{\hat Y})$, with $\xi$ a measurable function and $\varepsilon_{\hat Y}$ an exogenous noise independent of $(\U,S,\varepsilon_\mathbf{X})$.
  Then simply resample $\varepsilon_{\hat Y}$ in the prediction step; the proof goes through with $f$ replaced by the conditional law of $\xi$. 
  Under mild regularity (e.g., continuity of $v\mapsto Law(f(\mathbf{X},S)\mid \U=v,S=s)$), the “$P_\U$-a.e.” qualifier can be strengthened to “for all $v$ in the support of $\U$”.
  \end{remark}

\section{Proof of Propositions~\ref{prop:fair_optimal_predictor} and~\ref{prop:closed_form_relaxed}} \label{proof:reduction_OT}

In this section, we prove the expression of the optimal predictor under the relaxed fairness constraint stated in Proposition \ref{prop:closed_form_relaxed}. 
The proof for the exact fairness case described in Proposition~\ref{prop:fair_optimal_predictor} is a particular case of this one when $\alpha=0$.

We recall that the optimal fair predictor $f^*_0$ is the solution of the optimization problem
\begin{equation*}
  f_0^* \in \argmin_{f \in \mathcal{F}} \left\{\mathcal{R}(f) : \mathcal{U}(f) =0\right\}. 
\end{equation*}

\ClosedFormRelaxed*

First, we define a metric space on the set of family of measures on $\mathbb{R}$ with finite second moment, and show that this metric space satisfies the barycenter property. 
With this property we use an \textit{abstract geometric lemma} from \cite{chzhen2022minimaxframeworkquantifyingriskfairness} to prove the closed form of the optimal predictor under the relaxed fairness constraint and the optimal risk-unfairness pair of this predictor.

The proof of Proposition~\ref{prop:closed_form_relaxed} leverages an abstract geometric lemma (Lemma 4.3 from \cite{chzhen2022minimaxframeworkquantifyingriskfairness}) recalled in Lemma~\ref{lem:geom_lemma}.
Let $(\mathcal{V},\mathcal{A},\nu)$ be a probability space. Define
\begin{equation*}
  \mathcal{M}
  :=\Big\{ a=(a_v)_{v\in\mathcal{V}}:\ \forall v \in \ExoSet, \ a_v\in\mathcal{P}_2(\mathbb{R}), \quad v \mapsto a_v \ \text{is measurable} \Big\},
\end{equation*}
the set of families of measures on $\mathbb{R}$ with finite second moment.
\begin{definition}
For $a,b\in\mathcal{M}$ we define the averaged Wasserstein metric as
  \begin{equation*}
  \bar d(a,b)\ :=\ \Big( \int_{\mathcal{V}} \mathcal{W}_2^2(a_v,b_v)\, d\nu(v) \Big)^{1/2}.
\end{equation*}
  \end{definition}
  The folowing proposition ensures that $(\mathcal{M},\bar d)$ is a metric space satisfying the barycenter property.
  The proof can be found in Section~\ref{sec:proof_auxiliary_lemmas}.
\begin{proposition}\label{prop:barycenter_property_M}
  The metric space $( \mathcal{M}, \bar{d})$ satisfies the barycenter property i.e.,
for all weights $w\in\Delta^{K-1}$ and any $K$-tuple $a=(a^{(1)},\dots,a^{(K)})\in\mathcal{M}^K$
there exists $C_{\mathbf{a}_{\mathbf{w}}}\in\mathcal{M}$ such that
\begin{equation*}
  C_{\mathbf{a}_{\mathbf{w}}}\ \in\ \arg\min_{C\in\mathcal{X}}\ \sum_{s=1}^K w_s\,\bar d^{\,2}\!\big(a^{(s)},C\big).
\end{equation*}
\end{proposition}

We established that $(\mathcal{M}, \bar{d})$ satisfies the barycenter property.
Next, we determine the form of the optimal predictor for the relaxed problem stated in Proposition~\ref{prop:closed_form_relaxed}. 
We propose a specific form for the predictor and demonstrate its optimality by verifying that it satisfies properties \eqref{P1}-\eqref{P2} from Lemma \ref{lem:geom_lemma}, then applying the lemma to conclude the result.
We use the notation introduced in Section~\ref{an:notations}.

  We can now proceed to the proof of Proposition~\ref{prop:closed_form_relaxed}.

  \paragraph{Risk/unfairness identities.}
    Before proving the form of the solution of the relaxed problem, we start by showing that if it holds then the expressions of the risk follow directly.
To derive the expression for the risk, we leverage Theorem \ref{thm:risk_opt_unfairness_bayes} from \cite{chzhen2020fairregressionwassersteinbarycenters}. For a given $v \in \ExoSet$, we have 
\begin{equation*}
  \mathcal{R}_v(f^*_0) = \mathcal{U}_v(f^*), 
\end{equation*}
with $ \mathcal{R}_v(f):= \sum_{s=1}^K w_s \mathbb{E}[(f(\mathbf{X}, S)-f^*(\mathbf{X}, S))^2 \mid \U = v, S = s]$ the local risk and $\mathcal{U}_v(f)$ the local unfairness as defined in Eq.\eqref{def:unfairness}. By integrating on the values of the exogenous variable we obtain
\begin{align*}
  \mathcal{R}(f^*_0) &= \mathcal{U}(f^*).
\end{align*}
The expression of the risk follows from plugging-in the expression of the $\alpha$-relaxed fair optimal predictor into the risk. 
We obtain
\begin{align*}
  \mathcal{R}(f_\alpha^*) &= \sum_{s=1}^K w_s \mathbb{E}[(f^*(\textbf{X}, S) - f_\alpha^*(\textbf{X}, S))^2\vert S=s] , \\
  &= (1-\sqrt{\alpha})^2 \sum_{s=1}^K w_s \mathbb{E}[(f^*(\textbf{X}, S) - f_0^*(\textbf{X}, S))^2 \vert S=s] , \\
  &= (1- \sqrt{\alpha})^2 \mathcal{R}(f^*_0).
\end{align*}
It remains to prove the form of the solution.
\paragraph{Closed form of the optimal predictor.}
  Considering the metric space $(\mathcal{M}, \bar{d})$ defined above, for $\alpha \in [0, 1]$,  $v \in \ExoSet$ , $s\in [K]$ and $\mathbf{x} \in \mathcal{X}$, the solution of the relaxed problem \eqref{eq:relaxed_OT_pb} is given by
  \begin{equation*}
  f^*_\alpha(\mathbf{x},s) \ :=\ \sqrt{\alpha}\,f^*(\mathbf{x},s) + (1-\sqrt{\alpha})\,T_{v,s}\big(f^*(\mathbf{x},s)\big), 
\end{equation*}
where $\mu_v^*$ is the Wasserstein-2 barycenter of the Bayes predictor group-conditional distributions $\{\mu_{f^* \vert v, s}\}_{s=1}^K$, and
$T_{v,s}:=\ F^{-1}_{\mu^*_v}\circ F_{\mu_{f^* \vert v, s}}$ is the optimal transport plan from $\mu_{f^* \vert v, s}$ to $\mu_{v}^*$.
\begin{remark}
  For all $v \in \ExoSet, s \in [K]$, $T_{v,s}$ is well-defined as a map because $\mu_{f^* \vert v, s}$ is non-atomic, satisfying Property~\ref{ass:distrib_densities}.
\end{remark}
Then proof is structured as follows. We define a candidate family $b$ by moving each $\mu_{f^*\mid v,s}$ along the $\mathcal{W}_2$-geodesic toward the barycenter $\mu_v^*$. 
We then show that $b$ satisfies \eqref{P1}--\eqref{P2} in $(\mathcal{M},\bar d)$, so Lemma~\ref{lem:geom_lemma} implies that $b$ solves the relaxed problem.

  \paragraph{Defining the geodesic candidate \textbf{b}.}

For each $(v,s)$, let $\alpha \in [0, 1]$ and set $b^{(s)}_v:=\gamma_{v,s}(1-\sqrt{\alpha})$, i.e.
  \begin{equation}
  b^{(s)}_v=\left((1-t)\,\mathrm{Id}+t\,T_{v,s}\right)_{\#} \mu_{f^* \vert v, s}.
  \label{eq:geodesic_1D}
\end{equation}
With $t= 1 - \sqrt{\alpha}$.
Our aim is to prove that the candidate family $b^{(s)}=(b^{(s)}_v)_{v\in\ExoSet}$ with a fixed parameter $\alpha\in[0,1]$—\emph{constant across $v$}—solves the relaxed problem~\eqref{eq:relaxed_OT_pb}. 
  To do so, we define the families $\mu_{s}:=\{\mu_{f^* \vert v, s}\}_{v \in \ExoSet} \in \mathcal{M}$, $\mu^*:=\{\mu_v^*\}_{v \in \ExoSet} $ and $b^{(s)}:= \{b_v^{(s)} \}_{v \in \ExoSet} \in \mathcal{M}$. 
  We then show by pointwize minimization that $\mu^*$ is the barycenter of $\{\mu_s\}_{s \in [K]}$ with weights $\textbf{w}$ in the metric space $(\mathcal{M}, \bar d)$. 

  \paragraph{Pointwise minimization: the barycenter.}
  We use the metric space $(\mathcal{M}, \bar{d})$ defined above.
Given the group-conditional families $\mu_s:=\{\mu_{f^*\mid v,s}\}_{v\in\ExoSet}\in  \mathcal{M}$ and weights $\mathbf w=(w_s)_{s=1}^K$,
consider the global \emph{objective}
\begin{align*}
\Objective(\mu)\;:=\;\sum_{s=1}^K w_s\,\bar d^2(\mu,\mu_s)
&\;=\;\int_{\ExoSet}\Bigg[\sum_{s=1}^K w_s\,\mathcal{W}_2^2\!\big(\mu_v,\mu_{f^*\mid v,s}\big)\Bigg]\,d\nu(v), \\
&\;=\;\int_{\ExoSet} \Objectivev(\mu_v)\,d\nu(v),
\end{align*}

Since $\mu \mapsto \Objective(\mu)$ is an integral of functions of $\mu_v$ only, we can minimize independently for each $v$.

For each $v\in\ExoSet$, the \emph{pointwise objective} is
\begin{equation*}
\Objectivev(\nu)\;:=\;\sum_{s=1}^K w_s\,\mathcal{W}_2^2\!\big(\nu,\mu_{f^*\mid v,s}\big).
\end{equation*}

We now show that if, 
\begin{equation*}
\mu_v^* \in \arg\min_{\eta\in\mathcal P_2}\ \Objectivev(\eta),
\end{equation*}
then $\mu^*:=\{\mu_v^*\}_{v\in\ExoSet}$ minimizes the global objective $\Objective$.

Fix any $\mu=(\mu_v)_v$. For each $v$, since $\mu_v^*\in\arg\min_\eta \Objectivev(\eta)$, we have
\begin{equation*}
  \Objectivev(\mu_v) \ge \Objectivev(\mu_v^*).
\end{equation*}
Integrating over $v$ yields
\begin{equation*}
\Objective(\mu)
=\int_{\ExoSet}\Objectivev(\mu_v)\,\nu(v) \ge \int_{\ExoSet}\Objectivev(\mu_v^*)\,\nu(v) = \Objective(\mu^*).
\end{equation*}
Hence, $\mu^*$ is a global minimizer of $\Objective$.
By the definition of $\Objectivev$, each $\mu_v^*$ is the weighted $\mathcal{W}_2$ barycenter at location $v$.
Thus, we have shown that the family $\mu^*=\{\mu_v^*\}_v$ is the barycenter of $\{\mu_s\}_{s=1}^K$ in the metric space $(\mathcal{M}, \bar{d})$. 
Next, we will demonstrate that the family $b^{(s)}$, defined in \eqref{eq:geodesic_1D},  satisfies the conditions of Lemma \ref{lem:geom_lemma}, allowing us to apply the lemma.

\paragraph{Verification of \eqref{P1}--\eqref{P2} along the geodesic.}
  Let us recall Properties~\eqref{P1}-\eqref{P2} from Lemma \ref{lem:geom_lemma} that we need to verify for $b^{(s)}$.
  \begin{align}
  \bar d^2(b^{(s)}, \mu_s) + \bar d^2(b^{(s)}, \mu^*) &= \bar d^2(\mu_s, \mu^*), \tag{P1}\label{P1}\\
  \bar d^2(b^{(s)}, \mu_s) &= (1-\sqrt{\alpha})^2 \bar d^2(\mu_s, \mu^*). \tag{P2}\label{P2}
\end{align}
We can now proceed to the verification. 
  To show $\eqref{P2}$ we use the expression of the distance $\bar{d}$ 

\begin{align*}
  \bar d^2(b^{(s)}, \mu_s) &= \int_{\ExoSet} \mathcal{W}_2^2(b_v^{(s)}, \mu_{f^* \vert v, s}) d\nu(v), \\
  &= (1-\sqrt{\alpha})^2 \int_{\ExoSet} \mathcal{W}_2^2(\mu_{f^* \vert v, s}, \mu_v^*) d\nu(v), \\
  &= (1-\sqrt{\alpha})^2 \bar d^2(\mu_s, \mu^*),
\end{align*}
where the second line follows from the application of the property \eqref{P2} from \cite{chzhen2022minimaxframeworkquantifyingriskfairness} in the metric space $(\mathcal{P}_2(\mathbb{R}), \mathcal{W}_2)$ at fixed $v$. 
Thus, \eqref{P2} is satisfied. 
To prove \eqref{P1}, we combine \eqref{P2} with the following equality
\begin{align}
  \bar d^2(b^{(s)}, \mu^*) &= \int_{\ExoSet} \mathcal{W}_2^2(b_v^{(s)}, \mu_v^*)\, d\nu(v) \nonumber \\
  &= \alpha \int_{\ExoSet} \mathcal{W}_2^2(\mu_{f^* \vert v, s}, \mu_v^*)\, d\nu(v) \nonumber \\
  &= \alpha \bar d^2(\mu_s, \mu^*) \label{eq:last_alpha_bar_d}
\end{align}
where the second line follows from the definition of $b_v^{(s)}$ in \eqref{eq:geodesic_1D} and the property \eqref{P2} in the metric space $(\mathcal{P}_2(\mathbb{R}), \mathcal{W}_2)$.
Combining \ref{eq:last_alpha_bar_d} with property \eqref{P2} and we finally get 
\begin{equation*}
  \bar d^2(b^{(s)}, \mu_s) + \bar d^2(b^{(s)}, \mu^*) = \bar d^2(\mu_s, \mu^*),
\end{equation*}
so \eqref{P1} is satisfied. 
  We have shown that $b^{(s)}$ satisfies \eqref{P1}–\eqref{P2} from Lemma \ref{lem:geom_lemma}. We can thus apply the lemma and deduce
  \begin{equation}
    \sum_{s=1}^K w_s\,\bar d^2 (b^{(s)},\mu_s)= \inf_{\mathbf{b}\in \mathcal{M}^K(\mathbb{R})} \left\{\sum_{s=1}^K w_s\,\bar d^2 (b^{(s)},\mu_s) \quad \text{s.t.} \quad \sum_{s=1}^K w_s\,\bar d^2 (b^{(s)},C_\mathbf{b}) \leq \alpha \sum_{s=1}^K w_s\,\bar d^2 (\mu_s,\mu^*) \right\}.
  \end{equation}
  \textbf{Conclusion.}
  From the expression of $b^{(s)}_v$ in Eq.\eqref{eq:geodesic_1D}, we have $b^{(s)}_v= {f^*_\alpha(\cdot,s)_{\#}\mu_{X\vert v, s}}$.
  We have
  \begin{align*}
    \mathcal{W}_2^2(b_v^{(s)}, \mu_{f^* \vert v, s}) = \mathbb{E}[(f_\alpha(X,S) - f^*(X,S))^2 \vert \U=v, S=s],
  \end{align*}
  so $\mathcal{U}(f_\alpha^*) = \alpha \mathcal{U}(f^*)$ which proves the unfairness espression.
  By definition of the risk
  \begin{align}\label{an:risk_decomposition}
  \mathcal{R}(f) = \int_v \sum_{s=1}^K w_s \mathbb{E}[(f^*(\mathbf{X}, S) - f(\mathbf{X},S))^2 \vert \U=v, S=s] d\nu(v),
  \end{align}
  and by Theorem 3 from \cite{gouic2020projectionfairnessstatisticallearning}, for $f \in \mathcal{F}$  such that $\mathcal{U}(f) \leq \alpha \mathcal{U}(f^*)$ we have 

  \begin{align*}
 \mathcal{R}(f) &\ge \int_v \sum_{s=1}^K w_s \mathcal{W}_2^2(b_v^{(s)}, \mu_{f^* \vert v, s}) d\nu(v). 
 \end{align*}
 Thus, 
 \begin{align*}
  \mathcal{R}(f) &\ge \int_v\sum_{s=1}^K w_s \mathbb{E}[(f_\alpha^*(\mathbf{X},S) - f^*(\mathbf{X},S))^2 \vert \U=v, S=s]d\nu(v), \\
  &= \mathcal{R}(f_\alpha^*),
  \end{align*}
where the second line comes from the optimality of $b^{(s)}$ for the relaxed problem.

Therefore, $f_\alpha^*$ is optimal for the relaxed problem.

\begin{remark}
$b_v^{(s)}$ is the $\mathcal{W}_2$-geodesic from $\mu_{f^* \vert v, s}$ to $\mu_v^*$ (we refer the interested reader to Section 6.1 from \cite{agueh_carlier} for more details about the constant speed geodesic). 

\end{remark}

\section{Proof of theorems and corollaries}  \label{proof:l_star}

Recall that we assume that the unconstrained optimum $f^*$ and the black-box predictor $f^{bb}$ satisfy Properties~\ref{ass:distrib_densities} and \ref{ass:lipschitz_cdf}. 

\paragraph{Structure of the proof of Theorem~\ref{thm:l_star}.} 
First, we establish that post-processed predictors inherit the regularity properties of their inputs (§\ref{sec:stability_fbb}). 
Next, we decompose the unfairness into discretization bias and estimation error terms, then bound each component separately (§\ref{proof:thm_l_star}). 
For the bias term, we leverage the $\frac{1}{2}$-Hölder continuity of the unfairness function with respect to the exogenous variable $\U$ (Lemma~\ref{proof:lemma_bias_unfairness}). 
For the error term, we combine Wasserstein-to-Kolmogorov-Smirnov bounds (Lemma~\ref{lemma:KS_W2}) with the Dvoretzky-Kiefer-Wolfowitz inequality to control empirical deviations (Lemma~\ref{lem:variance_unfairness}).
Finally, we optimize the resulting bias-error tradeoff to determine the optimal number of intervals $L^*$ (§\ref{proof:thm_l_star}).

\paragraph{Extension to relaxed fairness.} 
For Theorem~\ref{thm:relaxed_l_star}, we additionally show that interval unfairness is convex under predictor mixing (Lemma~\ref{lem:interval-convexity}), which allows us to decompose the relaxed problem into independent discretization and relaxation steps (§\ref{sec:proof_relaxed_l_star}).

\subsection{Regularity properties of the post-processed predictors}
\label{sec:stability_fbb}

Before proving the results we recall the assumptions under which we work. 
We assume that the distributions of the predictions of the unconstrained optimal predictor $f^*$ and black-box $f^{bb}$'s distributions $\left\{\mu_{f^* \vert v, s}\right\}_{v \in \ExoSet, s\in [K]}$ and $\left\{\mu_{f^{bb} \vert v, s}\right\}_{v \in \ExoSet, s\in [K]}$ respectively are non-atomic and are supported on an interval $[-M, M]$ (Property~\ref{ass:distrib_densities}). 
This implies in particular that the measures admit finite second moment bounded by $M^2$.
We assume that their c.d.f. verify Property~\ref{ass:lipschitz_cdf}.

We define a property for the post-processed predictor.
\begin{property}[Piecewise in $v$ Lipschitzness]
  \label{ass:piecewise_v_lipschitz}
  A predictor $f$ is piecewise in $v$ Lipschitz for a partition $\{\mathcal I_\ell\}_{\ell=1}^L$ of $\ExoSet$ if there exists a constant $\Lcdf>0$ such that, for every $\ell\in[L]$, $s\in[K]$, and $t\in\mathbb R$, the map $v\mapsto F_{f\mid v,s}(t)$ is $\Lcdf$-Lipschitz on $\mathcal I_\ell$, i.e.,
  \begin{equation*}
      \left\vert F_{f\mid v,s}(t)-F_{f\mid v',s}(t)\right\vert\le \Lcdf\ \vert v-v'\vert 
  \quad\text{for all } v,v'\in\mathcal I_\ell.
\end{equation*}

  \end{property}

In this section, we derive the regularity assumption that the post-processed predictor inherits from the black-box predictor. 

We show that the bounded support of the distributions and the Lipschitz property of the c.d.f. are preserved by the post-processed predictor $\widehat{f}_\ell$. 
The proof of the following lemma is provided in Section~\ref{proof:lem_stability_monotone}.

\begin{lemma}[Stability under plug-in post-processing]\label{lem:stability_monotone}
  Assume Properties~\ref{ass:distrib_densities}--\ref{ass:lipschitz_cdf} hold for the black-box predictor $f^{bb}$.
  For a fixed interval $\mathcal I_\ell$,
  Almost surely with respect to the randomness of training sample, the post-processed predictor $\widehat f_\ell$ satisfies:
  
  \begin{enumerate}[label=(\roman*)]
  \item \textbf{Bounded support:} For all $(v,s)$, $Supp(\mu_{\widehat f_\ell\mid v,s}) \subseteq [-M,M]$.
  \item \textbf{Piecewise $\Lcdf$-Lipschitz c.d.f. in $v$:} For all $v,v'\in \mathcal I_\ell$,
  \begin{equation*}
  \sup_{t\in\mathbb R}\left|\widehat G_{v,s}(t)-\widehat G_{v',s}(t)\right| \;\le\; \Lcdf\,|v-v'|,
  \end{equation*}
  where $\widehat G_{v,s}$ denotes the conditional c.d.f.\ of $\widehat f_\ell(X,s)$ given $(V=v,S=s)$.
  \end{enumerate}
  In particular, the global post-processed predictor $\widehat f_L(\mathbf x,s):=\sum_{\ell} \widehat f_\ell(\mathbf x,s)\,\mathbbm 1_{\{v\in\mathcal I_\ell\}}$ satisfies Property~\ref{ass:piecewise_v_lipschitz}.
  \end{lemma}

  \subsection{Proof of Theorem~\ref{thm:l_star}.} \label{proof:thm_l_star}

We start by decomposing the unfairness of the post-processed predictor into a bias and an estimation error term. 
\paragraph{Unfairness decomposition.}
Let $f$ be a regressor. The discretization of the exogenous variable $\U$ into $L$ intervals $\StratExoSet= \{\mathcal{I}_\ell\}_{\ell=1 }^L$ induces the following \textit{bias-error} tradeoff decomposition of the unfairness
\begin{align}\label{an:decomp_unfairness}
  \mathcal{U}(f) = \sum_{\ell=0}^{L-1}  p_\ell \left[\underbrace{\mathcal{U}_{\ell}(f)}_{\text{estimation error}} + \frac{1}{p_\ell}\underbrace{\int_{\mathcal{I}_\ell} \left( \mathcal{U}_v(f) - \mathcal{U}_{\ell}(f)\right)d\nu(v)}_{\text{discretization bias}} \right],
  \end{align}
where for each interval $\ell$, $p_\ell = \mathbb{P}(\U \in \mathcal{I}_\ell)$ is the interval weight, in the case of a uniform discretization, $p_\ell= \frac{1}{L}$ and $\mathcal{U}_\ell$ is the local unfairness defined as $\mathcal{U}_{\ell}(f) := \min_{\mu \in \mathcal{P}_2(\mathbb{R})}\sum_{s=1}^K w_s \mathcal{W}_2^2(\mu_{f \vert \ell, s}, \mu).$.

When considering the post-processed predictor $\widehat f_L$, we have the following decomposition; 
The first term $\mathcal{U}_\ell(\widehat{f}_L)$ represents the \textit{estimation error} of the unfairness within each interval due to using empirical estimates of the conditional distributions. 
The second term captures the \textit{discretization bias} from approximating the pointwise unfairness $\mathcal{U}_v(\widehat{f}_L)$ by the interval-wise unfairness $\mathcal{U}_{\ell}(\widehat f_L)$.
In the following sections, we bound these two terms separately 
using the regularity conditions satisfied by $\widehat f_L$ established in Section~\ref{sec:stability_fbb}.

\paragraph{Bound on the bias term.}
Fix $\ell \in [L]$.
Our objective is to establish a bound for the bias component in Eq.~\eqref{an:decomp_unfairness}. Specifically, we seek to find $\delta_{\text{bias}}(L)$ such that
\begin{equation*}
  \int_{\mathcal{I}_\ell} \left(\mathcal{U}_v(\widehat f_L)-\mathcal{U}_\ell(\widehat f_L)\right)d\nu(v) \leq \delta_{\text{bias}}(L).
\end{equation*}

The key insight is that the unfairness function is Lipschitz continuous with respect to the exogenous variable $v$, and this property is preserved when transitioning to the Wasserstein barycenter. This regularity allows us to control how much the pointwise unfairness $\mathcal{U}_v(\widehat f_L)$ can deviate from the interval-wise unfairness $\mathcal{U}_\ell(\widehat f_L)$ within each discretization interval.

We establish this through the following sequence of results that are proved in Section~\ref{sec:proof_auxiliary_lemmas}. 
First, we show that the unfairness inherits the Lipschitz property from our regularity assumptions (the proof can be found in \ref{proof2:lemma_bias_unfairness}). 

  \begin{lemma}[Lipschitz continuity of Unfairness]
    \label{proof:lemma_bias_unfairness}
    Let $\mathcal{I}\subset [0, 1]$ and $f$ be a predictor with bounded distribution support satisfying Property~\ref{ass:lipschitz_cdf} on $\mathcal{I}$. 
    The unfairness of $f$ is Lipschitz continuous on $\mathcal{I}$. 
    Specifically, for all $v, v' \in \mathcal{I}$
    \begin{equation*}
      \vert \mathcal{U}_v(f) - \mathcal{U}_{v'}(f) \vert \leq 16 M^2 \Lcdf \vert v - v' \vert, \quad a.s.
  \end{equation*}
  where $M>0$ is the constant bounding the distribution's support.
  \end{lemma}
    
  From Lemma~\ref{lem:stability_monotone}, we know that $\widehat f_\ell$ has a bounded distribution support and is $\Lcdf$-Lipschitz in $v$ on interval $\mathcal{I}_\ell \in \StratExoSet$.
  Consequently, for a uniform discretization of the variable $\U$ into $L$ intervals, the bias term can be bounded as
\begin{align*}
  \int_{\mathcal{I}_\ell} \left(\mathcal{U}_v(\widehat f_L)-\mathcal{U}_\ell(\widehat f_L)\right)d\nu(v) 
  & =   \int_{\mathcal{I}_\ell} \left(\mathcal{U}_v(\widehat f_\ell)-\mathcal{U}_\ell(\widehat f_\ell)\right)d\nu(v), \\
  &\leq \frac{16M^2\Lcdf}{L},
  \end{align*}
  where the last line follows from Lemma~\ref{proof:lemma_bias_unfairness}.

\paragraph{Bound of the estimation error term.}

We now address the estimation error term in Eq.~\eqref{an:decomp_unfairness}.
The control proceeds in two steps: (i) relate the Wasserstein and Kolmogorov–Smirnov (KS) distances under the regularity of the post-processed predictor; (ii) apply the Dvoretzky–Kiefer–Wolfowitz inequality \citep{Massart1990} to bound the KS gaps between group-wise empirical laws.

The next result links the KS distance to the unfairness of the post-processed predictor.
The proof of all statements can be found in Appendix~\ref{proof:lem_KS_W2}.

\begin{lemma}\label{lemma:KS_W2}
  Let $\mu$ and $\nu$ be two univariate probability measures supported on a compact set of diameter $2M$. Then, the following inequality holds
  \begin{equation*}
    \mathcal{W}_2^2(\mu, \nu) \leq 4M^2 \sup_{t \in \mathbb{R}} \vert F_\mu(t) - F_\nu(t) \vert.
  \end{equation*}
\end{lemma}

We remind the Dvoretzky-Kiefer-Wolfowitz inequality (e.g., Corollary 1 of \citep{Massart1990}).
\begin{theorem}[Dvoretzky-Kiefer-Wolfowitz]  \label{lem:2sampleDKW}
  Let $F$ and $G$ be c.d.f.'s on $\mathbb{R}$ and let $F_m, G_n$ be the empirical c.d.f.'s built from $m$ and $n$ i.i.d. samples from $F$ and $G$, respectively. For all $\varepsilon>0$,
  \begin{equation}
  \mathbb{P}\!\left(\,\sup_{t\in\mathbb{R}}\bigl|(F_m-G_n)(t)-(F-G)(t)\bigr|>\varepsilon\,\right)
  \;\le\; 2\exp\!\Big(-\,\tfrac{2mn}{m+n}\,\varepsilon^2\Big).
\end{equation}

  Equivalently, with probability at least $1-\kappa$,
  \begin{equation*}
  \sup_{t\in\mathbb{R}}\bigl|(F_m-G_n)(t)-(F-G)(t)\bigr|
  \;\le\;\sqrt{\tfrac{m+n}{2mn}\,\log\tfrac{2}{\kappa}}.
\end{equation*}  
  \end{theorem}

  With this theorem, we have all the ingredients to derive the following bound on the estimation error term of the unfairness bias-error tradeoff in high-probability.

      \begin{lemma}\label{lem:variance_unfairness}
        Fix $\ell\in[L]$ and let $\widehat f_\ell$ be the local post-processed predictor and sample-split defined in Section.~\ref{sec:form_estimator}.
        Define 
        $\eta_L :=KL\exp(-n \frac{\min_s w_{s}}{8L}) + \frac{1}{n}$.
        With probability at least $1-\eta_L$ 
        \begin{equation}\label{eq:variance_term_bound_main}
        \mathcal{U}_{\ell}(\widehat{f}_\ell) \leq 8M^{2}\!
        \left[
        \sum_{s=1}^K   w_s \,\varepsilon_{\ell,s}
        \right],
        \end{equation}
        where
        \begin{equation*}
        \varepsilon_{\ell,s}\ :=\
        \sup_{t} \big| \widehat F_{{f^{bb}|\ell,s}}(t) - F_{{f^{bb}|\ell,s}}(t)\big|
        \ \le\
        \sqrt{\frac{2}{\,n\, p_{\ell,s}}\,
\log\!\Big(2KLn\Big)},
      \end{equation*}
      where $ p_{\ell, s} = \mathbb{P}(\U\in \mathcal{I}_\ell, S=s)$.
        \end{lemma}
We prove a more general result Lemma~\ref{lem:variance_unfairness_general} in \ref{proof:lem_variance_unfairness}.

We can now prove Thm.~\ref{thm:l_star}, which we retsate here for the reader's convenience.
\Lstar*

\paragraph{Trade-off between discretization bias and estimation error.} 
We combine the bounds derived above.
By Lemmas~\ref{proof:lemma_bias_unfairness} and~\ref{lem:variance_unfairness} and Eq.\eqref{an:decomp_unfairness}, for a uniform discretization $\mathcal{I}_\ell = [\frac{\ell-1}{L}, \frac{\ell}{L}] $ we have with probability $1-\eta_L$

\begin{equation}
    \label{eq:global_bound_optimization}
    \mathcal U(\widehat f_L)\ \le\
  \frac{16 M^2 \Lcdf}{L}
    \;+\;
  8M^2\,\sqrt{\frac{2L\log\!\left( 2KLn\right)}{n}}\left(\sum_s \sqrt{w_s} \right).
\end{equation}
where we used $p_{\ell}=1/L$ and $p_{\ell,s}= w_s/L$. Now, Jensen inequality implies that $\sum_{s=1}^K \sqrt{ w_s} \leq \sqrt{K}$, so 
\begin{equation*}
  \mathcal U(\widehat f_L)\ \le\
  \frac{16 M^2 \Lcdf}{L}
\;+\;
8M^2\,\sqrt{\frac{2LK\log\!\left( 2KLn\right)}{n}}.
\end{equation*}
Set
\begin{equation*}
    L^* = \left\lfloor \left( \frac{8\Lcdf^2 n}{K \log(2Kn)} \right)^{1/3} \right\rfloor.
\end{equation*}
On the one hand, using the inequality $1/\lfloor x \rfloor \leq 2/x$ for $x \geq 1$. Then, if $L^* \geq 1$ we have
\begin{equation*}
  \frac{16 M^2 \Lcdf}{L}\leq 16M^2\left(\frac{\Lcdf  K \log(2Kn)}{n}\right)^{1/3}.
  \end{equation*}

    On the other hand, for $n \geq 37$, we have $\log(2KL^*n) \leq 2\log(2Kn)$. Thus
\begin{align*}
    8M^2\,\sqrt{\frac{2L^*K\log\!\left( 2KL^*n\right)}{n}} 
    &= 16 \sqrt{2} M^2 \left(\frac{\Lcdf K \log(2Kn)}{n}\right)^{1/3}.
\end{align*}

    Now, $2K \geq 4$ so for $n\geq 37$, $\log(2Kn)\leq \log(2K)\log(n)$. Thus, we get that if $L^* \geq 4$ with probability $1 - \eta_{L^*}$ we have 
    \begin{equation*}
        \mathcal U(\widehat f_{L^*})\ \le\
    40M^2\left(\Lcdf  K \log(2K)\frac{\log(n)}{n}\right)^{1/3 }.
    \end{equation*}
    This holds in particular if $n/\log(n) \geq K/(8\Lcdf^2)$. 
    Choosing $c\geq 37$ such that $n \geq c \implies n/\log(n) \geq 40MK/\Lcdf$ and $C = 40M^2\left(\Lcdf K \log(2K)\right)^{1/3}$ concludes the proof of Theorem \ref{thm:l_star}.

    \subsection{Proof of Theorem~\ref{cor:risk-exact-fair}} \label{proof:thm_risk_exact}
    Let $f_L$  be the oracle post-processor computed with the true conditional distributions $\mu_{f^{bb}|\ell,s}$ instead of their empirical estimates, i.e.,
    \begin{equation}\label{eq:oracle_post_processed_discretized}
      f_{L}(\mathbf{x}, s) := \sum_\ell \sum_{s'=1}^K w_{s'} {F}_{{f^{bb} \vert {\ell,s'}}}^{-1} \circ {F}_{{f^{bb}\vert \ell,s}}(f^{bb}(\mathbf{x}, s))\mathbbm{1}_{v \in \mathcal{I}_\ell}.
    \end{equation}
    
    For an interval $\ell\in [L]$, $v\in \ExoSet$ and a group $s\in [K]$, we will denote by $F_{\ell, s}= F_{{f^{bb}|\ell,s}}$, $F_{v, s}= F_{{f^{bb}|v,s}}$ the cumulative distribution functions of the intervalwise and ideal CF maps, respectively.
    Similarly, $F^{-1}_{\ell, s}= F^{-1}_{\mu_{f^{bb}|\ell,s}}$ and $F^{-1}_{v, s}= F^{-1}_{\mu_{f^{bb}|v,s}}$ will denote the corresponding quantile functions.
    The c.d.f of the wasserstein barycenter will be noted $F_{v, \star}= \left( \sum_{s} w_s F^{-1}_{v, s} \right)^{-1}$ and $F_{\ell, \star}= \left( \sum_{s} w_s F^{-1}_{\ell, s} \right)^{-1}$ for the ideal and intervalwise barycenters, respectively.
    The quantile functions of the barycenters will be noted $ F^{-1}_{v, \star}$ and $F^{-1}_{\ell, \star}$. We define the ideal and intervalwise CF maps at level $v$ and $\ell$ as
    \begin{equation*}
      T_{v,s}(y)=F^{-1}_{v, \star}(F_{v,s}(y)),\qquad
      T_{\ell,s}(y)=F^{-1}_{\ell, \star}(F_{\ell,s}(y)),
      \end{equation*}
    $Y_{v,s}\sim \mu_{f^{bb}|v,s}$ and $Y_{\ell,s}\sim \mu_{f^{bb}|\ell,s}$.
    We recall the oracle post-processed predictor defined as $f^{bb}_0:= \left( \sum_{s'=1}^K w_{s'} F^{-1}_{{f^{bb} \vert v,s'}} \right) \circ F_{{f^{bb} \vert {v , s}}} (f^{bb}(\mathbf{x}, s)),$ at $\U= v$. 
    where $v=\psi(\mathbf x,s)$.
    We define
    \begin{equation*}
      \mathcal{R}^{est}_v= \sum_{s} w_s \mathbb{E}[(\widehat T_{\ell,s}(Y_{v,s})-T_{\ell,s}(Y_{v,s}))^2]
    \end{equation*}
    the bound on the estimation error due to the empirical estimation of the conditional distributions $\mu_{f^{bb}|\ell,s}$ and 
    \begin{equation*}
      \mathcal{R}^{disc}_v= \sum_{s} w_s \mathbb{E}[(T_{\ell,s}(Y_{v,s})-T_{v,s}(Y_{v,s}))^2]
    \end{equation*}
    the bound on the discretization error due to the discretization of the exogenous variable $\U$ into $L$ intervals.
   
    The risk of the estimated post-processed predictor $\widehat f_L$
    \begin{equation*}
        \mathcal{R}_v(\widehat f_L) = \sum_s w_s\mathbb{E}[(\widehat f_L(X, S) - f^*(X, S))^2\vert \U=v, S=s]
    \end{equation*}
     can be decomposed as
    
    \begin{align}
      \mathcal{R}_v(\widehat f_L)
      &\le (1+\lambda)\,\mathcal{R}_v(f^{bb}_0)\notag \\
         &+2\!\left(1+\frac{1}{\lambda}\right)
           \sum_{s} w_s\Big(
             \mathbb{E}\!\big[(\widehat f_L(\mathbf X,S)-f_L(\mathbf X,S))^2\mid \U=v,S=s\big]
             +\mathbb{E}\!\big[(f_L(\mathbf X,S)-f^{bb}_0(\mathbf X,S))^2\mid \U=v,S=s\big]\Big)
      \notag\\[2pt]
      &= (1+\lambda)\,\mathcal{R}_v(f^{bb}_0)
         +2\!\left(1+\frac{1}{\lambda}\right)
           \sum_{s} w_s\Big(
             \mathbb{E}\!\big[(\widehat T_{\ell,s}(Y_{v,s})-T_{\ell,s}(Y_{v,s}))^2\big]
             +\mathbb{E}\!\big[(T_{\ell,s}(Y_{v,s})-T_{v,s}(Y_{v,s}))^2\big]\Big)
      \notag\\[2pt]
      &=: (1+\lambda)\,\mathcal{R}_v(f^{bb}_0)
         +2\!\left(1+\frac{1}{\lambda}\right)\Big(\mathcal{R}^{est}_v+\mathcal{R}^{disc}_v\Big),
      \qquad \forall\,\lambda>0, \label{eq:risk_decomposition}
      \end{align}

      where in the first inequality we use the identity $(a+b)^2 \leq (1+\lambda)a^2 + (1+\frac{1}{\lambda})b^2$ for all $\lambda>0$.
      
    We start by showing a bound on the discretisation error that is proved in \ref{proof:lem_disc_error}.
    \begin{lemma}[Discretization error]\label{lem:disc_error}
      Assume $f^{bb}$ verifies Property~\ref{ass:distrib_densities}--\ref{ass:lipschitz_cdf}.
      Let $v \in\ExoSet, s \in [K]$.
      Then the discretization error is bounded as 
      \begin{equation*}
        \mathbb E\!\left[(T_{\ell,s}(Y_{v,s})-T_{v,s}(Y_{v,s}))^2\right]
        \ \le\ \frac{4M^2\,\Lcdf}{L},
    \end{equation*}
    
      \end{lemma}
    
    Then we show the following bound on the estimation error. 
      \begin{lemma}[Estimation error]\label{lem:estimation_risk}
        Under the assumptions in Lemma~\ref{lem:disc_error} and with $c$ as in Theorem.~\ref{thm:l_star}.
        If $n>c$,  with probability 1- $\eta_L$ where $\eta_L = KL^*\exp(-n \frac{\min_s w_{s}}{8L^*}) + \frac{1}{n}$ we have 
        \begin{equation*}
          \mathcal{R}^{est}_v \leq 16M^2\sqrt{\frac{2L\,\log\!\left( 2KLn\right)}{n}}\,
          \sum_{s=1}^K \sqrt{w_s}.
        \end{equation*}
      \end{lemma}
      \begin{remark}
        In order to bound the risk we need to control two terms. 
        First the estimation error $\mathcal{R}^{est}_v$ that occurs when taking the empirical post-processed predictor $\widehat f_{L}$ instead of the non estimated discretized $f_L$ defined in Eq.\eqref{eq:oracle_post_processed_discretized}. 
        Second the discretization error $\mathcal{R}^{disc}_v$ that occurs when taking $f_L$ instead of the population CF–optimal $f_0^{bb}$.
    
        The risk can be bounded as 
        \begin{equation*}
        \mathcal{R}(\widehat f_L)\leq 
        (1+\lambda)\,\mathcal{R}(f^{bb}_0)
        +2\!\left(1+\frac{1}{\lambda}\right) \int_\ExoSet \left(\mathcal{R}^{est}_v+\mathcal{R}^{disc}_v\right)d\nu(v).
        \end{equation*}

    \paragraph{Discretization error $\mathcal{R}^{disc}_v$.}
    Lemma~\ref{lem:disc_error} yields
    \begin{align*}
    \mathcal R^{disc}_v
    & \leq 4M^2\,\frac{\Lcdf}{L}.
    \end{align*}
    \paragraph{Estimation error $\mathcal{R}^{est}_v$.}
    Lemma~\ref{lem:variance_unfairness} and Lemma~\ref{lem:estimation_risk} yield, with probability at least $1-\eta_L$,
    \begin{equation*}
    \mathcal{R}^{est}_v \leq 16M^2 \sqrt{\frac{2L\,\log\!\left( 2KLn\right)}{n}}\,
      \sum_{s=1}^K \sqrt{w_s}.
    \end{equation*}
    
      \end{remark}
      
    We now proceed to the proof of Theorem~\ref{cor:risk-exact-fair}. 
    \RiskExactFair*
      
      \begin{proof}[Proof of Theorem~\ref{cor:risk-exact-fair}]
        Using Lemma~\ref{lem:disc_error} and Lemma~\ref{lem:estimation_risk}  and $\sum_s \sqrt{w_s} \leq \sqrt{K}$, integrating \eqref{eq:risk_decomposition} over $v$ gives

    \begin{align}\label{eq:risk_bound_for_any_L}
      \mathcal{R}(\widehat f_{L})
      \le (1 +\lambda)\mathcal R(f^{bb}_0)+ 2 \left(1+\frac{1}{\lambda}\right)\left(4M^2\,\frac{\Lcdf}{L} +16M^2  \sqrt{\frac{2LK\,\log\!\left( 2KLn\right)}{n}}\,
    \right).
    \end{align}
        
        Let \(L^\star\) be the discretization in Theorem~\ref{thm:l_star}:
    \begin{equation*}
      L^* = \left\lfloor \left( \frac{8\Lcdf^2 n}{K \log(2Kn)} \right)^{1/3} \right\rfloor.
  \end{equation*}
    We use the same arguments as in the proof of Theorem~\ref{thm:l_star}.
    On the one hand, for $x\geq 4$, $1/\lfloor x \rfloor\leq 2/x$. Then, if $L^*\geq 4$ we have

    \begin{equation*}
      4M^2\,\frac{\Lcdf}{L}\leq 4 M^2 \left(\frac{\Lcdf K \log(2Kn)}{n}\right)^{1/3}.
  \end{equation*}
    On the other hand, for $n\geq \sqrt{\frac{\Lcdf}{M}}$, $\log(2KnL^*) \leq 2\log(2Kn)$, thus

    \begin{equation*}
      16M^2\,\sqrt{\frac{2L^*K\log\!\left( 2KL^*n\right)}{n}} \leq 32\sqrt{2}M^2\left(\frac{\Lcdf K \log(2Kn)}{n}\right)^{1/3}.
      \end{equation*}

        \begin{equation*}
          \mathcal R(\widehat f_{L^\star})
          \;\le\;
          (1 +\lambda)\mathcal R(f^{bb}_0) 
          + 2\left(1 +\frac{1}{\lambda}\right)\cdot\left( 4+32\sqrt{2}\right) M^2 \left(\frac{\Lcdf K \log(2Kn)}{n}\right)^{1/3}.
          \end{equation*}
        Now, $2K \geq 4$ so for $n\geq 37$, $\log(2Kn)\leq \log(2K)\log(n)$. Thus, we get that if $L^* \geq 4$ with probability $1 - \eta_{L^*}$ we have 
        
        \begin{equation*}
          \mathcal R(\widehat f_{L^\star})
          \;\le\;
          (1 +\lambda)\mathcal R(f^{bb}_0) 
          + 2\left(1 +\frac{1}{\lambda}\right)\cdot \left( 4+32\sqrt{2}\right) M^2 \left(\frac{\Lcdf K \log(2K)\log(n)}{n}\right)^{1/3}
          \end{equation*}

        Setting $C_{risk} = 100M^2\left(\Lcdf K \log(2K)\right)^{1/3}$ we get 
        \begin{equation*}
          \mathcal R(\widehat f_{L^\star})
          \;\le\; 
          (1+\lambda)\mathcal R(f^{bb}_0) + C_{risk}\left(1 +\frac{1}{\lambda}\right)\cdot\left(\frac{\log(n)}{n}\right)^{1/3}.
        \end{equation*}
        \end{proof}
        \subsection{Proof of Theorem~\ref{thm:mse_lower_bound}} \label{proof:thm_estimation_only_minimax}
      
      The proof relies on Le Cam's method of "least favorable priors" applied to the estimation of the optimal fair transfer function. 
      We construct a family of distributions where determining the correct fair prediction requires estimating a parameter that is statistically hard to recover. The constraint on unfairness forces the predictor to attempt this estimation, thereby incurring an unavoidable cost in risk. The following Lemma is derived from the classical statement
        \citep[Lemma~2.12]{tsybakov2009introduction}; for the sake of completeness its proof is included in Appendix \ref{proof:lem_assouad}.
      \begin{lemma}[Assouad lemma](derived, e.g, from \citet{tsybakov2009introduction})
        \label{lem:assouad}
      Let $\{\mathbb{P}_\eta:\eta\in\{-1,+1\}^J\}$ be a hypercube of distributions and let $\phi(\eta)$ be the parameter of interest valued in a metric space $(\mathcal{T},d)$.
      Assume that whenever $\eta$ and $\eta'$ differ in exactly one coordinate,
      \begin{equation*}
          d\big(\phi(\eta),\phi(\eta')\big)\ \ge\ \Delta\quad\text{and}\quad
      \mathrm{KL}\!\left(\mathbb{P}_\eta\middle \vert \mathbb{P}_{\eta'}\right)\ \le\ \kappa.
      \end{equation*}
      
      Then for any estimator $\widehat\phi$,
      \begin{equation*}
      \sup_{\eta}\ \mathbb{E}_\eta\,d\big(\widehat\phi,\phi(\eta)\big)
      \ \ge\ \frac{J\,\Delta}{2}\,\Big(1-\sqrt{\tfrac{\kappa}{2}}\Big).
      \end{equation*}
      
      \end{lemma}
      
      

              \textbf{Construction of the least favorable case.}
              We restrict our attention to the covariate space $V \in [0,1]$. We partition this space into $J$ disjoint intervals of equal width
              \begin{equation*}
                  I_j = \left[\frac{j-1}{J}, \frac{j}{J}\right[,  \quad \text{for } j=1, \dots, J-1, \quad I_J = \left[\frac{J-1}{J}, 1\right].
              \end{equation*}
              We want to define $\varphi: \mathbb{R} \to [0, 1]$ a smooth "bump" function supported on $(0,1)$ such that $\varphi(0)=0$ and $\varphi(1)=0$. 
              We set $\varphi(x) \propto \exp(-1/(x(1-x)))$, normalized such that $\|\varphi'\|_\infty = 1$. For each interval $j$, we define a localized perturbation function
              \begin{equation*}
                  \varphi_j(v) = \varphi(Jv - (j-1))\mathbbm{1}\{v \in I_j\}.
              \end{equation*}
              Consider the hypercube described by $\eta = (\eta_1, \dots, \eta_J) \in \{-1, +1\}^J$. For each $\eta$, we define the data-generating distribution $\mathbb{P}_\eta$ as follows
              \begin{itemize}
                  \item The marginal distribution of $V$ is uniform on $[0,1]$.
                  \item The marginal distribution of the sensitive attribute $S$ is uniform on $\{1, 2\}$ (so $w_1=w_2=1/2$).
                  \item Conditioned on $V=v$ and $S=s$, $X$ is distributed such that the output of $f^{*}(X,S)$ follows a mixture of two uniform distributions on disjoint intervals $I_- = [-M, -M/2]$ and $I_+ = [M/2, M]$.
              \end{itemize}
             The distribution of $f^*(X,S)$ is fully characterized by the probability of $f^*(X,S)$ falling in the positive interval $I_+$. In the second group, this probability is constant, equal to $1/2$. In the first group, it is denoted by $\theta_s^{(\eta)}(v)$. We define
              \begin{align}
                  \theta_1^{(\eta)}(v) &= \frac{1}{2} + \gamma \sum_{j=1}^J \eta_j \varphi_j(v), \label{eq:theta_1} \\
                  \theta_2^{(\eta)}(v) &= \frac{1}{2}, \label{eq:theta_2}
              \end{align}
              where $0< \gamma \leq 1/4 $ is a small amplitude parameter to be chosen later.\\
              Thus, the distribution of $f^*(X,S)$ in the second group is fixed throughout the different values of $V$, and is equally distributed in $I_+$ and $I_-$, while in the first group it is fixed, with a small excess of mass in either $I_+$ and $I_-$. To make the distributions non distinguishable we need to choose $\gamma$ not too large w.r.t. $J$ and $n$. With this construction, the barycenter of these distributions depends on this small excess of mass in group 1. In what follows we will show that the estimator cannot determine if the corresponding bump is positive of negative leveraging Assouad's lemma \ref{lem:assouad}.

              \paragraph{Verifying the regularity assumption.}
              We must ensure that for all $\eta$, $\mathbb{P}_\eta \in \mathcal{P}(M, \Lcdf)$.
              By construction, the support is contained in $I_- \cup I_+ \subset [-M, M]$, and $\nu_{f^*\vert v,s}$ is non-atomic, thus satisfying Property~\ref{ass:distrib_densities}.
              
              We now verify Property~\ref{ass:lipschitz_cdf}. In our setting, the c.d.f. $F_{f^*|v,s}(t)$ depends linearly on the weight $\theta_s^{(\eta)}(v)$. Thus, for any $t \in [-M,M]$, the mapping $v \mapsto F_{f^*|v,s}(t)$ is Lipschitz if and only if the parameter function $\theta_s^{(\eta)}(v)$ is Lipschitz.
              
              By construction, $\theta_1^{(\eta)}$ is a differentiable function. Its derivative is bounded by
              \begin{equation*}
                  |\theta_1'^{(\eta)}(v)| = \left| \gamma \sum_{j=1}^J \eta_j \cdot J \cdot \varphi_j'(Jv - (j-1)) \right| \le \gamma J \|\varphi'\|_\infty.
              \end{equation*}
              To satisfy the Lipschitz condition with constant $\Lcdf$, we must enforce 
              \begin{equation}
                   \gamma J \|\varphi'\|_\infty \le \Lcdf \implies \gamma \le c_2 J^{-1},
                   \label{eq:constraint_lipschitz}
              \end{equation}
              where $c_2 = \frac{\Lcdf}{ \Vert \varphi' \Vert_{\infty}}$.
              This is the first condition on the parameters.

              \paragraph{Decomposition of the MSE.}
              The oracle fair predictor $f_0^{*}$ corresponds to the Wasserstein barycenter of the two groups. 
              Given the disjoint supports of our construction, this barycenter is a mixture of three distributions: one supported on the negative interval $I_-$, one on the positive interval $I_+$, and one on the middle interval $I_0 = [-M/4, M/4]$. The weights of these components are fully determined by the parameter $\theta^{*,(\eta)}$ defined as
              \begin{align*}
                  \theta^{*,(\eta)}(v) &= \frac{1}{2}\theta_1^{(\eta)}(v) + \frac{1}{2}\theta_2^{(\eta)}(v) \\
                  &= \frac{1}{2} + \frac{\gamma}{2} \sum_{j=1}^J \eta_j \varphi_j(v).
              \end{align*}
              Specifically, the barycenter assigns mass $\theta^*$ to $I_+ \cup I_0$, and mass $1-\theta^*$ to $I_-$.
              Consequently, the optimal transport for the first group consists of shifting a probability mass of at least $|\theta_1^{(\eta)} - \theta^{*,(\eta)}|$ to align with this target.
        
              To achieve this fair distribution, the optimal fair predictor must transport a mass $\delta_{opt}^{(\eta)}(v)$ from $I_+\cup I_0$ to $I_-$, where
              \begin{align*}
                  \delta_{opt}^{(\eta)}(v) &:= \theta_1^{(\eta)}(v) - \theta^{*,(\eta)}(v), \\
                  &= \frac{\gamma}{2} \sum_{j=1}^J \eta_j \varphi_j(v).
              \end{align*}

            \paragraph{Connection between risk and Wasserstein.}
              We first establish a lower bound on the error of any predictor relative to the fair optimum.
              For any input $v$ and group $s$, the pair of random variables $(\widehat{f}_n(X,S), f_0^{*}(X,S))$ constitutes a valid coupling of the distributions $\mu_{\widehat{f}|v,s}$ and $\mu_{f_0^{*}|v,s}$.
              Since the squared Wasserstein distance $\mathcal{W}_2^2$ is defined as the infimum of the squared error over all possible couplings, it is necessarily upper-bounded by the error of this specific coupling induced by the estimator:
              \begin{equation}
                   \mathbb{E}\left[ (\widehat{f}_n - f_0^*)^2 \mid V=v, S=s \right] \;\ge\; \mathcal{W}_2^2(\mu_{\widehat{f}|v,s}, \mu_{f_0^*|v,s}).
              \end{equation}
              Aggregating over $V$ and $S$, we obtain the bound:
              \begin{equation}
                  \label{eq:mse_w2_bound}
                  \mathbb{E}\left[ \|\widehat{f}_n - f_0^*\|^2 \right] \;\ge\; \sum_{s=1}^K w_s \mathbb{E}_V \left[ \mathcal{W}_2^2(\mu_{\widehat{f}|v,s}, \mu_{f_0^*|v,s}) \right]
              \end{equation}
              where $\mathbb{E}_V$ is the expectation taken wrt $V$. Recall that our construction has a gap $M/4$ between the negative interval $I_-$ and the positive intervals $I_+\cup I_0$. 
              For any two distributions supported on $I_+\cup I_0$ and $I_-$, the Wasserstein distance is naturally lower-bounded by the cost of moving the net probability mass difference from one interval to the other.
              Specifically, if an estimator $\widehat{f}_n$ puts a different amount of mass in $I_+\cup I_0$ than the target $f_0^*$, it incurs a penalty
              \begin{equation*}
                  \mathcal{W}_2^2(\mu_{\widehat{f}_n}, \mu_{f_0}) \;\ge\; M^2/16 \left| \mathbb{P}(\widehat{f}_n \in I_+\cup I_0) - \mathbb{P}(f_0^* \in I_+\cup I_0) \right|.
              \end{equation*}
              We define the estimator $\widehat{\delta}(v) := \mathbb{P}(\widehat{f}_n \in I_+\cup I_0) - \theta^*(v)$ as the deviation from the reference parameter.
              Since the optimal transport $\delta_{opt}^{(\eta)}$ is exactly the mass adjustment required for the first group, the term on the right becomes precisely $|\widehat{\delta}(v) - \delta_{opt}^{(\eta)}(v)|$.
              Integrating this lower bound over $V$ yields
              \begin{equation*}
                  \mathbb{E}\left[ \|\widehat{f}_n - f_0^*\|^2 \right] 
                  \;\ge\; w_1 M^2/16 \cdot \mathbb{E} \left[ \| \widehat{\delta} - \delta_{opt}^{(\eta)} \|_{L^1} \right].
              \end{equation*}
      
              Let $\varphi(\eta) := \delta_{opt}^{(\eta)} \in L^1([0,1])$ be the parameter of interest.
          
              \paragraph{Verifying Assouad's Lemma's conditions.}
              We apply Lemma~\ref{lem:assouad} to the estimation of $\phi(\eta)$. To do so, we need to calculate the separation $\Delta$ and the KL divergence $\kappa$.  We start by the Separation $\Delta$.
              Consider two configurations $\eta, \eta'$ that differ exactly at one coordinate $j$ (meaning they have an Hamming distance of 1).
              \begin{align*}
                  d(\phi(\eta), \phi(\eta')) &= \| \delta_{opt}^{(\eta)} - \delta_{opt}^{(\eta')} \|_{L^1} \\
                  &= \int_{0}^{1} \left| \frac{\gamma}{2} \eta_j \varphi_j(v) - \frac{\gamma}{2} \eta'_j \varphi_j(v) \right| dv \\
                  &= \gamma \int_{I_j} |\varphi(Jv - (j-1))| dv \\
                  &= \frac{\gamma}{J} \|\varphi\|_{L^1}.
              \end{align*}
              Thus, we set $\Delta = \frac{\gamma}{J} \|\varphi\|_{L^1}$.
              
              Then we compute the KL Divergence $\kappa$.
              The dataset consists of $n$ i.i.d. observations $(V_i, S_i, Y_i)$ where $Y_i= f^*(X_i, S_i)$. The joint distribution $\mathbb{P}_\eta$ factors as $P(v)P(s)P_\eta(y|v,s)$. 
              Since the marginal distributions of the covariates $V$ and sensitive attribute $S$ do not depend on $\eta$, the KL divergence simplifies using the chain rule and the independence of samples is given by
              \begin{align*}
                  \mathrm{KL}(\mathbb{P}_\eta^n \| \mathbb{P}_{\eta'}^n) 
                  &= n \cdot \mathrm{KL}(\mathbb{P}_\eta \| \mathbb{P}_{\eta'}),  \\
                  &= n \cdot \mathbb{E}_{V, S} \left[ \mathrm{KL}(\mu_{f|V,S}^{(\eta)} \| \mu_{f|V,S}^{(\eta')}) \right],  \\
                  &= n \int_0^1 \sum_{s=1}^2 w_s \mathrm{KL}(\mu_{f|v,s}^{(\eta)} \| \mu_{f|v,s}^{(\eta')}) \, dv.
              \end{align*}
              Now, we can use the specific properties of our construction where $\eta$ and $\eta'$ differ only at index $j$.
              For the second group, the distributions are identical ($\theta_2^{(\eta)} = \theta_2^{(\eta')} = 1/2$), so the KL term is 0.
              For the first, the distributions differ only for $v \in I_j$. Outside this interval, the term is 0.
              Finally, the distribution $\mu_{f|v,1}$ is a mixture of two disjoint uniforms with weights $\theta_1(v)$ and $1-\theta_1(v)$. 
              Distinguishing two such mixtures equivalent to distinguishing two Bernoulli distributions with parameters corresponding to the weights.
      
              Applying these simplifications, we have
              \begin{align*}
                  \mathrm{KL}(\mathbb{P}_\eta^n \| \mathbb{P}_{\eta'}^n) 
                  &= n \int_{I_j} w_1 \, \mathrm{KL}(\mu_{f|v,1}^{(\eta)} \| \mu_{f|v,1}^{(\eta')}) \, dv,  \\
                  &= n w_1 \int_{I_j} \mathrm{KL}\Big( \mathrm{Bern}(\theta_1^{(\eta)}(v)) \, \Big\| \, \mathrm{Bern}(\theta_1^{(\eta')}(v)) \Big) \, dv.
              \end{align*}
      
              Using the inequality $\mathrm{KL}(\text{Bern}(p) \| \text{Bern}(q)) \le \frac{(p-q)^2}{q(1-q)}$ and noting that probabilities are bounded in $[1/4, 3/4]$ (for small $\gamma$), we have a constant $c_4$ such that 
              \begin{align*}
                  \kappa &\le c_4 w_1\, n \int_{I_j} \left( \theta_1^{(\eta)}(v) - \theta_1^{(\eta')}(v) \right)^2 dv \\
                  &= c_4 w_1\, n \int_{I_j} \left( 2\gamma \varphi_j(v) \right)^2 dv \\
                  &= c_5 \, \frac{n \gamma^2}{J},
              \end{align*}
              with constant $c_5= 32w_1 \Vert \varphi \Vert_{L^2}^2$,
              because the $L^2$ norm of the bump is a constant.
              Assouad's Lemma requires $1 - \sqrt{\kappa/2} > 0$. We choose $\gamma$ and $J$ such that $\kappa \le 1$, which imposes the constraint
              \begin{equation}
                  \label{eq:constraint_kl}
                  \frac{n \gamma^2}{J} \le c_5^{-1} \implies \gamma \le c_6 \sqrt{\frac{J}{n}},
              \end{equation}
              where $c_6 = \sqrt{1/\left(32w_1 \Vert \varphi \Vert_{L^2}^2\right)}$.
              This is the second constraint on our parameters.

              \paragraph{Conclusion.}
              By Lemma~\ref{lem:assouad}, we have the lower bound
              \begin{align*}
                  \sup_{\eta} \mathbb{E} \| \widehat{\delta} - \delta_{opt}^{(\eta)} \|_{L^1} &\;\ge\; \frac{J \Delta}{2} \left(1 - \sqrt{\frac{\kappa}{2}}\right) , \\
                  &\;\ge\; c_8 \, J \left( \frac{\gamma}{J} \right) , \\
                  &\;=\; c_8 \, \gamma,
              \end{align*}
              where $c_8 = \frac{2-\sqrt{2}}{4}$.
              We maximize this bound where $\gamma$ subject to constraints \eqref{eq:constraint_lipschitz} and \eqref{eq:constraint_kl}, i.e. when $\gamma \leq c_2J^{-1}$ and $\gamma \leq c_6 \sqrt{\frac{J}{n}}$.
              Solving for $J$ we get 
              \begin{equation*}
                  J = \left( \frac{32\Lcdf^2 w_1 \Vert \varphi \Vert_{L^2}^2n}{\Vert \varphi' \Vert_{\infty}^2} \right)^{1/3}.
              \end{equation*}
              and
              \begin{equation*}
                  \gamma = \left( \frac{\Lcdf}{32 w_1 \Vert \varphi \Vert_{L^2}^2 \Vert \varphi' \Vert_{\infty} n } \right)^{1/3}, 
              \end{equation*}
              so as to verify the constraints and obtain

              \begin{equation*}
                  \mathbb{E} \| \widehat{\delta} - \delta_{opt} \|_{L^1} \ge c_9 \, n^{-1/3}.
              \end{equation*}
              where $c_9 = \left( \frac{\Lcdf}{32 w_1 \Vert \varphi \Vert_{L^2}^2 \Vert \varphi' \Vert_{\infty} } \right)^{1/3} $.
              Finally 
              \begin{align*}
                \mathbb{E}\left[ \|\widehat{f}_n - f_0^*\|^2 \right] 
                &\;\ge\; c_2 M^2 (c_9 n^{-1/3}), \\
                &\;=\; c_{10} \, n^{-1/3}.
              \end{align*}
              where $c_{10} = M^2 \left( \frac{\Lcdf^4}{32 w_1 \Vert \varphi \Vert_{L^2}^2 \Vert \varphi' \Vert_{\infty}^4 } \right)^{1/3}$.

    This concludes the proof of Theorem \ref{thm:mse_lower_bound}.

\subsection{Proof of Corollary \ref{cor:risk_lower_bound}}    
We establish the lower bound on the excess risk by considering the risk relative to the optimal predictor within the constrained set $F_{u_n} = \{ f : \mathcal{U}(f) \le u_n \}$ where $u_n = C_1 \left( \frac{\log(n)}{n}\right)^{1/3}$ with $C_1$ a constant to be determined. We also define the event $\mathcal{E} = \{\widehat{f}_n \in F_{u_n}\}$. Recall that $f^*_{u_n} = \argmin_{f \in F_{u_n}} \|f - f^*\|^2$ is the projection of the unconstrained Bayes optimum $f^*$ onto $F_{u_n}$, where $\| \cdot \|$ is the $L_2$-norm.

By the property of projection onto a convex set, on the event $\mathcal{E}$ the excess risk is lower-bounded by the squared distance since
              \begin{align*}
                  \mathcal{R}(\widehat{f}_n) - \mathcal{R}(f^*_{u_n}) &= \|\widehat{f}_n - f^*\|^2 - \|f^*_{u_n} - f^*\|^2  \\
                  & = \|\widehat{f}_n - f^*_{u_n}\|^2 + 2 \left \langle \widehat{f}_n - f^*_{u_n}\vert f^*_{u_n} - f^* \right\rangle \\
                  &\ge \|\widehat{f}_n - f^*_{u_n}\|^2.
              \end{align*}
Now, $\mathcal{R}(\widehat{f}_n)- \mathcal{R}(f^*_{u_n}) \geq - M^2$, so
\begin{align*}
\mathbb{E}[\mathcal{R}(\widehat{f}_n)] - \mathcal{R}(f^*_{u_n}) &\geq \mathbb{E}[\mathcal{R}(\widehat{f}_n)- \mathcal{R}(f^*_{u_n})\vert \mathcal{E}]\mathbb{P}(\mathcal{E}) - \mathbb{P}(\overline{\mathcal{E}})M^2\\
&\geq\frac{1}{2}\mathbb{E}[ \|\widehat{f}_n - f^*_{u_n}\|^2\vert \mathcal{E}] - C_2M^2n^{-1/3}
 \end{align*}
and
\begin{align*}
\mathbb{E}[\|\widehat{f}_n - f^*_{u_n}\|^2\vert \mathcal{E}] &= \frac{\mathbb{E}[\|\widehat{f}_n - f^*_{u_n}\|^2] - \mathbb{E}[ \|\widehat{f}_n - f^*_{u_n}\|^2\vert \overline{\mathcal{E}}]\mathbb{P}(\overline{\mathcal{E}})}{\mathbb{P}(\mathcal{E})}\\
& \geq \mathbb{E}[\|\widehat{f}_n - f^*_{u_n}\|^2]\left(1 - \frac{M^2 \mathbb{P}(\overline{\mathcal{E}})}{\mathbb{E}[\|\widehat{f}_n - f^*_{u_n}\|^2]}\right).
 \end{align*}
To obtain a lower bound on $\mathbb{E}[\|\widehat{f}_n - f^*_{u_n}\|^2]$, we use the expansion $\|a-b\|^2 \ge \|a\|^2 - 2\|a\|\|b\|$. Setting $a = \widehat{f}_n - f^*_0$ and $b = f^*_{u_n} - f^*_0$, and taking the expectation over $\mathcal{D}_n$
              \begin{equation*}
                  \mathbb{E}\left[ \|\widehat{f}_n - f^*_{u_n}\|^2 \right] \ge \mathbb{E}\left[ \|\widehat{f}_n - f^*_0\|^2 \right] - 2 \|f^*_{u_n} - f^*_0\| \, \mathbb{E}\left[ \|\widehat{f}_n - f^*_0\| \right].
              \end{equation*}
              Applying the Cauchy-Schwarz inequality $\mathbb{E}[Z] \le \sqrt{\mathbb{E}[Z^2]}$ to the term on the right, we obtain
              \begin{equation}\label{eq:triangulation_inequality}
                   \mathbb{E}\left[ \|\widehat{f}_n - f^*_{u_n}\|^2 \right] \ge \mathbb{E}\left[ \|\widehat{f}_n - f^*_0\|^2 \right] \left( 1 - 2 \frac{\|f^*_{u_n} - f^*_0\|}{\sqrt{\mathbb{E}[\|\widehat{f}_n - f^*_0\|^2]}} \right).
              \end{equation}
              From the fairness rate assumption, we have the deterministic upper bound 
              \begin{equation*}
                  \|f^*_0 - f^*_{u_n}\|^2 \le u_n \le C_1 (\log(n)/n)^{-1/3}.
            \end{equation*}
            Now, Theorem \ref{thm:mse_lower_bound} implies that for some $f^*$, 
            \begin{equation*}
                  \mathbb{E}[\|\widehat{f}_n - f^*_0\|^2] \ge c_{10} n^{-1/3}
            \end{equation*}
            for the constant $c_{10}$ specified in the proof of this Theorem.
            Thus, for $C_1$ small enough, we have
    \begin{equation}\label{eq:triangulation_inequality}
                   \mathbb{E}\left[ \|\widehat{f}_n - f^*_{u_n}\|^2 \right] \ge \frac{1}{3}\mathbb{E}\left[ \|\widehat{f}_n - f^*_0\|^2 \right] \geq \frac{c_{10}n^{-1/3}}{2},
    \end{equation}
    and for $C_2$ small enough
\begin{align*}
\mathbb{E}[\|\widehat{f}_n - f^*_{u_n}\|^2\vert \mathcal{E}] & \geq \frac{1}{2}\mathbb{E}[\|\widehat{f}_n - f^*_{u_n}\|^2]\\
&\geq \frac{c_{10}n^{-1/3}}{4}.
 \end{align*}
            Substituting these into \eqref{eq:triangulation_inequality}, and choosing $C_1$ such that $C_1 \leq c_{10}/16$ we have 
              \begin{equation*}
                  \mathcal{R}(\widehat{f}_n) - \mathcal{R}(f^*_{u_n})\;\ge\; c_{10} n^{-1/3} \left( 1 - 2 \sqrt{\frac{C_1}{c_{10}}} \right)- C_2M^2n^{-1/3} \geq c_{R}n^{-1/3}.
              \end{equation*}
              for $C_2$ small enough, and some positive constant $c_{R}$.

              \paragraph{Extension to general $K$.}
              The lower bound construction derived above for the binary case ($K=2$) extends naturally to the general setting with $K \ge 2$ sensitive groups. 
              To establish the lower bound, it suffices to construct a specific "least favorable" subfamily of distributions where the learning task is maximally difficult.
              We achieve this by fixing the distributions for $K-1$ groups (e.g., $s=2, \dots, K$) to be static and uninformative (e.g., $\theta_s(v) = 1/2$), while perturbing only the first group $s=1$ using the hypercube construction $\theta_1^{(\eta)}$ defined in Eq.~\eqref{eq:theta_1}.
              
              In this "one-vs-rest" scenario, the optimal fair predictor must still estimate the parameter $\theta_1^{(\eta)}$ to align the first group with the barycenter. The difficulty of this estimation is determined by the effective sample size for the perturbed group, which is approximately $n w_1$.
              The KL divergence calculation in the application of Assouad's Lemma is modified simply by restricting the sum over groups to the single perturbed component:
              \begin{equation*}
                  \mathrm{KL}(\mathbb{P}_\eta^n \| \mathbb{P}_{\eta'}^n) = n \sum_{s=1}^K w_s \mathrm{KL}(\mu_{f|v,s}^{(\eta)} \| \mu_{f|v,s}^{(\eta')}) = n w_1 \mathrm{KL}(\mu_{f|v,1}^{(\eta)} \| \mu_{f|v,1}^{(\eta')}).
              \end{equation*}
              This replaces the factor $n/2$ (from $w_1=1/2$) with $n w_1$ in the indistinguishability constraint. 
              Consequently, the balancing of the bias ($J^{-1}$) and variance ($\sqrt{J/(n w_1)}$) yields the same rate $J \asymp (n w_1)^{1/3}$. 
              Since the group weights $w_s$ are constants bounded away from zero, the asymptotic convergence rate remains $n^{-1/3}$, confirming that the difficulty of the problem does not decrease with additional uninformative groups.
      
              This completes the proof.

    \subsection{Proof of Theorem~\ref{thm:relaxed_l_star}} \label{sec:proof_relaxed_l_star}

    In the relaxed fairness setting, we impose a global budget $\mathcal{B}$. 
    We first prove a lemma showing that interval unfairness is convex under affine mixtures of predictors. 
    This lemma then lets us decompose the unfairness bound for the relaxed post-processed predictor into two parts: a discretization term and a relaxation term.
    The proof of the following lemma can be found in \ref{proof:lem_interval_convexity}.
\begin{lemma}[Convexity of interval unfairness under affine mixing]\label{lem:interval-convexity}
  Fix an interval $\mathcal I_\ell$ and group weights $(w_s)_{s=1}^K$.
  For all two predictors $f,g \in \mathcal F$ with bounded distribution support, and any $\lambda\in[0,1]$, define the mixture
  \begin{equation*}
  h \;:=\; \lambda f + (1-\lambda) g.
\end{equation*}
  Then the interval-level unfairness verifies
  \begin{equation*}
  \mathcal U_\ell(h)\ \le\ \lambda\,\mathcal U_\ell(f)\ +\ (1-\lambda)\,\mathcal U_\ell(g).
  \end{equation*}
  That is, $\mathcal U_\ell$ is convex with respect to mixtures of predictors.
  \end{lemma}

Recall that $\widehat f _{\alpha, L}= \sqrt{\alpha}f^{bb} + (1- \sqrt{\alpha})\widehat f_L$, Lemma  \ref{lem:interval-convexity} allows us to bound the unfairness of the relaxed post-processed predictor by the sum of two terms: the unfairness of the black-box predictor and the unfairness of the optimal relaxed predictor.
Since the unfairness of the black-box predictor is unknown, we will use Lemma~\ref{lem:unfairness_bb_bound} to upper bound it with high probability.
The proof of Lemma~\ref{lem:unfairness_bb_bound} can be found in \ref{proof:lem_unfairness_bb_bound}.

For notational simplicity, when no ambiguity arises we denote by
$\widehat\mu_{\ell,s} := \mathrm{Law}\!\big(f^{bb}(X,s)\mid \U\in\mathcal I_\ell,S=s\big)$
the (empirical) conditional distribution of the black‑box predictor on interval $\ell$ and group $s$. 
Similarly, we denote by $\mu_{\ell,s} := \mathrm{Law}\!\big(f^{bb}(X,s)\mid \U\in\mathcal I_\ell,S=s\big)$ the true conditional distribution and $\mu_{\ell,\star}$ the corresponding barycenter.

\begin{lemma}\label{lem:unfairness_bb_bound}
  Assume $f^{bb}$ satisfies Property~\ref{ass:distrib_densities}.
  Consider the sample-split introduced in Section~\ref{sec:form_estimator}. Let $\eta_L$ be defined as in Theorem~\ref{thm:l_star}. 
  We denote $\widehat\mu_{\ell,\star}$ the distribution of the empirical quantile of the barycenter,
  and the empirical unfairness
  \begin{equation*}
  \widehat{\mathcal U}_\ell(f^{bb})\ :=\ \sum_{s=1}^K  w_s\,\mathcal W_2^2\left(\widehat\mu_{\ell,s},\widehat\mu_{\ell,\star}\right),
  \qquad
  \widehat{\mathcal U}_L(f^{bb})\ :=\ \sum_{\ell=1}^L p_\ell\,\widehat{\mathcal U}_\ell(f^{bb}).
\end{equation*}

With probability at least
  \begin{equation*}
  1-\;\eta_L,
\end{equation*}

  we have
\begin{equation*}
  \mathcal U(f^{bb})
  \ \le\
  2\,\widehat{\mathcal U}_L(f^{bb}) + 8M^2\sqrt{\frac{2LK}{\,n}\,
  \log\!\left(2KLn\right)} + \frac{16M^2\Lcdf}{L}
\end{equation*}

  \end{lemma}

  We can now prove Theorem~\ref{thm:relaxed_l_star} on the unfairness of the relaxed post-processed predictor $\widehat f_{\alpha,L}$ defined in Eq.\eqref{eq:relaxed_post_processed}.
    \RelaxedLstar* 

\begin{proof}[Proof of Theorem \ref{thm:relaxed_l_star}]\label{proof:relaxed_l_star}
  By Lemma~\ref{lem:interval-convexity}, the interval unfairness $\mathcal U_\ell$ is convex with respect to mixtures of predictors.
   Hence, the unfairness of the post-processed predictor $\widehat f_\ell$ on interval $\ell$ satisfies
   \begin{equation*}
   \mathcal U_\ell(\widehat f_{\alpha, L})\ \le\ \sqrt{\alpha} \ \mathcal U_\ell(f^{bb})+(1-\sqrt{\alpha})\mathcal U_\ell\ (\widehat f_L).
   \end{equation*}
   Summing over $\ell$ and using the fact that $\sum_\ell p_\ell=1$ gives
   \begin{align*}
   \mathcal U(\widehat f_{\alpha, L})\ &\le\ \sqrt{\alpha}\ \mathcal U(f^{bb})\ +\ (1-\sqrt{\alpha})\ \mathcal U(\widehat f_L).
   \end{align*}
   \paragraph{Bound on the unfairness of the post-processed predictor.}
   Since the unfairness of the post-processed predictor is inferior to the unfairness of the black-box predictor (otherwise the post-processing would not be used), the RHS bound is increasing in $\alpha$. 
   Therefore, in order to minimize the bound on the unfairness of $\widehat f_{\alpha, L}$, it suffices to minimize the bound on the unfairness of the post-processed predictor $\widehat f_L$ over $L$.
   So 
   \begin{equation*}
     \mathcal U (\widehat f_{\alpha,L^*})\ \le\ \sqrt{\alpha}\ \mathcal U(f^{bb})\ +\ (1-\sqrt{\alpha})\ \mathcal U(\widehat f_{L^*}).
   \end{equation*}
   By Theorem~\ref{thm:l_star}, for $L=L^*$, we have
   \begin{equation*}
   \mathcal U(\widehat f_{L^*})\ \le\ \delta^*.
   \end{equation*}
   where $\delta^*$ is the bound in Theorem~\ref{thm:l_star}.

   \paragraph{Bound on the unfairness of the black-box predictor.} Lemma~\ref{lem:unfairness_bb_bound} gives a high-probability upper bound on the unfairness of the black-box predictor.
   We have with probability at least $1-\eta_L $
   \begin{equation*}
    \mathcal U(f^{bb})
    \ \le\
    2\,\widehat{\mathcal U}_L(f^{bb}) + 8M^2\sqrt{\frac{2LK}{\,n}\,
    \log\!\left(2KLn\right)} + \frac{16M^2\Lcdf}{L}
  \end{equation*}
   taking $L=L^*$ with 
  \begin{equation*}
    L^* = \left\lfloor \left( \frac{8\Lcdf^2 n}{K \log(2Kn)} \right)^{1/3} \right\rfloor.
\end{equation*}
we have, for $x\geq 4$, $1/\sqrt{\lfloor x \rfloor}\leq 1/\sqrt{x - 1} \leq \sqrt{2/x}$. Then, if $L^*\geq 4$
      \begin{equation*}
        \frac{16 M^2 \Lcdf}{L}\leq 16M^2\left(\frac{\Lcdf  K \log(2Kn)}{n}\right)^{1/3}.
        \end{equation*}
      and for $n\geq \sqrt{\frac{\Lcdf}{M}}$, $\log(2KnL) \leq 2 \log(2Kn)$, thus
      \begin{align*}
        8M^2\,\sqrt{\frac{2L^*K\log\!\left( 2KL^*n\right)}{n}} 
        &= 16 \sqrt{2} M^2 \left(\frac{\Lcdf K \log(2Kn)}{n}\right)^{1/3}.
    \end{align*}
    Now, $2K \geq 4$ so for $n\geq 37$, $\log(2Kn)\leq \log(2K)\log(n)$, thus with probability at least $1-\eta_{L^*}$
   \begin{align*}
    \mathcal{U}(f^{bb}) &\leq
    2\,\widehat{\mathcal U}_{L^*}(f^{bb}) +\delta^*.
   \end{align*}

   For $\mathcal B>\delta$, to choose $\alpha$ so that $\mathcal U(\widehat f_{\alpha,L^*})\le \mathcal B$, it suffices to take
   \begin{equation*}
   \sqrt{\alpha}\ :=\ \frac{\mathcal B - \delta^*}{2\,\widehat{\mathcal U}_{L^*}(f^{bb})},
   \end{equation*}
  If  $\mathcal B<\delta^*$ the problem is infeasible because the budget is already spent due to the discrtization of the exogenous variable.
   Finally, the unfairness of the relaxed post-processed predictor satisfies
   \begin{equation*}
    \mathcal U(\widehat f_{\alpha,L^*})\ \le\ \mathcal B.
   \end{equation*}
   This concludes the proof.
     \end{proof}
 
    \subsection{Proof of Corollary~\ref{cor:risk-relaxed-fair}}\label{proof:corollary_relaxed_risk}
The bound on the risk of the relaxed-fair predictor $\widehat f_{\alpha^*, L^*}$ follows directly from the expression of the predictor and Theorem~\ref{cor:risk-exact-fair}.
\RiskRelaxedFair*
\begin{proof}
  Plugging in the expression of $\widehat f_{\alpha^*, L^*}$ in the risk definition, we have
  \begin{align*}
    \mathcal{R}(\widehat f_{\alpha^*, L^*}) &= \sum_\ell \sum_s \mathbb{E}[(\widehat f_{\alpha^*, L^*}(\mathbf{X}, S) - f^*(\mathbf{X}, S))^2\vert \U \in \mathcal{I}_\ell, S=s], \\
    &= \sum_\ell \sum_s \mathbb{E}[((1- \sqrt{\alpha^*})\widehat f_{L^*}(\mathbf{X}, S) + \sqrt{\alpha^*}f^{bb}(\mathbf{X}, S) - f^*(\mathbf{X}, S))^2\vert \U \in \mathcal{I}_\ell, S=s], \\
    &= \sum_\ell \sum_s  \mathbb{E}[(1 - \sqrt{\alpha^*})(\widehat f_{L^*}(\mathbf{X}, S)-f^*(\mathbf{X}, S)) -\sqrt{\alpha^*}(f^{bb}(\mathbf{X}, S)- f^*(\mathbf{X}, S))^2\vert \U \in \mathcal{I}_\ell, S=s], \\
    &\leq (1-\sqrt{\alpha^*})\mathcal{R}(\widehat f_{L^*}) + \sqrt{\alpha^*}\mathcal{R}(f^{bb}).
  \end{align*}
  With the expression we just showed for $\mathcal{R}(\widehat f_{L^*})$ in Theorem~\ref{cor:risk-exact-fair}, we obtain the result.
\end{proof}

\subsection{Proof of auxiliary lemmas}\label{sec:proof_auxiliary_lemmas}

\subsubsection{Proof of Lemma~\ref{lem:stability_monotone}} \label{proof:lem_stability_monotone}

  Fix an interval $\mathcal I_\ell$ and a group $s\in[K]$. 
  Write, for brevity,
  \begin{equation*}
  F_{\mu_{f|\ell,s}}:=F_{f\mid \ell,s},\qquad 
  F^{-1}_{\ell, \star}:=\sum_{s'=1}^K w_{s'}\,F^{-1}_{f\mid \ell, s'},
  \qquad 
  F_{\ell,\star}:= \left(F^{-1}_{\ell, \star}\right)^{-1}.
\end{equation*}
Following the plug-in maps using sample splitting 
(as introduced in Section~\ref{sec:form_estimator}), define the local post-processing map 
\begin{equation*}
\widehat T_{\ell,s}(\cdot):=\left(\sum_{s'} \widehat w_{s'}\,\widehat F^{-1}_{f^{bb}\mid \ell,s'}\right)\circ \widehat F_{f^{bb}\mid \ell,s}(\cdot),
\end{equation*}

  By definition,
  \begin{equation*}
  \widehat T_{\ell,s}(t)= \widehat F^{-1}_{\ell, \star}\!\big(\widehat F_{{f|\ell,s}}(t)\big),\qquad 
  \widehat f_\ell(\mathbf x,s)= \widehat T_{\ell,s}\big(f^{bb}(\mathbf x,s)\big).
\end{equation*}

\paragraph{Bounded support is preserved.}
Condition on the training sample used to build $\widehat F_{f^{bb}\mid \ell,\cdot}$ and $\{\widehat F^{-1}_{f^{bb}\mid \ell,\cdot}\}$. 
By Property~\ref{ass:distrib_densities}, for all $(v,s)$ the score $f^{bb}(X,s)\mid(V=v,S=s)$ is supported in $[-M,M]$. 
Hence each empirical law $\widehat\mu^{j}_{f^{bb}\mid \ell,s'}$ (and thus its quantile $\widehat F^{-1}_{f^{bb}\mid \ell,s'}$) takes values in $[-M,M]$. 
Therefore their convex combination $\widehat F^{-1}_{\ell,\star}:=\sum_{s'} \widehat w_{s'}\,\widehat F^{-1}_{f^{bb}\mid \ell,s'}$ also maps $[0,1]$ into $[-M,M]$. 
Since $\widehat F_{f^{bb}\mid \ell,s}$ takes values in $[0,1]$, the plug–in transport
\begin{equation*}
\widehat T_{\ell,s}:=\widehat F^{-1}_{\ell,\star}\circ \widehat F_{f^{bb}\mid \ell,s}
\end{equation*}
maps $[-M,M]$ into $[-M,M]$. Consequently,
\begin{equation*}
  \mathrm{Supp}\big(\mu_{\widehat f_\ell\mid v,s}\big)\ \subseteq\ [-M,M]
\quad\text{for }\nu\text{-a.e.\ }v\in\mathcal I_\ell,
\end{equation*}

and in particular $\sup_{v,s}\int x^2\,d\mu_{\widehat f_\ell\mid v,s}(x)\le M^2$.

  \paragraph{Lipschitz continuity in $v$ of the c.d.f. is preserved.}

Let $\widehat G_{v,s}$ be the conditional c.d.f.\ of $\widehat f_\ell(X,s)=\widehat T_{\ell,s}(f^{bb}(X,s))$ given $(V=v,S=s)$.
Then, for all $v,v'\in \mathcal I_\ell$, we show that
\begin{equation*}
\sup_{t\in\mathbb R}\big|\widehat G_{v,s}(t)-\widehat G_{v',s}(t)\big| \;\le\; \Lcdf\,|v-v'|.
\end{equation*}

Indeed, conditioned on the training sample, $\widehat T_{\ell,s}$ is deterministic and does not depend on $v$.
Since $\widehat T_{\ell,s}$ is nondecreasing, for all $t$ let $\widehat T_{\ell,s}^{-1}(t):=\inf\{x:\widehat T_{\ell,s}(x)\ge t\}$ as its generalized inverse. 
Then, we can define
\begin{align*}
\widehat G_{v,s}(t)&=\mathbb P\!\big(\widehat T_{\ell,s}(f^{bb}(X,S))\le t\mid V=v,S=s\big), \\
&= F_{f^{bb}\mid v,s}\big(\widehat T_{\ell,s}^{-1}(t)\big).
\end{align*}

Hence for $v,v'\in\mathcal I_\ell$,
\begin{align*}
\sup_t\big|\widehat G_{v,s}(t)-\widehat G_{v',s}(t)\big|
&= \sup_t\big|F_{f^{bb}\mid v,s}(\widehat T_{\ell,s}^{-1}(t)) - F_{f^{bb}\mid v',s}(\widehat T_{\ell,s}^{-1}(t))\big|, \\
&\le \sup_x\big|F_{f^{bb}\mid v,s}(x)-F_{f^{bb}\mid v',s}(x)\big|,\\
&\le \Lcdf|v-v'|.
\end{align*}
  This concludes the proof. 

\subsubsection{Proof of Proposition~\ref{prop:barycenter_property_M}} \label{proof:prop_5}

  We prove the result by reducing the problem to the barycenter problem in the metric space $(\mathcal{P}_2(\mathbb{R}),\mathcal{W}_2)$ which is known to satisfy the barycenter property (c.f. Remark~\ref{remark:barycenter_W2}).
  We first show that the barycenter $C_{\mathbf{a}_{\mathbf{w}}}$ exists and belongs to $\mathcal{M}$, then we prove that it is optimal.

Let $v\in\mathcal{V}$ and $(a_v^{(s)})_{s \in [K]}\in \mathcal{P}_2(\mathbb{R})$. Since the metric space $(\mathcal{P}_2(\mathbb{R}),\mathcal{W}_2)$ satisfies the barycenter property (c.f. Remark~\ref{remark:barycenter_W2}), it follows that the barycenter of $(a_v^{(s)})_{s \in [K]}$ exists, and we denote it by $b_v$. We have 
\begin{equation} \label{an:bary_quantile}
b_v \in \arg\min_{\rho\in\mathcal{P}_2(\mathbb{R})}\ \sum_{s=1}^K w_s\,\mathcal{W}_2^2\!\big(a^{(s)}_v,\rho\big),
\qquad
F^{-1}_{b_v}(t)=\sum_{s=1}^K w_s F^{-1}_{a^{(s)}_v}(t),
\end{equation}
where the second equation follows from the Lemmas~\ref{prop:cumulative_bary} and \ref{prop:equ_cdf_quantile} reminded at the begining of the Appendix.\\
To prove Proposition~\ref{prop:barycenter_property_M}, we define the candidate barycenter $C_{\mathbf{a}_{\mathbf{w}}}$, and we show that it is the barycenter of $(a^{(s)})_{s \in [K]}$ in the metric space $(\mathcal{M},\bar d)$.
Define $C_{\mathbf{a}_{\mathbf{w}}}:=(b_v)_{v\in\mathcal{V}}$. 

We start by showing that $C_{\mathbf{a}_{\mathbf{w}}}\in \mathcal{M}$.
From the formula of the quantile of $b_v$ in Eq.~\eqref{an:bary_quantile} we have
\begin{align}
\int x^2 db_v(x) &=\int_0^1 \Big|\sum_{s=1}^K w_s F^{-1}_{a^{(s)}_v}(t)\Big|^2 dt,
\end{align}
because if $Z\sim \mathcal{U}([0,1])$ then $F^{-1}_{b_v}(Z)\sim b_v$.
Then, using Jensen's inequality along with the earlier argument regarding the c.d.f. of $a_v^{(s)}$, we obtain
\begin{align*}
  \int x^2 db_v(x)  &\le\ \sum_{s=1}^K w_s \int_0^1 |F^{-1}_{a^{(s)}_v}(t)|^2 dt,\\
&= \sum_{s=1}^K w_s \int x^2 da^{(s)}_v(x).
\end{align*}
After integrating with respect to $v$, we obtain
\begin{align*}
\int \left(\int x^2 db_v(x)\right) d\nu(v)\le \sum_s w_s \int \int \left( x^2 da^{(s)}_v(x) \right) d\nu(v)<\infty,
\end{align*}
so $C_{\mathbf{a}_{\mathbf{w}}}\in\mathcal{M}$. Measurability in $v$ follows from the measurability of $v\mapsto F^{-1}_{a^{(s)}_v}(t)$
for each $t$ and the fact that $b_v$ is given by their convex combination.

Next we show that $C_{\mathbf{a}_{\mathbf{w}}}$ is optimal.
For all competitor $C=(c_v)_{v}\in\mathcal{M}$,
\begin{align*}
\sum_{s=1}^K w_s\,\bar d^{\,2}\!\big(a^{(s)},C\big)
&=\int_{\mathcal{V}} \sum_{s=1}^K w_s\,\mathcal{W}_2^2\!\big(a^{(s)}_v,c_v\big)\,d\nu(v), \\
&\ge\ \int_{\mathcal{V}} \min_{\rho\in\mathcal{P}_2(\mathbb{R})}\sum_{s=1}^K w_s\,\mathcal{W}_2^2\!\big(a^{(s)}_v,\rho\big)\,d\nu(v),
\end{align*}
and the inner minimum is attained at $\rho=b_v$ by definition of the barycenter in the metric space $(\mathcal{P}_2(\mathbb{R}), \mathcal{W}_2)$. Hence
\begin{align*}
\sum_{s=1}^K w_s\,\bar d^{\,2}\!\big(a^{(s)},C\big)\ &\ge\ \int_{\mathcal{V}} \sum_{s=1}^K w_s\,\mathcal{W}_2^2\!\big(a^{(s)}_v,b_v\big)\,d\nu(v),\\
&=\sum_{s=1}^K w_s\,\bar d^{\,2}\!\big(a^{(s)},C_{\mathbf{a}_{\mathbf{w}}}\big),
\end{align*}
which proves that $C_{\mathbf{a}_{\mathbf{w}}}$ is a barycenter.

  \subsubsection{Proof of Lemma \ref{proof:lemma_bias_unfairness}} \label{proof2:lemma_bias_unfairness}
    
    Let $\mathcal{I} \subset [0, 1]$ and $f$ be a predictor with bounded distribution support in $[-M, M]$ satisfying Property~\ref{ass:lipschitz_cdf} on $\mathcal{I}$.
    We aim to show that the map $v \mapsto \mathcal{U}_v(f)$ is Lipschitz continuous on $\mathcal{I}$.
    The unfairness is defined as
    \begin{equation*}
        \mathcal{U}_v(f) = \sum_{s=1}^K w_s \mathcal{W}_2^2(\mu_{f|v,s}, \mu^*_{f|v}).
    \end{equation*}
    We bound the variation of the term $\mathcal{W}_2^2(\mu_{f|v,s}, \mu^*_{f|v})$ with respect to $v$. Let $v, v' \in \mathcal{I}_\ell$.
    Using the quantile representation of the Wasserstein distance
    \begin{align*}
        \Delta_v &= \left| \mathcal{W}_2^2(\mu_{f|v,s}, \mu^*_{f|v}) - \mathcal{W}_2^2(\mu_{f|v',s}, \mu^*_{f|v'}) \right| \\
        &= \left| \int_0^1 \left( F^{-1}_{v,s}(t) - F^{-1}_{*,v}(t) \right)^2 dt - \int_0^1 \left( F^{-1}_{v',s}(t) - F^{-1}_{*,v'}(t) \right)^2 dt \right|,
    \end{align*}
    where $F_{*,v}^{-1} = \sum w_s F_{v,s}^{-1}$.
    Using the identity $a^2 - b^2 = (a-b)(a+b)$, let $A(t) = F^{-1}_{v,s}(t) - F^{-1}_{*,v}(t)$ and $B(t) = F^{-1}_{v',s}(t) - F^{-1}_{*,v'}(t)$.
    Note that since the support is bounded in $[-M, M]$, $|A(t)| \le 2M$ and $|B(t)| \le 2M$, so $|A(t) + B(t)| \le 4M$.
    Thus
    \begin{equation*}
        \Delta_v \le \int_0^1 |A(t) - B(t)| \cdot |A(t) + B(t)| dt \le 4M \int_0^1 |A(t) - B(t)| dt.
    \end{equation*}
    Now substitute the definitions of $A$ and $B$, by the triangle inequality, this integral splits into the sum of two terms
    \begin{align*}
        |A(t) - B(t)| &= \left| (F^{-1}_{v,s} - F^{-1}_{*,v}) - (F^{-1}_{v',s} - F^{-1}_{*,v'}) \right| \\
        &\le \left| F^{-1}_{v,s} - F^{-1}_{v',s} \right| + \left| F^{-1}_{*,v} - F^{-1}_{*,v'} \right|.
    \end{align*}
    We now bound the term $\int_0^1 |A(t) - B(t)| dt$. We have 
    \begin{equation*}
        \int_0^1 |A(t) - B(t)| dt \le \int_{0}^{1}\left(\left| F^{-1}_{v,s}(t) - F^{-1}_{v',s}(t) \right| + \left| F^{-1}_{*,v}(t) - F^{-1}_{*,v'}(t) \right|\right)dt.
    \end{equation*}
  By Proposition 2.17 in \citep{Santambrogio2015}), we have for probabilty measures $\mu, \nu$
    \begin{equation*}
        \int_0^1 |F_\mu^{-1}(t) - F_\nu^{-1}(t)| dt = \int_{\mathbb{R}} |F_\mu(x) - F_\nu(x)| dx.
    \end{equation*}
    Since the distributions are supported on $[-M, M]$, the integral over $\mathbb{R}$ restricts to $[-M, M]$.
    We now apply Property~\ref{ass:lipschitz_cdf}. For the first term:
    \begin{align*}
        \int_{-M}^M |F_{v,s}(x) - F_{v',s}(x)| dx 
        &\le \int_{-M}^M \Lcdf |v-v'| dx \\
        &= 2M \Lcdf |v-v'|.
    \end{align*}
    For the second term (the barycenter), we use the linearity of the quantile function for the Wasserstein barycenter.
    \begin{align}
      \int_0^1\left| F^{-1}_{*,v}(t) - F^{-1}_{*,v'}(t) \right|dt 
        &\le \sum_{s=1}^K w_s \int_0^1 |F_{v,s}^{-1}(t) - F_{v',s}^{-1}(t)| dt, \label{eq:Jensen}\\
        &\le \sum_{s=1}^K w_s (2M \Lcdf |v-v'|), \label{eq:santambrogio}\\
        &= 2M \Lcdf |v-v'|, \notag
    \end{align}
    where equality \eqref{eq:Jensen} follows from the closed-form of the quantile function of the barycenter and inequality \eqref{eq:santambrogio} follows from Jensen's inequality applied to the convex square function. 

    Substituting these back into the bound for $\Delta_v$
    \begin{equation*}
        \Delta_v \le 4M \left( 2M \Lcdf |v-v'| + 2M \Lcdf |v-v'| \right) = 16 M^2 \Lcdf |v-v'|.
    \end{equation*}
    Thus, $\mathcal{U}_v$ is Lipschitz with constant $L_{\mathcal{U}} = 16 M^2 \Lcdf$.

\subsubsection{Proof of Lemma \ref{lemma:KS_W2}} \label{proof:lem_KS_W2}

We want to show the following result for two univariate probability measures $\mu$ and $\nu$ with suport bounded by $M >0$:
    \begin{equation*}
      \mathcal{W}_2^2(\mu, \nu) \leq 4M^2 \sup_{t \in \mathbb{R}} \vert F_\mu(t) - F_\nu(t) \vert.
    \end{equation*}
    Using the properties of the Wasserstein-2 distance stated in Lemma~\ref{prop:wasserstein_quantiles}, we have
    \begin{align*}
      \mathcal{W}_2^2(\mu, \nu) &= \int_0^1 \left\vert F_{\mu}^{-1}(t) - F_{\nu}^{-1} (t) \right\vert^2 dt, \\
      &\leq 2M \int_0^1 \left\vert F_{\mu}^{-1}(t) - F_{\nu}^{-1} (t) \right\vert dt,\\
      &= 2M \int_0^1 \left\vert F_{\mu}(x) - F_{\nu}(x) \right\vert dx,\\
      &\leq 2M^2 \sup_{t \in \mathbb{R}} \vert F_\mu(t) - F_\nu(t) \vert.
    \end{align*}
    Where the second line follows from the compact support assumption and the third line is derived from Proposition 2.17 in \citep{Santambrogio2015}. 

\subsubsection{Proof of Lemma \ref{lem:variance_unfairness}} \label{proof:lem_variance_unfairness}

We prove the following more general result : 
    \begin{lemma}\label{lem:variance_unfairness_general}
        Fix $\ell\in[L]$ and let $\widehat f_\ell$ be the local post-processed predictor and sample-split defined in Section.~\ref{sec:form_estimator}.
        $\widehat f_\ell$ satisfies Property~\ref{ass:distrib_densities}.
        Let $\kappa, \gamma \in (0, 1)$ and define 
        $\eta_{\kappa, \gamma}(L):=\kappa+KL\exp(-\tfrac{\gamma^2}{2}n \frac{\min_s w_{s}}{L})$.
        With probability at least $1-\eta_{\kappa, \gamma}(L)$ uniformly over all $\ell$ and $s$
        \begin{equation}\label{eq:variance_term_bound_main}
        \mathcal{U}_{\ell}(\widehat{f}_\ell) \leq 8M^{2}\!
        \left[
        \sum_{s=1}^K w_s \,\varepsilon_{\ell,s}
        \right],
        \end{equation}
        where
        \begin{equation*}
        \varepsilon_{\ell,s}\ :=\
        \sup_{t} \big| \widehat F_{{f^{bb}|\ell,s}}(t) - F_{{f^{bb}|\ell,s}}(t)\big|
        \ \le\
        \sqrt{\frac{1}{(1-\gamma)\,n\, p_{\ell,s}}\,
\log\!\Big(\frac{2KL}{\kappa}\Big)},
      \end{equation*}
      where $ p_{\ell, s} = \mathbb{P}(\U\in \mathcal{I}_\ell, S=s)$.
\end{lemma}
The proof of Lemma \ref{lem:variance_unfairness} then follows by choosing $\gamma = 1/2$ and $\kappa = \frac{1}{n}$.
          
From the definition of unfairness in Definition~\eqref{def:unfairness}, we know that
          \begin{equation*}
            \mathcal{U}_{\ell}(\widehat{f}_\ell) := \min_{\mu \in \mathcal{P}_2(\mathbb{R})}\sum_{s=1}^K  w_s \mathcal{W}_2^2(\mu_{\widehat{f}_\ell \vert \ell, s}, \mu).
          \end{equation*}
          We denote by $\mu_{\widehat{f}_\ell \vert \ell}^*$ the minimizer of the above expresion. 
          By the expresssion of the cumulative distribution function of the Wasserstein barycenter in Lemma~\ref{prop:cumulative_bary}, we have
          \begin{equation*}
            \widehat{F}_{\mu_{\widehat{f}_\ell \vert \ell}^*}^{-1}(\cdot) = \left( \sum_{s=1}^{K}  w_s \widehat{F}_{{\widehat{f}_\ell \vert \ell, s}}^{-1} \right)(\cdot).
          \end{equation*}
          Using the properties of the Wasserstein distance stated in Lemma~\ref{prop:wasserstein_quantiles}, we obtain
          \begin{align*}
            \mathcal{U}_{\ell}(\widehat{f}_\ell) &= \sum_{s=1}^K  w_s \mathcal{W}_2^2(\mu_{\widehat{f}_\ell \vert \ell, s}, \mu_{\widehat{f}_\ell \vert \ell}^*), \\
            &\leq \sum_{s=1}^K  w_s \int_0^1 \left\vert F_{\mu_{\widehat{f}_\ell \vert \ell, s}}^{-1}(t) - F_{\mu_{\widehat{f}_\ell \vert \ell}^*}^{-1} (t) \right\vert^2 dt, \\
            &= \sum_{s=1}^K w_s \int_0^1 \left\vert F_{\mu_{\widehat{f}_\ell \vert \ell, s}}^{-1}(t) - \sum_{s'=1}^K  w_{s'} F_{\mu_{\widehat{f}_\ell \vert \ell, s'}}^{-1}(t) \right\vert^2 dt, \\
          \end{align*}
          where the second line comes from the Lemma~\ref{prop:wasserstein_quantiles} and the last from the expression of the cumulative distribution function of the barycenter in Prop.~\ref{prop:cumulative_bary}.
          Then, by applying Jensen's inequality and the characterization of the Wassertein distance in terms of quantiles again (Lemma~\ref{prop:wasserstein_quantiles}) and the previous Lemma~\ref{lemma:KS_W2}, we get
          \begin{align}
          \mathcal{U}_{\ell}(\widehat{f}_\ell)& \leq \sum_{s, s' = 1}^K w_s  w_{s'} \mathcal{W}_2^2(\mu_{\widehat{f}_\ell \vert \ell, s}, \mu_{\widehat{f}_\ell \vert \ell, s'}),\notag \\
          &\leq 4M^2 \sum_{s, s' = 1}^K w_s  w_{s'} \sup_{t \in \mathbb{R}} \vert F_{\mu_{\widehat{f}_\ell \vert \ell, s}}(t) - F_{\mu_{\widehat{f}_\ell \vert \ell, s'}}(t) \vert. \label{eq:U_ell_to_KS}
        \end{align}
        
        It remains to control the Kolmogorov gap between the \emph{post-processed} laws.
        Write, for simpler notations within intervall $\ell$,
        \begin{equation*}
        G_{\ell,s}(t)\ := F_{\mu_{\widehat f_\ell|\ell,s}}(t)
        \ =\ \mathbb P\!\left(\sum_{s'} w_{s'} \widehat F^{-1}_{{f^{bb}|\ell,s'}}\!\big(\widehat F_{{f^{bb}|\ell,s}}(Y_{\ell, s})\big)\le t\right),
      \end{equation*}

        where $Law(Y_{\ell, s}):=\mu_{f^{bb}(X,s)\vert \ell, s}$ is the base score in group $s$ and is the empirical barycentric quantile in interval $\ell$.

        The goal is to bound $\sup_t |G_{\ell,s}(t)-G_{\ell,s'}(t)|$ for $s\neq s'$.
        Since $\sum_{s'} \widehat w_{s'} \widehat F^{-1}_{f^{bb}|\ell,s'}$ is the generalized inverse of the empirical barycenter c.d.f.\ $\widehat F_{\ell,\star}$, we have
        \begin{equation*}
          \left( \sum_{s' \in [K]} w_{s'} \widehat{F}_{{f^{bb} \vert \ell, s'}}^{-1} \right)\circ \widehat{F}_{{f^{bb} \vert \ell, s}}(f^{bb}(x, s)) \leq t \iff \widehat{F}_{{f^{bb} \vert \ell, s}}(f^{bb}(x, s)) \leq \left( \sum_{s' \in [K]}  w_{s'} \widehat{F}_{{f^{bb} \vert \ell, s'}}^{-1} \right)^{-1}(t).
        \end{equation*}
        Which gives 
        \begin{equation*}
          G_{\ell,s}(t)\ =\ \mathbb P\!\left(\widehat F_{{f^{bb}|\ell,s}}(Y_{\ell, s})\le \left(\sum_{s' \in [K]}  w_{s'} \widehat{F}_{{f^{bb} \vert \ell, s'}}^{-1}\right)^{-1}(t)\right).
      \end{equation*}
        Define $H_{\ell,s}(u):=\mathbb P\big(\widehat F_{{f^{bb}|\ell,s}}(Y_s)\le u\big)$. 
        Then
        $\sup_t |G_{\ell,s}(t)-G_{\ell,s'}(t)|\le \sup_u |H_{\ell,s}(u)-H_{\ell,s'}(u)|$.
        A triangle inequality with the identity map $u\mapsto u$ yields
        \begin{equation*}
        \sup_u |H_{\ell,s}(u)-H_{\ell,s'}(u)|
        \ \le\
        \sup_u |H_{\ell,s}(u)-u|+\sup_u |u-H_{\ell,s'}(u)|.
      \end{equation*}

        For all fixed $u$, we have that $\widehat F_{f^{bb}|\ell,s}(u)$ is non-decreasing. 
        Therefore
        \begin{align*}
        H_{\ell, s}(u)&= \mathbb{P}(Y_{\ell, s} \leq \widehat F^{-1}_{{f^{bb} \vert \ell, s}}(u)),\\
        &= F_{f^{bb}\vert \ell, s}(\widehat F^{-1}_{f^{bb} \vert \ell, s}(u)).
      \end{align*}
      Thus, 
        \begin{equation*}
        |H_{\ell,s}(u)-u|
        = \left|F_{{f^{bb}|\ell,s}}\!\left(\widehat F^{-1}_{{f^{bb}|\ell,s}}(u)\right)-u\right|.
        \end{equation*}

        Taking the supremum over $u$ gives
        \begin{equation*}
        \sup_u |H_{\ell,s}(u)-u|
        \ \le\
        \underbrace{\sup_t |\widehat F_{{f^{bb}|\ell,s}}(t)-F_{{f^{bb}|\ell,s}}(t)|}_{\varepsilon_{\ell,s}}.
      \end{equation*}
\begin{remark}
  Crucially, because $\widehat F_{f^{bb}|\ell,s}$ is estimated on a fold independent of the samples $Y_{\ell,s}$, the term $H_{\ell,s}(u)$ represents the c.d.f. of the probability integral transform of a fresh sample. 
  This independence is what allows us to bound $|H_{\ell,s}(u) - u|$ directly by the DKW error $\varepsilon_{\ell,s}$ of the marginal distribution. 
  Without this two-fold separation, $Y_{\ell,s}$ and $\widehat F$ would be coupled, and the distribution of their composition would no longer converge uniformly to the identity map at the required rate.
\end{remark}
        By the same reasoning on $s'$ we get
        \begin{equation*}
        \sup_t |G_{\ell,s}(t)-G_{\ell,s'}(t)|
        \ \le\
        \varepsilon_{\ell,s}+\varepsilon_{\ell,s'}.
      \end{equation*}

        Plugging this into~\eqref{eq:U_ell_to_KS} and using $\sum_{s'}\widehat w_{s'}=1$ twice yields
        \begin{equation*}
        \mathcal{U}_{\ell}(\widehat{f}_\ell)
        \ \le\
        8M^2\!\left[
        \sum_{s=1}^K  w_s\,\varepsilon_{\ell,s}
        \right].
      \end{equation*}

        Finally, by Dvoretzky–Kiefer–Wolfowitz with a union bound over all $(\ell,s)$ conditionally on $N_{\ell, s}$,
        \begin{equation}\label{eq:DKW_deviation}
        \varepsilon_{\ell,s}\ \le \sqrt{\frac{1}{N_{\ell,s}}\,
        \log\!\Big(\frac{2KL}{\kappa}\Big)}.
      \end{equation}
        holds with probability at least $1-\kappa$.\\

        Given this bound, it remains to control the random denominators $N_{\ell,s}$. 
        To this end, we start by deriving a multinomial Chernoff bound to ensure that all $N_{\ell,s}$ are uniformly large with high probability.

Let $p_{\ell,s}:=\mathbb P(\U\in\mathcal I_\ell,S=s)$ and $p_{\min}:=\min_{\ell,s}p_{\ell,s}= \frac{\min_s w_s}{L}>0$.
For all $\ell \in [L]$, $s \in [K]$ we have $N_{\ell, s} \sim Bin(n, p_{\ell, s})$, therefore $(N_{\ell,s})_{(\ell,s)\in[L]\times[K]}\sim\mathrm{Multinomial}\!\big(n,\,(p_{\ell,s})\big)$.

Fix $\gamma\in(0,1)$ and define the event
\begin{equation*}
\mathcal E_{mult}(\gamma)\ :=\ \big\{\,N_{\ell,s}\ge (1-\gamma)\,n\,p_{\ell,s}\ \ \text{for all }(\ell,s)\,\big\}.
\end{equation*}
By a multiplicative Chernoff bound and a union bound over the $KL$ cells
\begin{equation*}
\mathbb P\big(\mathcal E_{mult}(\gamma)\big)\ \ge\ 1\ -\ KL\exp\!\Big(-\tfrac{\gamma^2}{2}\,n\,p_{\min}\Big).
\end{equation*}
On $\mathcal E_{mult}(\gamma)$, the DKW deviation from Eq.~\eqref{eq:DKW_deviation} yields, simultaneously for all $(\ell,s)$ and with probability $1- \kappa$, 
\begin{equation}\label{eq:DKW_count}
\varepsilon_{\ell,s}\ \le\ \sqrt{\frac{2}{(1-\gamma)\,n\,p_{\ell,s}}\,
\log\!\Big(\frac{2KL}{\kappa}\Big)}.
\end{equation}
Plugging this into~\eqref{eq:U_ell_to_KS} gives, for each interval $\ell$
\begin{equation*}
\mathcal{U}_{\ell}(\widehat{f}_\ell)
\ \le\
8M^2\sum_{s=1}^K
w_s\sqrt{\frac{2}{(1-\gamma)\,n\,p_{\ell,s}}\,
\log\!\Big(\frac{2KL}{\kappa}\Big)}.
\end{equation*}
Finally, we intersect $\mathcal E_{mult}(\gamma)$ with the DKW event (probability $\ge 1-\kappa$) and use a union bound.
Let \begin{equation*}
\;\mathcal E_{\mathrm{DKW}}:=\big\{\forall(\ell,s):~\varepsilon_{\ell,s}\le
\sqrt{\tfrac{1}{N_{\ell,s}}\log(\tfrac{2KL}{\kappa})}\big\}\;
\text{ so that }\;\mathbb P(\mathcal E_{\mathrm{DKW}})\ge 1-\kappa,
\end{equation*}
and 
\begin{equation*}
  \;\mathcal E_{mult}(\gamma):=\{\,\forall(\ell,s):~N_{\ell,s}\ge (1-\gamma)n p_{\ell,s}\,\},
\end{equation*}
with 
\begin{equation*}
\mathbb P(\mathcal E_{mult}(\gamma))\ge 1-\kappa_{mult}, \quad 
\kappa_{mult}:=KL\exp\!\big(-\tfrac{\gamma^2}{2}np_{\min}\big).
\end{equation*}

Then, by a union bound on complements
\begin{equation*}
  \mathbb P\big(\mathcal E_{\mathrm{DKW}}\cap \mathcal E_{mult}(\gamma)\big)
=1-\mathbb P\big(\mathcal E_{\mathrm{DKW}}^{\mathrm c}\cup \mathcal E_{mult}(\gamma)^{\mathrm c}\big)
\ge 1-\big[\mathbb P(\mathcal E_{\mathrm{DKW}}^{\mathrm c})+\mathbb P(\mathcal E_{mult}(\gamma)^{\mathrm c})\big]
\ge 1-(\kappa+\kappa_{mult}).
\end{equation*}

On the intersection event, using $N_{\ell,s}\ge (1-\gamma)np_{\ell,s}$ inside the DKW deviation
\begin{equation*}
\mathcal{U}_{\ell}(\widehat{f}_\ell)
\ \le\
8M^2\sum_{s=1}^K
w_s\sqrt{\frac{2}{(1-\gamma)\,n\,p_{\ell,s}}\,
\log\!\Big(\frac{2KL}{\kappa}\Big)}
\end{equation*}

Hence the bound holds with probability at least
\begin{equation*}
1-\eta,\qquad \text{where }\ \eta:=\kappa+\kappa_{mult}
=\kappa+KL\exp\!\Big(-\tfrac{\gamma^2}{2}\,n\,p_{\min}\Big).
\end{equation*}
This completes the proof.

\subsubsection{Proof of Lemma \ref{lem:disc_error}} \label{proof:lem_disc_error}

  Fix an interval $\ell$ and group $s$, and let $v\in\mathcal I_\ell$.
  Let $Y_{v,s}\sim\mu_{f^{bb}\mid v,s}$ and note that $U:=F_{{f^{bb}|v,s}}(Y_{v,s})\sim\mathrm{Unif}(0,1)$.
  Recall that 
  \begin{equation*}
  T_{v,s}(y)=F^{-1}_{v, \star}(F_{v,s}(y)),\qquad
  T_{\ell,s}(y)=F^{-1}_{\ell, \star}(F_{\ell,s}(y)),
  \end{equation*}
  with $F^{-1}_{v, \star}(u)=\sum_s w_s F^{-1}_{v, s}(u)$ and $F^{-1}_{\ell ,\star}(u)=\sum_s w_s F^{-1}_{\ell,s}(u)$.
  By $(a+b)^2\le 2a^2+2b^2$ and monotonicity,
  \begin{align*}
  \mathbb E\!\left[(T_{\ell,s}(Y_{v,s})-T_{v,s}(Y_{v,s}))^2\right]
  &= \mathbb E\!\left[\big(F^{-1}_{\ell, \star}(F_{\ell,s}(Y_{v,s}))-F^{-1}_{v, \star}(F_{v,s}(Y_{v,s}))\big)^2\right]\\
  &\le 2\,\underbrace{\mathbb E\!\left[\big(F^{-1}_{\ell, \star}(F_{v,s}(Y_{v,s}))-F^{-1}_{v, \star}(F_{v,s}(Y_{v,s}))\big)^2\right]}_{(\mathrm{I})}\\
  &\hspace{2em}+\;2\,\underbrace{\mathbb E\!\left[\big(F^{-1}_{\ell, \star}(F_{\ell,s}(Y_{v,s}))-F^{-1}_{\ell, \star}(F_{v,s}(Y_{v,s}))\big)^2\right]}_{(\mathrm{II})}.
  \end{align*}
  The first term (I) corresponds to the barycenter shift, while the second term (II) corresponds to the c.d.f. shift under a fixed quantile.

  \paragraph{(I) Barycenter shift.}
  Since $U=F_{{f^{bb}|v,s}}(Y_{v,s})\sim\mathrm{Unif}(0,1)$,
  \begin{align*}
  (\mathrm{I})&=\int_0^1\!\big(F^{-1}_{\ell, \star}(u)-F^{-1}_{v, \star}(u)\big)^2\,du, \\
  &=\int_0^1\!\Big(\sum_r w_r\big(F^{-1}_{\ell, r}(u)-F^{-1}_{v, r}(u)\big)\Big)^2 du, \\
  &\le \sum_r w_r \!\int_0^1\!(F^{-1}_{\ell, r}-F^{-1}_{v, r})^2.
  \end{align*}
  The last integral is $\mathcal W_2^2(\mu_{f^{bb}\mid \ell,r},\mu_{f^{bb}\mid v,r})$, which, by Lemma~\ref{lemma:KS_W2}, satisfies
  \begin{equation*}
  \mathcal W_2^2(\mu_{f^{bb}\mid \ell,r},\mu_{f^{bb}\mid v,r})
  \ \le\ 4M^2\,\|F_{\ell,r}-F_{v,r}\|_\infty.
  \end{equation*}
  By Property~\ref{ass:lipschitz_cdf}, for a uniform partition of $[0,1]$ into $L$ intervals,
  \begin{align*}
  \|F_{\ell,r}-F_{v,r}\|_\infty
  &=\Big\|\frac{1}{p_\ell}\!\int_{\mathcal I_\ell}(F_{u,r}-F_{v,r})\,d\nu(u)\Big\|_\infty, \\
  &\le \frac{1}{|\mathcal I_\ell|}\!\int_{\mathcal I_\ell}\!\Lcdf\,|u-v|\,d\nu(u), \\
  &\le \frac{\Lcdf}{2L}.
  \end{align*}
  Therefore,
  \begin{equation*}
  (\mathrm{I})\ \le\ 4M^2\sum_r w_r\,\frac{\Lcdf}{2L}
  \ =\ \frac{4M^2\,\Lcdf}{2L}.
  \end{equation*}

\paragraph{(II) c.d.f.\ shift under a fixed quantile.}
Let $U:=F_{v,s}(Y_{v,s})\sim\mathrm{Unif}(0,1)$, $h:=F^{-1}_{\ell,\star}$ and $G:=F_{\ell,s}\circ F_{v,s}^{-1}$. Note that $h$ and $G$ are non-decreasing.
  We can rewrite (II) as
  \begin{equation*}
  (\mathrm{II})
  =\mathbb E\!\left[\big(h(G(U))-h(U)\big)^2\right].
  \end{equation*}
  Since $h$ and $G$ are non-decreasing, the random variables $h(G(U))$ and $h(U)$ are comonotonic (coupled via the identity on $U$). 
  This coupling is optimal for the 1-Wasserstein distance, hence the expectation equals the squared 2-Wasserstein distance
  \begin{equation*}
  (\mathrm{II}) = \mathcal W_2^2\!\big(h(G(U)),\,h(U)\big).
  \end{equation*}
  Using the bound $\mathcal W_2^2(X,Y)\le 4M^2\|F_X-F_Y\|_\infty$ (valid on $[-M,M]$) and the contractivity of the Kolmogorov distance under the non-decreasing map $h$, we get
  \begin{equation*}
  (\mathrm{II})\le 4M^2\,\|F_{h(G(U))}-F_{h(U)}\|_\infty
  \;\le\;4M^2\,\|F_{G(U)}-F_U\|_\infty.
  \end{equation*}
  For $t\in[0,1]$,
  \begin{equation*}
  F_{G(U)}(t)=\mathbb P\!\big(G(U)\le t\big)
  =F_{v,s}\!\big(F^{-1}_{\ell,s}(t)\big),
  \end{equation*}
  hence $\|F_{G(U)}-F_U\|_\infty=\|F_{v,s}-F_{\ell,s}\|_\infty$. Therefore, using the same Lipschitz bound as in (I)
  \begin{equation*}
  (\mathrm{II})\ \le\ 4M^2\,\|F_{\ell,s}-F_{v,s}\|_\infty
  \ \le\ 4M^2\cdot\frac{\Lcdf}{2L}
  \ =\ \frac{2M^2\,\Lcdf}{L}.
  \end{equation*}
  
  \paragraph{Conclusion.}
  Combining (I) and (II),
  \begin{equation*}
  \mathbb E\!\left[(T_{\ell,s}(Y_{v,s})-T_{v,s}(Y_{v,s}))^2\right]
  \ \le\ \frac{2M^2\,\Lcdf}{L}\;+\;2M^2\frac{\Lcdf}{L}.
  \end{equation*}
  Since the weights $w_s$ and $p_\ell$ sum to 1 we have the following bound on the discretization error
  \begin{align*}
    \mathcal{R}^{disc}_v &\le \frac{4M^2\,\Lcdf}{L}\;+\;\frac{4M^2\,\Lcdf}{L}, \\
    &\leq\frac{8M^2\Lcdf}{L}.
  \end{align*}

  \subsubsection{Proof of Lemma \ref{lem:estimation_risk}}

    Fix $\ell$ and $s$, we have
      $T_{\ell,s}(y):=F^{-1}_{\ell, \star}(F_{{f^{bb}|\ell,s}}(y))$ (population CF map) and
      $\widehat T_{\ell,s}(y):=\widehat F^{-1}_{\ell, \star}(\widehat F_{{f^{bb}|\ell,s}}(y))$ (empirical map),
      with $F^{-1}_{\ell, \star}(u)=\sum_{s}w_s F^{-1}_{\ell, s}(u)$. For $Y_{\ell, s}$ such that $Law(Y_{\ell, s})=\mu_{f^{bb}|\ell,s}$,
      conditionally on $(\U\in\mathcal I_\ell,S=s)$, $U:=F_{{f^{bb}|\ell,s}}(Y_{\ell, s})$ is uniform on $(0,1)$, hence

      Set $\varepsilon_{\ell,s}:=\|\widehat F_{f^{bb}|\ell,s}-F_{f^{bb}|\ell,s}\|_\infty$ and, conditional on the calibration sample,
      let $U:=F_{f^{bb}|\ell,s}(Y_{\ell,s})\sim\mathrm{Unif}(0,1)$ and
      $\widehat U:=\widehat F_{f^{bb}|\ell,s}(Y_{\ell,s})$, so that $|\,\widehat U-U\,|\le \varepsilon_{\ell,s}$ a.s.
      Then
      \begin{equation*}
      \widehat T_{\ell,s}(Y_{\ell,s})-T_{\ell,s}(Y_{\ell,s})
      =\widehat F^{-1}_{\ell,\star}(\widehat U)-F^{-1}_{\ell,\star}(U).
    \end{equation*}

      Add and subtract $\widehat F^{-1}_{\ell,\star}(U)$ and use $(a+b)^2\le 2a^2+2b^2$:
      \begin{align*}
      \big(\widehat F^{-1}_{\ell,\star}(\widehat U)-F^{-1}_{\ell,\star}(U)\big)^2
      &=\Big(\underbrace{\widehat F^{-1}_{\ell,\star}(\widehat U)-\widehat F^{-1}_{\ell,\star}(U)}_{(\mathrm{A})}
      +\underbrace{\widehat F^{-1}_{\ell,\star}(U)-F^{-1}_{\ell,\star}(U)}_{(\mathrm{B})}\Big)^2\\
      &\le 2\,(\mathrm{A})^2+2\,(\mathrm{B})^2.
      \end{align*}
      Taking expectations and using $U\sim\mathrm{Unif}(0,1)$ for the second term gives
      \begin{align*}
      \mathbb E\!\left[(\widehat T_{\ell,s}(Y_{\ell,s})-T_{\ell,s}(Y_{\ell,s}))^2\right]
      &\le 2\,\mathbb E\!\left[\big(\widehat F^{-1}_{\ell,\star}(\widehat U)-\widehat F^{-1}_{\ell,\star}(U)\big)^2\right]
      +2\,\int_0^1\!\big(\widehat F^{-1}_{\ell,\star}(u)-F^{-1}_{\ell,\star}(u)\big)^2\,du. \tag{$\ast$}
      \end{align*}
      For the first expectation, since $\widehat F^{-1}_{\ell,\star}$ is nondecreasing with range in $[-M,M]$ and
      $|\,\widehat U-U\,|\le \varepsilon_{\ell,s}$ a.s., we have the coarse bound
      \begin{equation*}
      \big(\widehat F^{-1}_{\ell,\star}(\widehat U)-\widehat F^{-1}_{\ell,\star}(U)\big)^2
      \le (2M)^2\,\mathbf 1\{\,\widehat U\neq U\,\}
      \le 4M^2\,\varepsilon_{\ell,s}^2,
    \end{equation*}

      whence
      \begin{equation*}
      2\,\mathbb E\!\left[\big(\widehat F^{-1}_{\ell,\star}(\widehat U)-\widehat F^{-1}_{\ell,\star}(U)\big)^2\right]
      \le 8M^2\,\varepsilon_{\ell,s}^2.
    \end{equation*}

      Therefore, from ($\ast$),
      \begin{align*}
      \mathbb E\left[(\widehat T_{\ell,s}(Y_{\ell, s})-T_{\ell,s}(Y_{\ell, s}))^2\right]
      &\leq 2\!\int_0^1\!\left(\widehat F^{-1}_{\ell, \star}(u)-F^{-1}_{\ell, \star}(u)\right)^2du
      +8M^2\sup_t \vert \widehat F_{{f^{bb}|\ell,s}}(t)-F_{{f^{bb}|\ell,s}}(t)\vert^2.
    \end{align*}
      For the first term, Jensen on quantiles yields
      \begin{align*}
      \int_0^1\!\big(\widehat F^{-1}_{\ell, \star}(u)-F^{-1}_{\ell, \star}(u)\big)^2du
      &=\int_0^1\!\Big(\sum_r w_r(\widehat F^{-1}_{\ell, r}(u)-F^{-1}_{\ell, r}(u))\Big)^2du\\
      &\le \sum_r w_r \int_0^1(\widehat F^{-1}_{\ell, r}(u)-F^{-1}_{\ell, r}(u))^2du\\
      &=\sum_r w_r\,\mathcal W_2^2(\widehat\mu_{\ell,r},\mu_{\ell,r}).
    \end{align*}
    
      By Lemma~\ref{lemma:KS_W2}, each $\mathcal W_2^2(\widehat\mu_{\ell,r},\mu_{\ell,r})\le 4M^2 \Vert \widehat F_{f^{bb}|\ell,r}-F_{f^{bb}|\ell,r}\Vert_\infty$. Altogether,
      \begin{equation*}
      \mathbb E\big[(\widehat T_{\ell,s}(Y_s)-T_{\ell,s}(Y_s))^2\big]
      \;\le\;
      8M^2\sum_{r} w_r\,\varepsilon_{\ell,r}+8M^2\,\varepsilon_{\ell,s}^2,\qquad
      \varepsilon_{\ell,r}:=\|\widehat F_{f^{bb}|\ell,r}-F_{f^{bb}|\ell,r}\|_\infty.
    \end{equation*}
    Since $\varepsilon_{\ell,r}\leq 1$ taking the $w_s$–average over $s$ and using $\sum_s w_s\varepsilon_{\ell,s}^2\le \max_s\varepsilon_{\ell,s}\sum_s w_s\varepsilon_{\ell,s}$,
      \begin{equation*}
      \sum_s w_s\,\mathbb E[(\widehat T_{\ell,s}-T_{\ell,s})^2]
      \;\le\; 16M^2 \sum_{s} w_s\,\varepsilon_{\ell,s}.
    \end{equation*}
    
      Now apply the uniform DKW as in Eq.~\eqref{eq:DKW_count} in Lemma~\ref{lem:variance_unfairness}.
      On $\mathcal E_{\mathrm{DKW}}\cap\mathcal E_{\mathrm{mult}}(\gamma)$ (probability $\ge 1-\eta_L$),
      \begin{equation*}
      \varepsilon_{\ell,s}\ \le\ \sqrt{\frac{2L\,\log\!\left( 2KLn\right)}{n}}\,
  \sum_{s=1}^K \sqrt{w_s}.
    \end{equation*}
    
Since $\sum_s w_s = 1$, we have
      \begin{align*}
      \sum_s w_s\,\varepsilon_{\ell,s}
      &\le \sqrt{\frac{2L\,\log\!\left( 2KLn\right)}{n}}\,
  \sum_{s=1}^K \sqrt{w_s}.
    \end{align*}
    Finally, by summing over $\ell$ and $s$ we get that the excess risk of taking the estimated post-processed predictor $\widehat f_{L}$ over the non estimated discretized $g_{0, L}$, $\mathcal{R}^{est}_v$ is bounded by
    \begin{equation*}
      \mathcal{R}^{est}_v \leq 16M^2\sqrt{\frac{2L\,\log\!\left( 2KLn\right)}{n}}\,
  \sum_{s=1}^K \sqrt{w_s}.
    \end{equation*}

\subsubsection{Proof of lemma \ref{lem:assouad}}\label{proof:lem_assouad}

       Lemma~\ref{lem:assouad} is derived from the classical statement
        \citep[Lemma~2.12]{tsybakov2009introduction}.  In Tsybakov's notation,
        let $\Omega=\{0,1\}^J$ and $\{\mathbb P_\eta:\eta\in\Omega\}$ be a
        family of $2^J$ probability measures.  Denote by
        $\rho(\eta,\eta')$ the Hamming distance between $\eta$ and $\eta'$.
        Lemma~2.12 states that
        \begin{equation*}
          \inf_{\widehat\eta} \max_{\eta\in\Omega}
          \mathbb E_\eta \rho(\widehat\eta,\eta)
          \;\ge\;
          \frac{J}{2}\,
          \min_{\substack{\eta,\eta':\\ \rho(\eta,\eta')=1}}
          \inf_\psi \big(\mathbb P_\eta(\psi\neq 0)+\mathbb P_{\eta'}(\psi\neq 1)\big),
        \end{equation*}
        where the infimum is over all tests $\psi$ taking values in $\{0,1\}$.
        
        In our setting, we assume a Kullback--Leibler bound between
        neighbouring models,
        \begin{equation*}
          \max_{\rho(\eta,\eta')=1}
          \mathrm{KL}(\mathbb P_\eta\ \vert\mathbb P_{\eta'})\;\le\;\kappa.
        \end{equation*}
      
        Using Pinsker's inequality,
        $\mathrm{TV}(\mathbb P_\eta,\mathbb P_{\eta'})^2
         \le \tfrac12\mathrm{KL}(\mathbb P_\eta\ \vert\mathbb P_{\eta'})$,
        we obtain for neighbours
        \begin{equation*}
          \inf_\psi \big(\mathbb P_\eta(\psi\neq 0)+\mathbb P_{\eta'}(\psi\neq 1)\big)
          \;=\;1-\mathrm{TV}(\mathbb P_\eta,\mathbb P_{\eta'})
          \;\ge\;1-\sqrt{\tfrac{\kappa}{2}}.
        \end{equation*}
      
        Plugging this into Tsybakov's lemma yields the bound
        \begin{equation}\label{eq:assouad_hamming}
          \inf_{\widehat\eta} \max_{\eta\in\Omega}
          \mathbb E_\eta \rho(\widehat\eta,\eta)
          \;\ge\;
          \frac{J}{2}\Big(1-\sqrt{\tfrac{\kappa}{2}}\Big).
        \end{equation}
        
        To obtain our formulation, we introduce $\phi(\eta)$ taking values in a metric space
        $(\mathcal T,d)$ and assume a neighboor separation condition:
        whenever $\rho(\eta,\eta')=1$,
        \begin{equation*}
          d\big(\phi(\eta),\phi(\eta')\big)\;\ge\;\Delta.
        \end{equation*}
      
        Intuitively, changing one coordinate of $\eta$ moves the true parameter
        by at least $\Delta$ in the metric $d$.
        From any estimator $\widehat\phi$ we can construct an estimator
        $\widehat\eta$ of the discrete index by testing, coordinate-wise,
        which of two neighbooring parameters
        $\phi(\eta)$ or $\phi(\eta')$ is closer to $\widehat\phi$.  By the
        triangle inequality, an error on coordinate $j$ implies
        $d(\widehat\phi,\phi(\eta))\ge \Delta/2$, hence
        \begin{equation*}
          d(\widehat\phi,\phi(\eta))
          \;\ge\; \frac{\Delta}{2}\,\rho(\widehat\eta,\eta)
          \quad\text{and thus}\quad
          \mathbb E_\eta d(\widehat\phi,\phi(\eta))
          \;\ge\; \frac{\Delta}{2}\,\mathbb E_\eta \rho(\widehat\eta,\eta).
        \end{equation*}
      
        Taking the supremum over $\eta$ and then the infimum over estimators,
        and combining with \eqref{eq:assouad_hamming}, we obtain
        \begin{equation*}
          \inf_{\widehat\phi} \sup_{\eta\in\{-1,+1\}^J}
          \mathbb E_\eta d(\widehat\phi,\phi(\eta))
          \;\ge\;
          \frac{J\,\Delta}{2}\Big(1-\sqrt{\tfrac{\kappa}{2}}\Big),
        \end{equation*}
      
        which is precisely the version of Assouad's lemma stated and used in
        our proof.
        In our application, the metric space is $(L^1([0,1]),\|\cdot\|_{L^1})$
        and the neighbor separation
        $d(\phi(\eta),\phi(\eta'))\ge\Delta$ follows from the bump construction
        on disjoint blocks of the covariate space.

\subsubsection{Proof of lemma \ref{lem:interval-convexity}} \label{proof:lem_interval_convexity}

  Let, for each $s\in[K]$, $\mu_{f\mid \ell,s}$ (resp.\ $\mu_{g\mid \ell,s}$, $\mu_{h\mid \ell,s}$) denote the law of $f(\mathbf X,s)$ (resp.\ $g(\mathbf X,s)$, $h(\mathbf X,s)$) conditional on $\{\,\U\in\mathcal I_\ell,\ S=s\,\}$.

Let $\mu_{f,\ell}^\star\in\argmin_{\mu}\sum_{s=1}^K w_s\,\mathcal W_2^2(\mu_{f\mid \ell,s},\mu)$ and
$\mu_{g,\ell}^\star\in\argmin_{\mu}\sum_{s=1}^K w_s\,\mathcal W_2^2(\mu_{g\mid \ell,s},\mu)$ be (any) interval barycenters for $f$ and $g$.
Define the linear map $T:\mathbb R^2\to\mathbb R$ by
\begin{equation*}
T(x,y)\ :=\ \lambda x + (1-\lambda)y,
\end{equation*}
and set the \emph{candidate} barycenter for $h$ as the $T$-image of the product of the two barycenters,
\begin{equation*}
\mu_{\ell}^{(\lambda)\star}\ :=\ T_{\#}\,(\mu_{f,\ell}^\star\otimes \mu_{g,\ell}^\star).
\end{equation*}

Fix $s\in[K]$. Let $\rho_{\ell,s}$ be the \emph{joint} conditional law of the pair
\begin{equation*}
(F_s,G_s)\ :=\ \big(f(\mathbf X,s),\,g(\mathbf X,s)\big)
\quad\text{given}\quad \{\U\in\mathcal I_\ell,\ S=s\},
\end{equation*}
so that $\mu_{f\mid \ell,s}$ and $\mu_{g\mid \ell,s}$ are the first and second marginals of $\rho_{\ell,s}$. Then, we have
\begin{equation*}
\mu_{h\mid \ell,s}\ =\ T_{\#}\,\rho_{\ell,s}.
\end{equation*}
We aim at finding a coupling between $\mu_{h\mid \ell,s}$ and $\mu_{\ell}^{(\lambda)\star}$ to bound $\mathcal W_2(\mu_{h\mid \ell,s},\mu_{\ell}^{(\lambda)\star})$.
The goal is to minimize 
\begin{equation*}
  \mathcal U_\ell(h) \leq \sum_s w_s\mathcal{W}_2^2(\mu_{h\vert \ell, s}, \mu_{\ell}^{(\lambda)\star}).
\end{equation*}
To find the susmentioned coupling, we start by taking the optimal couplings between $\mu_{f\mid \ell,s}$ and $\mu_{f,\ell}^\star$, and between $\mu_{g\mid \ell,s}$ and $\mu_{g,\ell}^\star$.
Then we use the product of these two couplings to build a coupling between $\nu_{\ell,s}$ and $\mu_{f,\ell}^\star\otimes \mu_{g,\ell}^\star$.
Finally, we push this coupling forward by $T\times T$ to get a coupling between $\mu_{h\mid \ell,s}$ and $\mu_{\ell}^{(\lambda)\star}$.

\paragraph{Construction of the coupling.}
Let $\pi_f^s\in\Pi(\mu_{f\mid \ell,s},\mu_{f,\ell}^\star)$ (resp.$\pi_g^s\in\Pi(\mu_{g\mid \ell,s},\mu_{g,\ell}^\star)$) be an \emph{optimal couplings} for the optimal transport problem obtained with the distance $\mathcal{W}_2^2(\mu_{f\mid \ell,s},\mu_{f,\ell}^\star)$ (resp. $\mathcal{W}_2^2(\mu_{g\mid \ell,s},\mu_{g,\ell}^\star)$)
  (not necessarily induced by maps). 
Disintegrate them as
\begin{equation*}
\pi_f^s(dx,dx^\star)=\mu_{f\mid \ell,s}(dx)\,\kappa_f^s(x,dx^\star),
\qquad
\pi_g^s(dy,dy^\star)=\mu_{g\mid \ell,s}(dy)\,\kappa_g^s(y,dy^\star),
\end{equation*}

for stochastic kernels $\kappa_f^s$ and $\kappa_g^s$.

$\kappa_f^s$ and $\kappa_g^s$ are defined as follows. 
For $(X, X^\ast) \sim \pi_g^S$, let $\kappa_f^s(x,\cdot)$ be the conditional law of $X$
under $\pi_f^s$, and similarly for $\kappa_g^s$.
Since $\mathbb{R}$ is a standard Borel space, the optimal couplings
$\pi_f^s$ and $\pi_g^s$ admit \emph{regular conditional probabilities}
with respect to their first marginals. 
The kernels $\kappa_f^s$ and $\kappa_g^s$ are precisely these conditionals. In particular, when
$\mu_{f\mid \ell,s}$ (respectively $\mu_{g\mid \ell,s}$) is non-atomic
in one dimension, the optimal plan is induced by the monotone map
$T_f^s=F^{-1}_{\mu_{f,\ell}^\star}\circ F_{\mu_{f\mid \ell,s}}$
(respectively $T_g^s$) and the kernel degenerates to a Dirac mass:
$\kappa_f^s(x,\cdot)=\delta_{T_f^s(x)}$ and
$\kappa_g^s(y,\cdot)=\delta_{T_g^s(y)}$. With atoms, the kernels may
randomize (split mass), which is exactly what is needed to keep the
marginals correct.

Define a coupling $\Lambda_s$ between $\rho_{\ell,s}$ and $\mu_{f,\ell}^\star\otimes\mu_{g,\ell}^\star$ on $\mathbb R^2\times\mathbb R^2$ by
\begin{equation*}
\Lambda_s(dxdydx^\star dy^\star)
:=\rho_{\ell,s}(dx,dy)\,\kappa_f^s(x,dx^\star)\,\kappa_g^s(y,dy^\star).
\end{equation*}

By construction, the first marginal of $\Lambda_s$ is $\rho_{\ell,s}$ and the second marginal is
\begin{equation*}
\left(\int \kappa_f^s(x,dx^\star)\,\mu_{f\mid \ell,s}(dx)\right)
\otimes
\left(\int \kappa_g^s(y,dy^\star)\,\mu_{g\mid \ell,s}(dy)\right)
=\mu_{f,\ell}^\star\otimes\mu_{g,\ell}^\star.
\end{equation*}

Finally, push forward by $T\times T$ to obtain
\begin{equation*}
\zeta_s:=(T\times T)_{\#}\Lambda_s,
\end{equation*}

which is the wanted coupling.
To see this, note that the first marginal is
\begin{equation*}
T_{\#}\,\nu_{\ell,s}\ =\ \mu_{h\mid \ell,s},
\end{equation*}
and the second marginal is
\begin{equation*}
T_{\#}\,(\mu_{f,\ell}^\star\otimes \mu_{g,\ell}^\star)\ =\ \mu_{\ell}^{(\lambda)\star}.
\end{equation*}

Let \(((X,Y),(X^\star,Y^\star))\sim\Lambda_s\) and define
\begin{equation*}
Z:=T(X,Y),\qquad Z^\star:=T(X^\star,Y^\star),
\end{equation*}

so that \((Z,Z^\star)\sim\zeta_s\). Since \(\zeta_s\) is a coupling between
\(\mu_{h\mid \ell,s}\) and \(\mu_{\ell}^{(\lambda)\star}\), we have
\begin{equation*}
\mathcal W_2^2\!\big(\mu_{h\mid \ell,s},\,\mu_{\ell}^{(\lambda)\star}\big)
\;\le\;\mathbb E_{\zeta_s}\big[\,|Z-Z^\star|^2\,\big]
\;=\;\mathbb E_{\Lambda_s}\big[\,|T(X,Y)-T(X^\star,Y^\star)|^2\,\big].
\end{equation*}

\paragraph{Bounding the transportation cost.}
Using the coupling described above as an admissible coupling for $\mathcal W_2$, we get
\begin{align*}
  \mathcal W_2^2\!\big(\mu_{h\mid \ell,s},\,\mu_{\ell}^{(\lambda)\star}\big)
&\le \mathbb E_{\Lambda_s}\!\Big[\big|T(X,Y)-T(X^\star,Y^\star)\big|^2\Big],\\
&= \mathbb E_{\Lambda_s}\!\Big[\big|\lambda(X-X^\star)+(1-\lambda)(Y-Y^\star)\big|^2\Big],\\
&\le \lambda\,\mathbb E_{\Lambda_s}\!\big[(X-X^\star)^2\big] + (1-\lambda)\,\mathbb E_{\Lambda_s}\!\big[(Y-Y^\star)^2\big],\\
\end{align*}
where we used the convexity inequality $|\lambda a+(1-\lambda)b|^2\le \lambda|a|^2+(1-\lambda)|b|^2$ in the third line, and the optimality of the couplings $\pi_f^s$ and $\pi_g^s$ in the last line.

By construction, the \((X,X^\star)\)-marginal of \(\Lambda_s\) is the optimal plan
\(\pi_f^s\), hence
\begin{equation*}
  \mathbb E_{\Lambda_s}\big[(X-X^\star)^2\big]
=\int |x-x^\star|^2\,d\pi_f^s(x,x^\star)
=\mathcal W_2^2(\mu_{f\mid \ell,s},\mu_{f,\ell}^\star),
\end{equation*}

and similarly for \(g\).
Thus, we have 
\begin{equation*}
\mathcal W_2^2\!\big(\mu_{h\mid \ell,s},\,\mu_{\ell}^{(\lambda)\star}\big)
\leq \lambda\,\mathcal W_2^2\!\big(\mu_{f\mid \ell,s},\,\mu_{f,\ell}^\star\big)+ (1-\lambda)\,\mathcal W_2^2\!\big(\mu_{g\mid \ell,s},\,\mu_{g,\ell}^\star\big).
\end{equation*}

Finally, weight by $w_s$ and sum over $s$
\begin{equation*}
\sum_{s=1}^K w_s\,\mathcal W_2^2\!\big(\mu_{h\mid \ell,s},\,\mu_{\ell}^{(\lambda)\star}\big)
\ \le\
\lambda\,\underbrace{\sum_{s=1}^K w_s\,\mathcal W_2^2\!\big(\mu_{f\mid \ell,s},\,\mu_{f,\ell}^\star\big)}_{=\ \mathcal U_\ell(f)}
\;+\;
(1-\lambda)\,\underbrace{\sum_{s=1}^K w_s\,\mathcal W_2^2\!\big(\mu_{g\mid \ell,s},\,\mu_{g,\ell}^\star\big)}_{=\ \mathcal U_\ell(g)}.
\end{equation*}
Since $\mu_{\ell}^{(\lambda)\star}$ is just a \emph{candidate} barycenter for $h$, taking the minimum over $\mu$ on the left-hand side yields
\begin{equation*}
\mathcal U_\ell(h)\ =\ \min_{\mu}\sum_{s=1}^K w_s\,\mathcal W_2^2\!\big(\mu_{h\mid \ell,s},\,\mu\big)
\ \le\ \sum_{s=1}^K w_s\,\mathcal W_2^2\!\big(\mu_{h\mid \ell,s},\,\mu_{\ell}^{(\lambda)\star}\big)
\ \le\ \lambda\,\mathcal U_\ell(f)+(1-\lambda)\,\mathcal U_\ell(g),
\end{equation*}
as claimed.

\subsubsection{Proof of Lemma \ref{lem:unfairness_bb_bound}} \label{proof:lem_unfairness_bb_bound}

Fix $\ell$. 
We have 
\begin{equation*}
  \widehat F^{-1}_{\ell,\star}(u):=\sum_{r=1}^K w_r\,\widehat F^{-1}_{\ell,r}(u),
\end{equation*}
Let $\mu_{\ell,\star}$ be the population barycenter of $\{\mu_{\ell,s}\}_s$ with weights $\mathbf{w}$, and $\widehat\mu_{\ell,\star}$ the empirical barycenter of $\{\widehat\mu_{\ell,s}\}_s$ with weights $\mathbf{w}$. By definition of $\mathcal U_\ell$ and suboptimality of $\widehat\mu_{\ell,\star}$ as a candidate barycenter,
\begin{equation*}
\mathcal U_\ell(f^{bb})
=\min_\nu \sum_s w_s\,\mathcal W_2^2(\mu_{\ell,s},\nu)
\ \le\ \sum_s w_s\,\mathcal W_2^2(\mu_{\ell,s},\widehat\mu_{\ell,\star}).
\end{equation*}
By the triangle inequality in $\mathcal W_2$ and $(a+b)^2\le 2a^2+2b^2$, so 
\begin{equation*}
\mathcal W_2^2(\mu_{\ell,s},\widehat\mu_{\ell,\star})
\ \le\ 2\,\mathcal W_2^2(\mu_{\ell,s},\widehat\mu_{\ell,s})
\;+\;2\,\mathcal W_2^2(\widehat\mu_{\ell,s},\widehat\mu_{\ell,\star}),
\end{equation*}

hence
\begin{equation}\label{eq:A-sum}
\mathcal U_\ell(f^{bb})
\ \le\
2\sum_s w_s\,\mathcal W_2^2(\mu_{\ell,s},\widehat\mu_{\ell,s})
\;+\;
2\sum_s w_s\,\mathcal W_2^2(\widehat\mu_{\ell,s},\widehat\mu_{\ell,\star}).
\end{equation}

Apply this to the RHS of \eqref{eq:A-sum} to get
\begin{equation*}
\mathcal U_\ell(f^{bb})
\ \le\
2\sum_s  w_s\,\mathcal W_2^2(\mu_{\ell,s},\widehat\mu_{\ell,s})
\;+\;
2\,\widehat{\mathcal U}_\ell(f^{bb}).
\end{equation*}

In one dimension, Lemma~\ref{lemma:KS_W2} gives
\begin{equation*}
\mathcal W_2^2(\mu_{\ell,s},\widehat\mu_{\ell,s})
\ \le\ 4M^2\,\varepsilon_{\ell,s},
\qquad
\varepsilon_{\ell,s}:=\big\|\widehat F_{f^{bb}\mid \ell,s}-F_{f^{bb}\mid \ell,s}\big\|_\infty.
\end{equation*}

Thus
\begin{equation*}
  \mathcal U_\ell(f^{bb})
\ \le\
8M^2\sum_s  w_s\,\varepsilon_{\ell,s}
\;+\;
2\,\widehat{\mathcal U}_\ell(f^{bb}).
\end{equation*}

Multiply by $p_\ell$ and sum over $\ell$
\begin{equation*}
\mathcal U_L(f^{bb})
\ \le\
8M^2\sum_{\ell=1}^L p_\ell\sum_s w_s\,\varepsilon_{\ell,s}
\;+\;
2\,\widehat{\mathcal U}_L(f^{bb}).
\end{equation*}

Since $\varepsilon_{\ell, s}$ is the same as in the proof of Theorem~\ref{thm:l_star} we have with probability at least $1-\eta_L$ uniformly over all $\ell$ and $s$, for $n\geq \sqrt{\frac{\Lcdf}{M}}$, $\log(2KnL) \leq 2 \log(2Kn)$,
\begin{equation*}
  \varepsilon_{\ell,s}\ \le\ \sqrt{\frac{2}{\,n\,p_{\ell,s}}\,
  \log\!\left(2KLn\right)},
\end{equation*}
Since we have $\U \perp S$ and the discretization is uniform, we have $p_\ell = 1/L$ and $p_{\ell,s}=w_s/L$. The previous inequality yields
\begin{equation*}
  \sum_{\ell=1}^L p_\ell\sum_s w_s\,\varepsilon_{\ell,s}
  \leq \sqrt{\frac{2L}{\,n}\,
  \log\!\Big(2KLn\Big)} \sqrt{K}.
\end{equation*}

Intersecting the two events (DKW and Chernoff on counts) gives a bound on $\mathcal U_L(f^{bb})$ holding with probability at least $1-\eta_L$.
To get the bound on $\mathcal{U}(f^{bb})$ we use the result from Lemma~\ref{proof:lemma_bias_unfairness}. Since $\mathcal{U}(f^{bb}) \leq \mathcal{U}_L(f^{bb}) + \text{Bias}$, we substitute the bound for $\mathcal{U}_L(f^{bb})$ derived above:
\begin{align*}
  \mathcal{U}(f^{bb}) &\leq 2\widehat{\mathcal{U}}_L(f^{bb}) + 8M^2\sqrt{\frac{2LK}{\,n}\,
  \log\!\left(2KLn\right)} + \frac{16M^2 \Lcdf}{L}.
\end{align*}
Which concludes the proof.

 \newpage 
\section{Experimental Details}
\label{an:experiments}

This section provides the complete details necessary to reproduce our experiments, including data generation, variable estimation, model architectures, and evaluation metrics.

\subsection{Datasets} \label{an:datasets}

\paragraph{Synthetic Data.}

All synthetic experiments use $n_{\mathrm{train}}=1000$ and $n_{\mathrm{test}}=10\,000$ samples. We generate the latent variable $V$ and sensitive attribute $S$ as
\begin{equation*}
  V \sim \mathrm{Unif}(0,1),\qquad S \sim \mathrm{Bernoulli}(0.5).
\end{equation*}
To intentionally expose the information loss inherent in causal methods, we map the sensitive attribute to a Rademacher random variable $S' = 2S - 1 \in \{-1, 1\}$. 
We then construct a two-dimensional observable feature vector $\mathbf{X} = (X_1, X_2)$ and the target outcome $Y$ according to
\begin{align*}
  X_1 &= S' V, \\
  X_2 &\sim \mathrm{Unif}(-\epsilon_x, \epsilon_x), \\
  Y &= X_1 + X_2 + \varepsilon_y,
\end{align*}
where the target micro-noise is $\varepsilon_y \sim \mathrm{Unif}(-0.01, 0.01)$. 
In our experiments, we set the feature noise scale to $\epsilon_x = 0.5$. 

This data generation process serves a dual purpose. 
First, the $X_1$ component shifts the marginal distributions of the groups in opposite directions, resulting in nearly disjoint structural supports. 
This separation guarantees an accuracy failure for methods enforcing global Demographic Parity. 
Second, the variance of $X_2$ injects an unbiased predictive signal that is completely independent of the sensitive attribute. 
Because strict causal methods (such as \texttt{Fair K}) explicitly discard the observable $\mathbf{X}$ to predict using only the latent $V$, they are blind to $X_2$. 
This results in a severe, artificial degradation of their predictive utility. 
Conversely, our Counterfactual Fairness framework post-processes an unconstrained base model that has full access to $\mathbf{X}$. 
This allows it to repair the biased dependency on $X_1$ while preserving the legitimate predictive information $X_2$.

\paragraph{Law School (LSAC) Data.}
We use the dataset provided by \citet{kusner2018counterfactual}\footnote{\url{https://github.com/mkusner/counterfactual-fairness}}. The outcome variable $Y$ is the first-year law school average (ZFYA).
\begin{itemize}
    \item \textbf{Sensitive Attribute ($S$):} Race, restricted to the two largest groups: White ($S=0$) and Black ($S=1$).
    \item \textbf{Features ($X$):} For the SCM estimation, we use LSAT, UGPA, race, and sex. For the black-box predictor input, we restrict $X$ to LSAT (one-dimensional) to strictly visualize the trade-offs, while $V$ captures the remaining signal.
\end{itemize}

\subsection{Estimation of the Exogenous Variable \U} \label{an:estimation_V}

A core assumption of our framework is access to a proxy $V$ for the unobserved latent variable $\U$. We compare two distinct approaches to estimate this proxy from observational data $(\mathbf{X}, S)$: the Structural Causal Model method from \citet{kusner2018counterfactual} and a VAE method.

\subsubsection{Structural Causal Model (SCM)}
As stated that we adopt the ``Level 2'' framework of \citet{kusner2018counterfactual}. Here, we detail the specific generative model used to infer the latent variable $\U$.

The LSAC dataset is modeled using a factor analysis structure where the latent knowledge $\U$ is the sole driver of correlation between outcomes, conditional on race and sex. Let $S$ denote the vector of sensitive attributes (race indicators and sex). The observed variables are generated as follows:
\begin{enumerate}
    \item \textbf{Latent Knowledge:} The unobserved ability is drawn from a standard normal distribution:
    \[ \U \sim \mathcal{N}(0, 1) \]
    \item \textbf{UGPA:} Modeled as a linear function of race, sex, and knowledge:
    \[ \text{UGPA} \sim \mathcal{N}(\mu_{\text{GPA}} + w_{\text{GPA}}^\top S + \lambda_{\text{GPA}}\,\U, \;\sigma_{\text{GPA}}^2) \]
    \item \textbf{LSAT Score:} Modeled as a count variable (approximated via Poisson) driven by the same factors:
    \[ \text{LSAT} \sim \text{Poisson}(\exp(\mu_{\text{LSAT}} + w_{\text{LSAT}}^\top S + \lambda_{\text{LSAT}}\,\U)) \]
    \item \textbf{First-Year Average (ZFYA):} The outcome variable is also a noisy linear function:
    \[ \text{ZFYA} \sim \mathcal{N}(\mu_{\text{ZFYA}} + w_{\text{ZFYA}}^\top S + \lambda_{\text{ZFYA}}\,\U, \;1) \]
\end{enumerate}

\textbf{Inference Procedure.}
We use the Stan implementation provided by \citet{kusner2018counterfactual}. The model is fitted using Hamiltonian Monte Carlo (HMC) with 2000 iterations. For each student $i$, the proxy $V_i$ used in our experiments is the posterior mean of the latent variable given their observed data:
\[ V_i = \mathbb{E}[\U_i \mid \text{UGPA}_i, \text{LSAT}_i, \text{ZFYA}_i, S_i]. \]

\subsubsection{Standard Variational Autoencoder (VAE)} \label{an:standard_VAE}
To assess robustness against model misspecification, we also estimate $V$ using a Variational Autoencoder (VAE) that does not assume the rigid linear/Poisson structure of the SCM.

\textbf{Network Architecture.}
To test robustness, we also estimate $V$ using a Variational Autoencoder (VAE).
\begin{itemize}
    \item \textbf{Architecture:} We use a fully connected network with LeakyReLU activations ($0.2$) and Batch Normalization. The encoder compresses the $d$-dimensional input to a 1D latent space via hidden layers of size 32 and 16. The decoder mirrors this topology ($1 \to 16 \to 32 \to d$) to reconstruct the inputs.
    \item \textbf{Training:} We optimize a $\beta$-VAE loss ($\beta=0.5$) for 100 epochs using the Adam optimizer (learning rate $10^{-3}$) to balance reconstruction quality with latent regularization.
    \item \textbf{Proxy Extraction:} The proxy $V_i$ is the mean of the approximate posterior distribution, min-max normalized to $[0, 1]$.
\end{itemize}

\textbf{Hyperparameters.}
\begin{itemize}
    \item \textbf{Loss Function:} We use a $\beta$-VAE loss with $\beta=0.5$ to reduce the regularization strength, ensuring the latent variable $V$ retains significant information about the input features rather than collapsing to the prior. The objective is:
    \[ \mathcal{L} = \text{MSE}(\mathbf{x}, \hat{\mathbf{x}}) + 0.5 \cdot D_{KL}(q_\phi(z|\mathbf{x}) \,||\, \mathcal{N}(0,1)), \]
    where $\mathbf{x}$ is the standardized input vector, $\hat{\mathbf{x}}$ is the decoder reconstruction, and $D_{KL}$ measures the divergence between the encoder's approximate posterior $q_\phi(z|\mathbf{x})$ and the standard normal prior.
    \item \textbf{Optimization:} Adam optimizer with learning rate $10^{-3}$ and weight decay $10^{-5}$.
    \item \textbf{Training:} 100 epochs with a batch size of 64.
\end{itemize}
\textbf{Post-Processing.}
Since the VAE latent space is unbounded ($\mathbb{R}$), we normalize the extracted means $V_i = \mu_\phi(\mathbf{x}_i)$ to the interval $[0, 1]$ via min-max scaling before feeding them into our discretization algorithm.

\subsubsection{Fair Variational Autoencoder (Fair VAE)} \label{an:fair_VAE}
To mitigate systematic proxy leakage and actively enforce statistical independence between the latent proxy and the sensitive attribute ($V \perp S$), we introduce a Fair VAE variant used in our robustness experiments. 

\textbf{Network Architecture.}
This model shares the exact same underlying architecture as the standard VAE, but modifies the training objective to disentangle the latent space.
\begin{itemize}
    \item \textbf{Architecture:} We use a fully connected network with LeakyReLU activations ($0.2$) and Batch Normalization. The encoder compresses the $d$-dimensional input to a 1D latent space via hidden layers of size 32 and 16. The decoder mirrors this topology ($1 \to 16 \to 32 \to d$) to reconstruct the inputs.
    \item \textbf{Training:} We optimize a joint objective combining a $\beta$-VAE loss and a Maximum Mean Discrepancy (MMD) penalty for 100 epochs using the Adam optimizer to balance reconstruction, latent regularization, and demographic independence.
    \item \textbf{Proxy Extraction:} The proxy $V_i$ is the mean of the approximate posterior distribution, min-max normalized to $[0, 1]$.
\end{itemize}

\textbf{Hyperparameters.}
\begin{itemize}
    \item \textbf{Loss Function:} The model optimizes a standard $\beta$-VAE loss ($\beta=0.5$) combined with an MMD penalty to enforce parity in the latent space. During training, the MMD is computed across the mini-batch using a Radial Basis Function (RBF) kernel ($\gamma=1.0$) to measure the distributional distance between the encoded latent means ($\mu$) of the different demographic groups ($S=0$ and $S=1$). The total objective is:
    \[ \mathcal{L}_{\text{Fair}} = \text{MSE}(\mathbf{x}, \hat{\mathbf{x}}) + 0.5 \cdot D_{KL}(q_\phi(z|\mathbf{x}) \,||\, \mathcal{N}(0,1)) + \lambda_{\text{MMD}} \cdot \text{MMD}(\mu_{S=0}, \mu_{S=1}), \]
    where we set the MMD penalty weight to $\lambda_{\text{MMD}} = 10.0$. This explicit regularization forces the encoder to strip away demographic-specific structural biases.
    \item \textbf{Optimization:} Adam optimizer with learning rate $10^{-3}$ and weight decay $10^{-5}$.
    \item \textbf{Training:} 100 epochs with a batch size of 64.
\end{itemize}

\textbf{Post-Processing.}
Since the VAE latent space is unbounded ($\mathbb{R}$), we normalize the extracted means $V_i = \mu_\phi(\mathbf{x}_i)$ to the interval $[0, 1]$ via min-max scaling before feeding them into our discretization algorithm.

\subsection{Black-Box Predictors}

\paragraph{Synthetic.}
We use a ridge-regularized linear model:
\begin{equation*}
  f^{bb}(x,v,s) = \theta_0 + \theta_1 x + \theta_2 v + \theta_3 s.
\end{equation*}
(In the unaware case, $\theta_3$ is forced to 0). The coefficients are obtained via ridge regression with regularization 0.1.

\paragraph{LSAC.}
We use a flexible polynomial ridge regression. For each student, we form an input vector $z=(X,V,S)$. We standardize $z$, expand it to all monomials of total degree at most 2 (including squares and pairwise interactions), and fit a ridge regression:
\begin{equation*}
    f^{bb}(x,v,s) \;=\; \beta_0 \;+\; \sum_j \beta_j \,\tilde z_j \;+\; \sum_{j \le k} \gamma_{jk}\,\tilde z_j \tilde z_k.
\end{equation*}
Because LSAC only has a few features, this yields a smooth, bounded predictor that satisfies our regularity assumptions without overfitting.

\paragraph{Communities and Crime.}
For the high-dimensional Crime dataset, we use a standard linear ridge regression ($\lambda = 1.0$) on the observables:
\begin{equation*}
    f^{bb}(x,v,s) = \theta_0 + \theta_x^T x + \theta_v v + \theta_s s.
\end{equation*}
Unlike LSAC, we intentionally restrict this predictor to a linear model without polynomial expansion. Because the Crime dataset contains over 100 features, a degree-2 polynomial expansion would drastically explode the feature space, leading to severe overfitting and matrix instability. The linear ridge regression ensures a stable, well-regularized, and continuous predictor.

\subsection{Implementation of the Discretized Post-Processor} \label{subsec:implementation_postprocessor}

\paragraph{Algorithm.}
Given a predictor $f^{bb}$ and a discretization level $L$
\begin{enumerate}
    \item We partition $V$ into $L$ equal-width intervals $\{\mathcal I_\ell\}_{\ell=1}^L$.
    \item In each interval $\ell$ and group $s$, we collect the scores $z_i = f^{bb}(x_i, s_i)$.
    \item We compute empirical quantiles $q_{\ell,s}$ on a fixed grid and form the barycenter quantiles $q_{\ell,\star} = \sum_s w_s q_{\ell,s}$.
    \item For a new test point with score $z_{new}$ in group $s$ and interval $\ell$, we:
    \begin{itemize}
        \item Compute its percentile $\tau$ within the group $s$ distribution via linear interpolation.
        \item Map $\tau$ to the barycenter distribution: $\widehat{y} = q_{\ell,\star}(\tau)$.
    \end{itemize}
\end{enumerate}

\paragraph{Plug-in Selection of $L^*$.}
We select the optimal discretization level $L^*$ using the formula derived in Theorem~\ref{thm:l_star}. We estimate the Lipschitz constant $L_{\mathrm{cdf}}$ using finite differences on the training data and set:
\begin{equation*}
  L^* \approx \left\lfloor \left( \frac{8 L_{\mathrm{cdf}}^2 \,n}{K \log(2 K n)} \right)^{1/3} \right\rfloor.
\end{equation*}
We cap $L^*$ to ensure a minimum number of samples per cell (typically $\ge 10$) to preserve statistical stability. This $L^*$ is used to build the
relaxed post-processor $f_\alpha = \sqrt{\alpha}\,f^{bb} + (1-\sqrt{\alpha})\,\widehat f_{L^*}$
for $\alpha \in \{0,0.1,\dots,1\}$, from which we report the empirical
risk–unfairness trade-off curves.

\subsection{Evaluation Metrics}

\paragraph{Unfairness Approximation.}
We approximate the integral $\mathcal{U}(f)$ using a Monte Carlo Riemann sum. We fix a grid of points $v_0$ on $[0,1]$. At each $v_0$, we select test points within a window $|V_i - v_0| \le h$, compute the group-wise quantiles, and calculate the squared Wasserstein-2 distance to the barycenter. These local distances are averaged over all valid windows.

\paragraph{Uncertainty Estimation.}
For all plots, we perform $B=30$ independent repetitions (reshuffling training/test splits). We report the mean and the 99\% confidence interval.

\subsection{Computational complexity} \label{an:complexity}

Since we operate strictly on 1D scalar predictions, Wasserstein distances and barycenters admit closed-form solutions via quantile functions. 
Thus, the computational complexity scales as $\tilde{\mathcal O}(n \log(n))$: according to the steps from \cref{alg:post_processing_compact}, partitioning the $n$ samples into $L^*$ intervals based on their proxy $V$ takes $\tilde{\mathcal O}(n)$; 
sorting the black box predictions (1D) within each interval $\ell$ to construct empirical c.d.f. $\widehat{F}_{\ell,s}$ (and quantile functions) takes $\tilde{\mathcal O}(n \log(n))$; 
computing the barycenter quantiles is $O(L^* K)$, which is negligible compared to sorting; and linearly aggregating these $K$ group quantiles to compute the local Wasserstein barycenters takes $\tilde{\mathcal{O}}(Kn)$.
Since $K$ is typically a small constant, the overall complexity is dominated by the sorting step, yielding $\tilde{\mathcal O}(n \log(n))$.

We empirically validate this complexity in \cref{fig:complexity} from \cref{an:additional_results}.

\section{Additional Experimental Results}
\label{an:additional_results}

In this section, we provide extensive supplementary evaluations to validate our theoretical claims, test the framework's robustness against its core assumptions, and compare against additional baselines and datasets.
\subsection{Real-World Datasets: Law School and Communities and Crime}

We provide here the absolute values for our metrics and extend our analysis to an additional real-world dataset. 
The \textbf{Base} model refers to a predictor trained to minimize standard squared risk without fairness constraints.
Results are given relative to base as described in Appendix~\ref{an:experiments}.

\textbf{Law School results (Table~\ref{tab:law_abs}).} 
The results are categorized by the method used to infer the latent proxy $V$:
\begin{itemize}
    \item \textbf{SCM Proxy:} Using the structural causal model from \citet{kusner2018counterfactual}, our method achieves nearly perfect Demographic Parity ($0.0001$) while maintaining an RMSE ($0.9078$) significantly closer to the Base model than the \texttt{Fair K} baseline ($0.9356$).
    \item \textbf{Standard VAE Proxy:} 
    In the standard VAE setting (see \cref{an:standard_VAE} for details), the Base model appears ``fairer'' in terms of CF initially ($0.0005$), suggesting that the VAE's latent space captures different data manifold aspects than the SCM. 
    However, our method still reduces the CF further to $0.0001$ while maintaining the best RMSE among all fairness-constrained competitors.
    \item \textbf{Fair VAE Proxy:} 
    Under a Standard VAE, the extracted proxy heavily encodes the sensitive attribute (proxy leakage), impacting both CF and DP for methods conditioning on it. 
    The Fair VAE corrects this by explicitly regularizing the latent space (adding the MMD penalty) to be independent of the sensitive attribute (see \cref{an:fair_VAE} for details). 
    On this disentangled proxy, our method successfully restores CF ($0.0142$) and DP ($0.0099$).
\end{itemize}

\begin{table}[h!]
  \centering
\small 
  \begin{tabular*}{\columnwidth}{@{\extracolsep{\fill}}llccc}
      \toprule
      Proxy & Method & RMSE $\downarrow$ & CF $\downarrow$ & DP $\downarrow$ \\
      \midrule
           & \texttt{Fair K} & 1.0861 & \textbf{0.0008} & \textbf{0.0011} \\
      SCM  & WFR             & \textbf{1.0518} & 0.0885 & 0.0050 \\
           & Ours            & 1.0539 & 0.0084 & 0.0022 \\
      \midrule
                   & \texttt{Fair K} & \textbf{1.0032} & 0.3467 & 0.6869 \\
      Standard VAE & WFR             & 1.0238 & 15.4846 & \textbf{0.0042} \\
                   & Ours            & \textbf{1.0032} & \textbf{0.2895} & 0.7155 \\
      \midrule
               & \texttt{Fair K} & 1.0380 & \textbf{0.0003} & 0.0093 \\
      Fair VAE & WFR             & \textbf{1.0241} & 0.4322 & \textbf{0.0043} \\
               & Ours            & 1.0266 & 0.0142 & 0.0099 \\
      \bottomrule
  \end{tabular*}
  \vspace{0.3em}
  \caption{\textbf{Relative to Base erformance on the Law School Dataset across different proxy $V$.} 
  We evaluate the impact of extracting the unobserved proxy $V$ via a known SCM, a Standard VAE, and a Fair VAE relative to the Base performances. 
  Under a Standard VAE, the extracted proxy heavily encodes the sensitive attribute (proxy leakage), impactong both CF and DP for methods conditioning on it. 
  A Fair VAE corrects this by explicitly regularizing the latent space (adding a MMD penalty) to be independent of the sensitive attribute. 
  On this disentangled Fair VAE proxy, our method successfully restores CF ($0.0142$) and DP ($0.0099$).}    
  \label{tab:law_abs}
\end{table}

\textbf{Communities and Crime results (Table~\ref{tab:crime_rel}).} 
To further evaluate the impact of proxy leakage, we apply our framework to the Communities and Crime dataset. 
As shown in Table~\ref{tab:crime_rel}, when using a Standard VAE, the extracted proxy remains entangled with the sensitive attribute. 
Because our method strictly conditions on this proxy, it inherits this ``proxy leakage," leading to degraded CF and DP. 
Conversely, employing the Fair VAE successfully disentangles the intrinsic ability from the sensitive attribute. 
Under this corrected proxy, our method reaches perfect CF ($0.000$) and DP ($0.001$). 
Furthermore, our method exhibits a superior fairness-utility trade-off: while the strict causal selection of Fair K also achieves perfect fairness under the Fair VAE, it degrades utility (RMSE: $1.632$), whereas our post-processor preserves better base accuracy (RMSE: $1.324$).

\begin{table}[ht!]
    \centering
    \caption{\textbf{Relative to Base performance on the Communities and Crime Dataset (relative to base).} This validates that mitigating proxy leakage via a Fair VAE allows our method to achieve perfect fairness while dominating the utility of Fair K.}
    \label{tab:crime_rel}
    \small
    \begin{tabular*}{\linewidth}{@{\extracolsep{\fill}}llccc}
        \toprule
        Proxy & Method & RMSE $\downarrow$ & CF $\downarrow$ & DP $\downarrow$ \\
        \midrule
        \multirow{3}{*}{Standard VAE} 
        & \texttt{Fair K} & 1.344 & \textbf{0.043} & 0.348 \\
        & WFR  & 1.296 & 0.275 & \textbf{0.006} \\
        & Ours & \textbf{1.106} & 0.065 & 0.147 \\
        \midrule
        \multirow{3}{*}{Fair VAE} 
        & \texttt{Fair K} & 1.632 & \textbf{0.000} & \textbf{0.001} \\
        & WFR  & \textbf{1.296} & 0.081 & 0.006 \\
        & Ours & 1.324 & 0.000 & 0.001 \\
        \bottomrule
    \end{tabular*}
\end{table}

\textbf{Optimal Discretization and Trade-offs (Law School).}
Figure~\ref{fig:lsac_analysis} replicates our algorithmic analysis on the real-world LSAC dataset. 
\textbf{Left:} The relationship between conditional unfairness and $L$ exhibits a clear bias-variance trade-off. 
Our theoretical $L^*_{th}$ (dashed line) closely aligns with the empirical minimum, confirming that our selection rule effectively scales to real-world distributions. 
\textbf{Center:} The Pareto frontier demonstrates a smooth, convex interpolation between risk and CF as the relaxation parameter $\alpha$ increases, empirically validating Proposition~\ref{prop:closed_form_relaxed}. 
\textbf{Right:} Consistent with our synthetic findings, when evaluated on global DP, WFR achieves a lower risk tradeoff since it explicitly targets global parity, whereas our method preserves conditional fairness.

\begin{figure}[ht!]
  \centering
  \begin{minipage}[t]{0.32\textwidth}
    \centering
    \includegraphics[width=\linewidth]{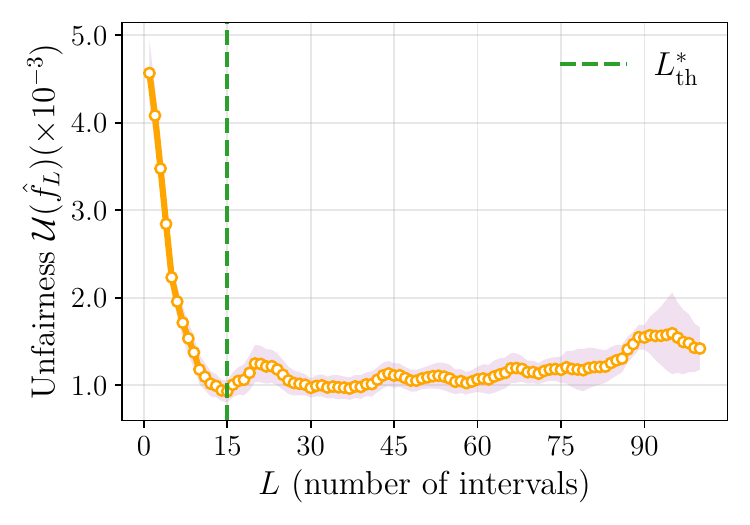}
  \end{minipage}
  \hfill 
  \begin{minipage}[t]{0.32\textwidth}
    \centering
    \includegraphics[width=\linewidth]{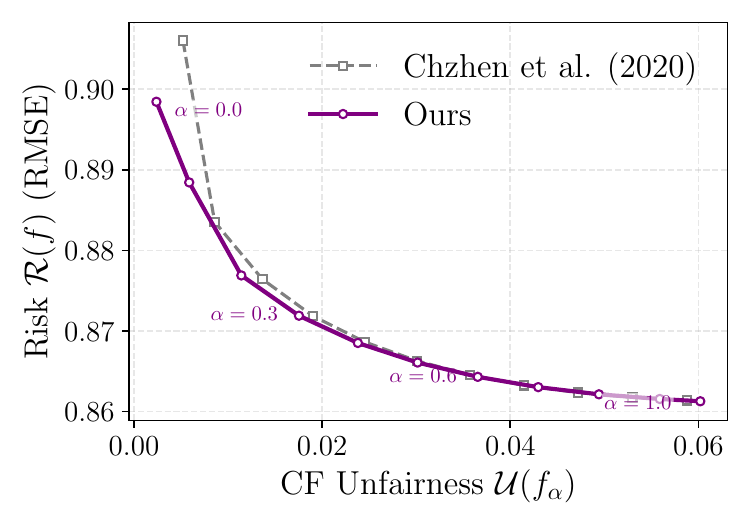}
  \end{minipage}
  \hfill
  \begin{minipage}[t]{0.32\textwidth}
    \centering
    \includegraphics[width=\linewidth]{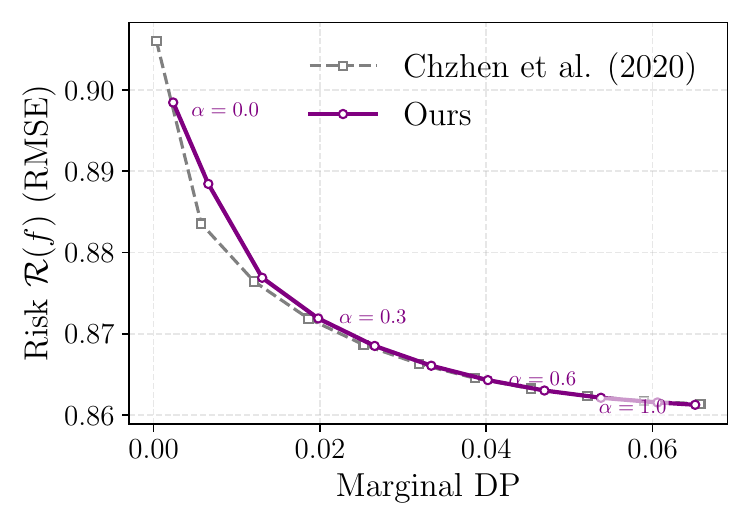}
  \end{minipage}
  \caption{\textbf{LSAC: Analysis of fairness and risk trade-offs.} 
  (Left) Unfairness vs. number of intervals $L$, showing the bias-variance trade-off with the theoretical $L^*_{\mathrm{th}}$ indicated by the vertical dashed line. 
  (Center) Pareto frontier of risk vs. unfairness for varying relaxation parameter $\alpha$, demonstrating smooth interpolation. 
  (Right) Pareto frontier of risk vs. Demographic Parity for different $\alpha$. Shaded regions represent 99\% confidence intervals.}
  \label{fig:lsac_analysis}
\end{figure}

\subsection{Absolute metrics on Synthetic dataset}

While the main text focuses on relative performance to highlight the impact of fairness interventions, we provide here the absolute values for the RMSE, Counterfactual Unfairness (CF), and Demographic Parity (DP) in Table~\ref{tab:synthetic_abs}.

The \textbf{Base} model refers to a predictor trained to minimize standard squared risk without fairness constraints, as described in Appendix~\ref{an:experiments}.

\begin{table}[ht!]
\centering
\caption{\textbf{Absolute performance metrics on Synthetic data.}}
\label{tab:synthetic_abs}
\small
\begin{tabular}{@{}lccc@{}}
  \toprule
  Method & RMSE $\downarrow$ & CF $\downarrow$ & DP $\downarrow$ \\
  \midrule
  Base            & \textbf{0.0058} & 0.3303 & 0.2433 \\
  \texttt{Fair K} & 0.6416 & \textbf{0.0000} & \textbf{0.0000} \\
  WFR             & 0.4946 & 0.0833 & 0.0001 \\
  Ours            & 0.5853 & 0.0005 & 0.0001 \\
  \bottomrule
\end{tabular}
\end{table}

\textbf{Synthetic Results (Table~\ref{tab:synthetic_abs}).} 
The absolute metrics highlight the information loss suffered by strict purely causal methods under our new setup. 
The unconstrained Base model achieves an exceptionally low RMSE ($0.0058$) because it fully utilizes both the biased and unbiased signals, but this comes at the cost of severe Counterfactual Unfairness ($0.3303$). 
Enforcing fairness requires breaking this statistical dependence, illustrating a natural ``price of fairness.'' 
However, because \texttt{Fair K} discards the observables to predict using only the latent variable $V$, it throws away the wide $\varepsilon_x$ variance, causing its RMSE to inflate to $0.6416$. 
In contrast, our method achieves near-perfect Counterfactual Fairness ($0.0005$) while preserving the unbiased predictive signal from the base model, yielding a superior RMSE ($0.5853$) compared to \texttt{Fair K}.

\subsection{Empirical Convergence and Computational Complexity}

Our theoretical analysis (\cref{an:complexity}) establishes a worst-case convergence rate of $\mathcal{O}(n^{-1/3})$. 
Figure~\ref{fig:convergence} evaluates this empirically on synthetic data. 
The log-log plot demonstrates that empirical unfairness decays significantly faster than the theoretical bound, reaching near-zero levels ($\mathcal{U}(\widehat{f}_{L^*}) < 10^{-5}$) with just $n \approx 2000$ samples. 

Furthermore, because our post-processor operates entirely on 1D scalars rather than high-dimensional features, it is highly scalable. 
Figure~\ref{fig:complexity} confirms an empirical $\mathcal{O}(n \log n)$ runtime scaling, efficiently processing $n=10^6$ samples in approximately $10$ seconds.

\begin{figure}[h!]
    \centering
    \begin{minipage}[t]{0.48\textwidth}
        \centering
        \includegraphics[width=\linewidth]{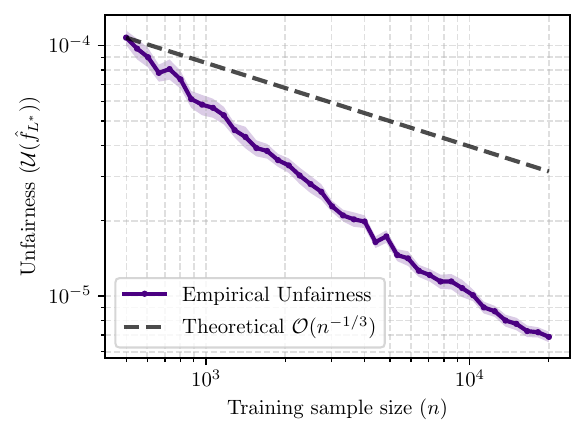}
        \caption{\textbf{Empirical convergence vs. Theoretical bound.} The theoretical rate is aligned at $n=500$. Empirical decay is notably faster than the worst-case bound, reaching near-zero unfairness rapidly.}
        \label{fig:convergence}
    \end{minipage}
    \hfill
    \begin{minipage}[t]{0.48\textwidth}
      \centering
      \includegraphics[width=\linewidth]{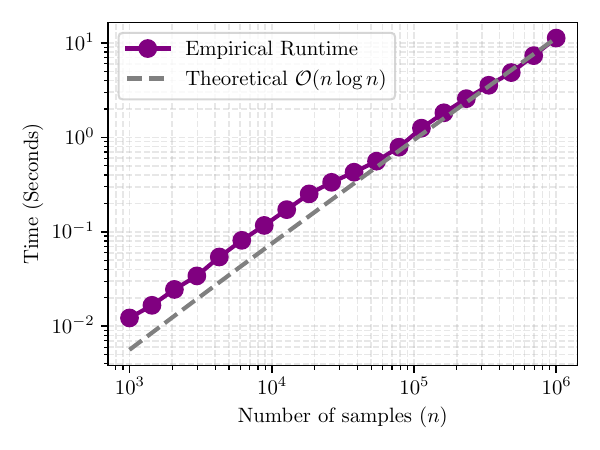}
      \caption{\textbf{Empirical runtime vs. sample size $n$.} The log-log slope tracks the theoretical $\mathcal{O}(n \log n)$ scaling, confirming the efficiency of our 1D optimal transport formulations.}
      \label{fig:complexity}
  \end{minipage}
\end{figure}

\subsection{Multi-Group Extensions ($K > 2$)}

\textbf{Setup.} We generalize our synthetic dataset to $K$ groups by sampling $S \in \{0, \dots, K-1\}$ uniformly and $V \sim \mathcal{U}(0,1)$. 
To create disjoint supports across all groups, we define the feature as $X = 2S - (K-1) + V + \varepsilon_x$ and the target as $Y = X + \varepsilon_y$. 
Crucially, we set the feature noise to $\varepsilon_x \sim \mathcal{U}(-0.5, 0.5)$ and the target noise to $\varepsilon_y \sim \mathcal{U}(-0.01, 0.01)$. 
This construction serves a dual purpose: the distinct group shifts penalize global Demographic Parity, while the wide variance of $\varepsilon_x$ injects legitimate predictive signal that exposes the information loss of strict causal methods which discard $X$.

\textbf{Results.} Our post-processor naturally accommodates settings with more than two sensitive groups. 
Our theoretical bounds indicate that while the discretization bias ($\mathcal{O}(1/L)$) is independent of $K$, the estimation error scales as $\mathcal{O}(\sqrt{K L/n})$. 
Our closed-form theoretical value of $L$ adapts by scaling as $L^\ast \propto K^{-1/3}$, shifting leftward to prevent overfitting. 

Figure~\ref{fig:unfairness_scaling_k} confirms this on synthetic data for $K \in \{3, 5, 10\}$: as $K$ increases and data sparsity becomes extreme, the theoretical $L^*_{th}$ dynamically adapts to locate the empirical minimum. 

Table~\ref{tab:multi_group} demonstrates that despite the separation of the $K$ groups, our method enforces near-perfect Counterfactual Fairness. 
Crucially, by post-processing a base model that observes $X$, our framework captures the legitimate $\varepsilon_x$ variance and achieves better accuracy (lower RMSE) than \texttt{Fair K}, which discards this signal.

\begin{figure}[h!]
    \centering
    \begin{minipage}{0.32\linewidth}
        \centering
        \includegraphics[width=\linewidth]{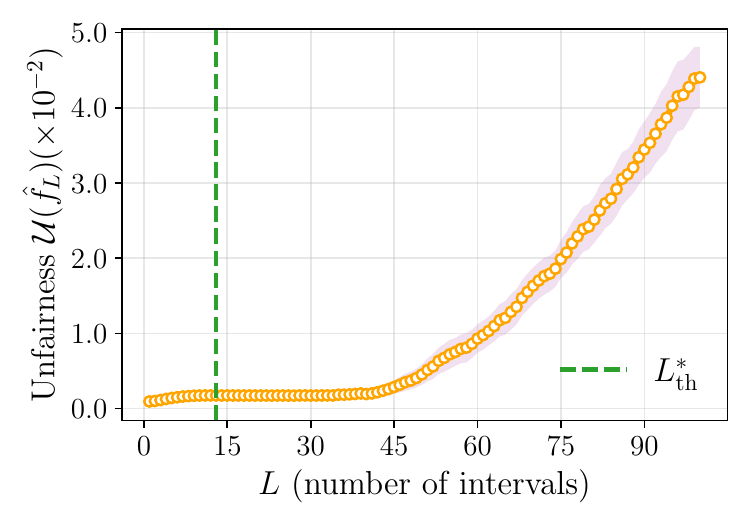}
    \end{minipage}
    \hfill
    \begin{minipage}{0.32\linewidth}
        \centering
        \includegraphics[width=\linewidth]{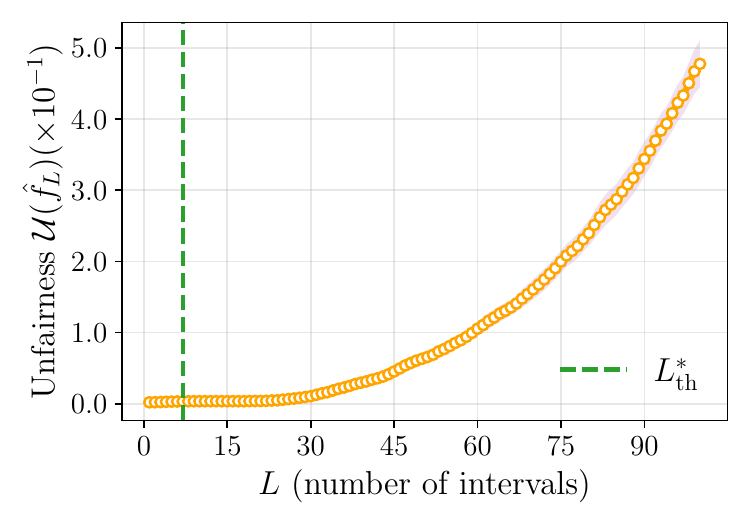}
    \end{minipage}
    \hfill
    \begin{minipage}{0.32\linewidth}
        \centering
        \includegraphics[width=\linewidth]{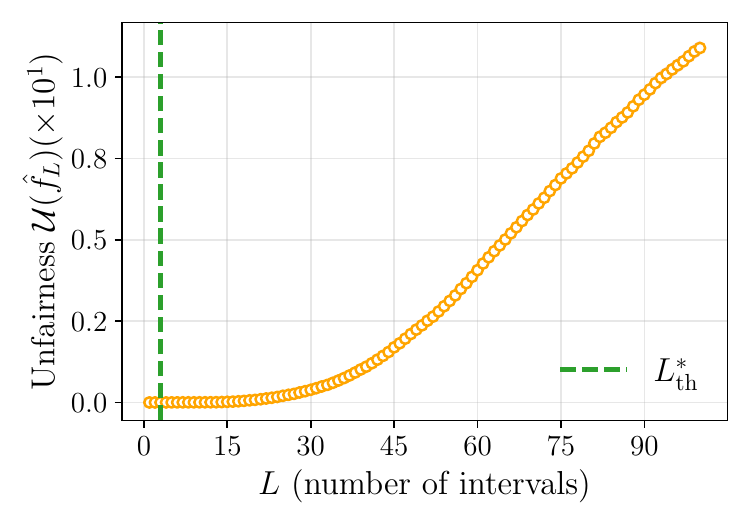}
    \end{minipage}
    \caption{\textbf{Empirical validation of the bias-error trade-off (Thm 1) across multi-class sensitive attributes ($K \in \{3, 5, 10\}$).} 
    Solid lines represent mean conditional unfairness (30 runs) with 95\% CI shaded. 
    As $K$ increases, extreme data sparsity causes the empirical c.d.f. estimates to degrade rapidly as $L$ grows (seen in the $K=10$ spike). 
    Our theoretical $L^*_{th}$ dynamically adapts to locate the empirical minimum.}
    \label{fig:unfairness_scaling_k}
\end{figure}

\begin{table}[h!]
  \centering
  \small
  \begin{tabular*}{\columnwidth}{@{\extracolsep{\fill}}l|ccc|ccc|ccc}
      \toprule
      & \multicolumn{3}{c|}{$K=3$} & \multicolumn{3}{c|}{$K=5$} & \multicolumn{3}{c}{$K=10$} \\
      Method & RMSE $\downarrow$ & CF $\downarrow$ & DP $\downarrow$ & RMSE $\downarrow$ & CF $\downarrow$ & DP $\downarrow$ & RMSE $\downarrow$ & CF $\downarrow$ & DP $\downarrow$ \\
      \midrule
      Base     & 1.0000 & 1.0000 & 1.0000 & 1.0000 & 1.0000 & 1.0000 & 1.0000 & 1.0000 & 1.0000 \\
      \texttt{Fair K} & 5.8251 & \textbf{0.0000} & \textbf{0.0000} & 9.8851 & \textbf{0.0000} & \textbf{0.0000} & 19.8666 & \textbf{0.0000} & \textbf{0.0000} \\
      WFR      & 5.6661 & 0.0004 & 0.0003 & 9.7596 & 0.0002 & 0.0002 & 19.7440 & 0.0001 & 0.0001 \\
      Ours     & 5.7223 & 0.0003 & 0.0002 & 9.7596 & 0.0002 & 0.0002 & 19.7440 & 0.0001 & 0.0001 \\
      \bottomrule
  \end{tabular*}
  \vspace{0.3em}
  \caption{\textbf{Relative to base summary for multiclass sensitive attribute settings ($K \in \{3, 5, 10\}$).} 
  Values are normalized relative to the Base model. 
  Because the $K$ groups have disjoint supports, enforcing fairness causes a massive relative increase in RMSE across all baselines. 
  Despite the data sparsity for $K=10$, our method successfully adapts its discretization to enforce Counterfactual Fairness while strictly improving upon the accuracy of \texttt{Fair K} by leveraging the legitimate feature variance $\varepsilon_x$ that strict causal methods discard.}
  \label{tab:multi_group}
\end{table}

\subsection{Robustness to model assumptions and hyperparameters}

Our theoretical framework relies on assumptions regarding the smoothness of the base predictor \cref{ass:distrib_densities,ass:lipschitz_cdf} and the estimation of hyperparameters like $L_{cdf}$. 
Here, we test the robustness of the method in cases where these assumptions are not satisfied.

\textbf{Base Predictors not satisfying \cref{ass:distrib_densities,ass:lipschitz_cdf}.} 
While we theoretically assume a Lipschitz continuous c.d.f. to bound discretization bias, Figure~\ref{fig:nonlinear_robustness} shows our method remains highly robust when applied to non-smooth models like Random Forests and Gradient Boosted Trees. 
While the initial spike in unfairness for small $L$ is more pronounced, the empirical discretization bias rapidly vanishes as $L$ increases.

\begin{figure}[ht!]
  \centering
  \includegraphics[width=0.6\linewidth]{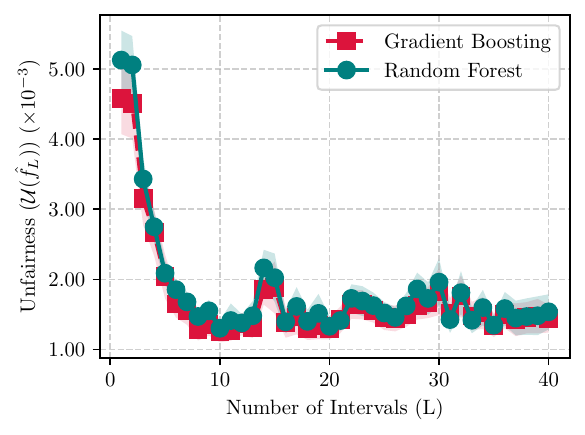}
  \caption{\textbf{Discretization Trade-off on the Law School Dataset for black-box models not satisfying \cref{ass:distrib_densities,ass:lipschitz_cdf}.} Applied to Gradient Boosting (red squares) and Random Forest (teal circles). The impact of non-smoothness dissipates as $L$ increases, proving the method works effectively without the Lipschitz assumption on the black-box model.}
  \label{fig:nonlinear_robustness}
\end{figure}

\textbf{Minority Group Imbalance.} 
To evaluate robustness to demographic scarcity, we generate our synthetic dataset while varying the minority sampling probability $w_{\min}$ from $0.5$ down to near-zero. 
Figure~\ref{fig:imbalance_robustness} illustrates the impact on RMSE and unfairness. 
As $w_{\min} \to 0$, the global RMSE (red) decreases; because the dataset becomes dominated by the majority group, the global barycenter converges to the majority distribution, resulting in negligible accuracy penalties. 
The empirical unfairness (orange) remains completely stable and near-zero ($\approx 5 \times 10^{-4}$) for the vast majority of the regime. 
It is only under extreme scarcity ($w_{\min} < 0.05$) that the unfairness exhibits a spike. 
When the groups are too imbalanced, the minority group contains too few samples to reliably estimate its empirical c.d.f. over the variance of the feature space, causing the OT mapping to fail.

\textbf{Misestimation of $L_{cdf}$.} 
Figure~\ref{fig:lcdf_robustness} evaluates robustness to misestimating the theoretical constant $L_{cdf}$ by a factor $c = \widehat{L}_{cdf}/L_{cdf}$. 
The method exhibits a remarkably wide stable zone around the true value ($c=1$). 
Degradation strictly requires severe misestimation: extreme underestimation ($c \le 0.1$) induces high approximation bias, while an order-of-magnitude overestimation ($c \ge 10$) is required to finally destabilize the OT estimation, strongly corroborating the robustness predicted by our theoretical $c^{1/3}$ scaling.

\begin{figure}[ht!]
  \centering
  \begin{minipage}[t]{0.48\textwidth}
      \centering
      \includegraphics[width=\linewidth]{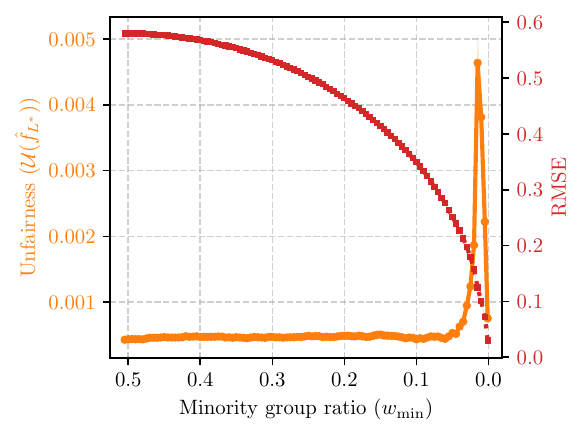}
      \caption{\textbf{Robustness to minority group imbalance ($w_{\min}$).} The method maintains stable, near-zero unfairness down to $w_{\min} \approx 0.05$. At extreme scarcity, the inability to reliably estimate the minority empirical c.d.f. causes a sharp spike in approximation bias, while global RMSE drops as the majority group dominates the metric.}
       \label{fig:imbalance_robustness}
  \end{minipage}
  \hfill
  \begin{minipage}[t]{0.48\textwidth}
      \centering
      \includegraphics[width=\linewidth]{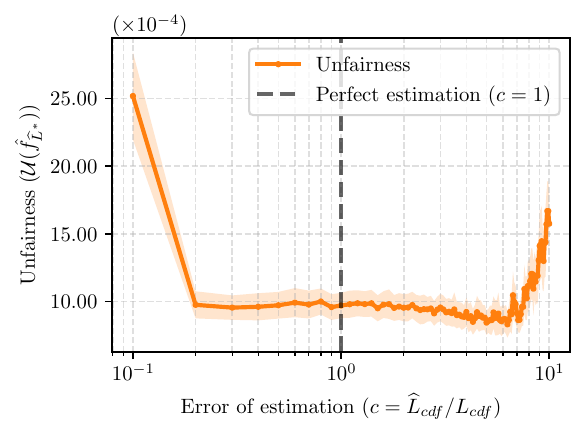} 
      \caption{\textbf{Ablation study on the robustness to $L_{cdf}$ estimation errors.} The method exhibits striking empirical robustness, corroborating the $c^{1/3}$ scaling predicted by our theoretical bounds.}
      \label{fig:lcdf_robustness}
  \end{minipage}
\end{figure}

\subsection{Impact of proxy quality and causal leakage}

The guarantee of our framework depends on the causal validity of the proxy $V$. 
We investigate two ways of proxy degradation: random measurement error and systematic proxy leakage. 

As shown in Figure~\ref{fig:random_noise}, if the proxy is estimated with independent Gaussian noise ($\widehat{V} = V_{true} + \mathcal{N}(0, \sigma^2)$), the injected noise is statistically independent of $S$. 
Consequently, fairness degrades only mildly (remaining below $0.05$), while RMSE smoothly decreases (dual axis). 
However, if the proxy systematically absorbs bias from observable features (proxy leakage, $\widehat{V}_\lambda = (1-\lambda)V_{true} + \lambda X$), Figure~\ref{fig:proxy_leakage} reveals a trade-off. 
Unfairness scales linearly with leakage before stabilizing, while RMSE decreases toward zero because the post-processor collapses back to the exact (but biased) unconstrained base predictions.
This validates our premise: CF relies heavily on accurately isolating ability from structural bias.

\begin{figure}[ht!]
  \centering
  \begin{minipage}[t]{0.48\textwidth}
      \centering
      \includegraphics[width=\linewidth]{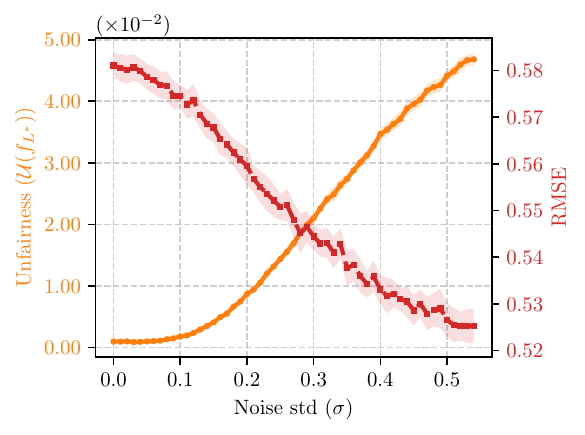}
      \caption{\textbf{Robustness to unbiased proxy noise.} Dual axes show that as noise $\sigma$ increases, unfairness rises mildly (peaking below $0.05$) while RMSE decreases.}
      \label{fig:random_noise}
  \end{minipage}
  \hfill
  \begin{minipage}[t]{0.48\textwidth}
      \centering
      \includegraphics[width=\linewidth]{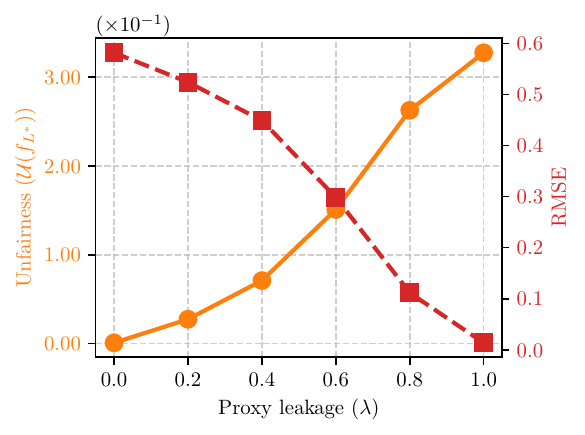}
      \caption{\textbf{Sensitivity to systematic proxy leakage.} Systematic leakage ($\lambda$) forces a sharp trade-off: unfairness scales linearly, while RMSE drops toward zero as the model recovers the biased base predictions.}
      \label{fig:proxy_leakage}
  \end{minipage}
\end{figure}

\begin{table}[ht!]
    \centering
    \small
    \begin{tabular*}{\linewidth}{@{\extracolsep{\fill}}llccc}
        \toprule
        Proxy & Method & RMSE & CF & DP \\
        \midrule
        \multirow{3}{*}{Standard  VAE} 
        & \texttt{Fair K} & 1.344 & \textbf{0.043} & 0.348 \\
        & WFR  & 1.296 & 0.275 & \textbf{0.006} \\
        & Ours & \textbf{1.106} & 0.065 & 0.147 \\
        \midrule
        \multirow{3}{*}{Fair VAE} 
        & \texttt{Fair K} & 1.632 & \textbf{0.000} & \textbf{0.001} \\
        & WFR  & \textbf{1.296} & 0.081 & 0.006 \\
        & Ours & 1.324 & 0.000 & 0.001 \\
        \bottomrule
    \end{tabular*}
    \vspace{0.3em}
    \caption{\textbf{Performance on the Communities and Crime Dataset and the Impact of Proxy Leakage (relative to base).} When using a Standard VAE, the extracted proxy remains entangled with the sensitive attribute, leading to degraded CF and DP. Employing a Fair VAE successfully disentangles the intrinsic ability. Under this corrected proxy, our post-processor preserves better base accuracy (RMSE: $1.324$) than the strict causal selection of Fair K (RMSE: $1.632$).}
    \label{tab:crime_rel}
\end{table}

\subsection{Comparison to in-processing Baselines}

Finally, we compare against \citet{zhou2025counterfactualfairnesscombiningfactual}, a recent in-processing baseline, in Table~\ref{tab:zhou_comparison}. 
We select this method as our primary baseline because it represents the state-of-the-art in characterizing and optimizing the inherent trade-off between predictive utility and counterfactual fairness. 
However, we note that direct comparisons are structurally asymmetric: \citet{zhou2025counterfactualfairnesscombiningfactual} operates under a deterministic, invertible SCM (Level 1 CF) that assumes exact counterfactuals can be uniquely generated, allowing them to evaluate fairness by explicitly penalizing pointwise differences during training. 
Conversely, our approach operates under causal uncertainty (Level 2 CF) as a model-agnostic post-processor.

Using \citet{zhou2025counterfactualfairnesscombiningfactual}'s extracted VAE latent space, their method successfully drives the deterministic Total Effect (TE) error to exactly $0.000$ due to explicit loss penalties, but this strict constraint severely degrades predictive utility. 
Despite using their highly clustered latent space (which violates our non atomic assumptions), our post-processor preserves the base model's accuracy much better while still achieving highly competitive conditional CF and exact Demographic Parity.

\begin{table}[ht!]
    \centering
    \small
    \begin{tabular*}{\linewidth}{@{\extracolsep{\fill}}lcccc}
        \toprule
        Method & RMSE $\downarrow$ & CF $\downarrow$ & TE $\downarrow$ & DP $\downarrow$ \\
        \midrule
        Base      & 1.032 & 0.002 & 0.136 & 0.002 \\
        Fair K    & 1.095 & \textbf{0.000} & \textbf{0.000} & \textbf{0.000} \\
        WFR       & 1.038 & 0.001 & 0.143 & \textbf{0.000} \\
        Zhou      & 1.057 & 0.000 & \textbf{0.000} & \textbf{0.000} \\
        Ours      & \textbf{1.038} & 0.001 & 0.148 & \textbf{0.000} \\
        \bottomrule
    \end{tabular*}
    \vspace{0.3em}
    \caption{\textbf{Comparison to Zhou et al. on the Law School dataset.} Using Zhou et al.'s extracted VAE latent space and deterministic exact counterfactuals (TE). While Zhou et al. drives TE to zero, it severely degrades predictive utility. Our post-processor preserves the base model's accuracy much better (RMSE: $1.038$) while achieving highly competitive conditional CF ($0.0012$).}
    \label{tab:zhou_comparison}
\end{table}

\clearpage

\newpage

\end{document}